\documentclass[10pt]{article}
\usepackage{pratik}
\definecolor{yuting}{RGB}{255,69,0}
\definecolor{yxc}{RGB}{21,199,133}

\newcommand{\papertitle}{Stochastic Runge-Kutta Methods: \\ Provable Acceleration of Diffusion Models}

\renewcommand{\hat}{\widehat}
\renewcommand{\bar}{\overline}

\title{\papertitle}

\author{%
	Yuchen Wu\thanks{Department of Statistics and Data Science,  University of Pennsylvania; email: \texttt{\{wuyc14,yuxinc,ytwei\}@wharton.upenn.edu}.}\and
 Yuxin Chen\footnotemark[1]
 \and
Yuting Wei\footnotemark[1] 
}

\date{\today}

\begin{document}

\maketitle

\begin{abstract}

Diffusion models play a pivotal role in contemporary generative modeling, claiming state-of-the-art performance across various domains. 
Despite their superior sample quality, mainstream diffusion-based stochastic samplers like DDPM often require a large number of score function evaluations, incurring considerably higher computational cost compared to
single-step generators like generative adversarial networks. 
While several acceleration methods have been proposed in practice, the theoretical foundations for accelerating diffusion models remain underexplored. 
In this paper, we propose and analyze a training-free acceleration algorithm for SDE-style diffusion samplers,  based on the stochastic Runge-Kutta method. 
The proposed sampler provably attains $\varepsilon^2$ error---measured in KL divergence---using  $\widetilde O(d^{3/2} / \varepsilon)$ score function evaluations (for sufficiently small $\varepsilon$), strengthening the state-of-the-art guarantees $\widetilde O(d^{3} / \varepsilon)$ in terms of dimensional dependency. 
Numerical experiments validate the efficiency of the proposed method.

\end{abstract}

\tableofcontents

\section{Introduction}

Initially introduced by \cite{sohl2015deep} in the context of thermodynamics modeling, diffusion models now play a pivotal role in modern generative modeling, a task that aims to generate new data instances that resemble the training data in distribution. 
Remarkably, diffusion models are capable of producing high-quality synthetic samples, and have claimed the state-of-the-art performance across various domains, ranging from image generation \citep{song2019generative,ho2020denoising,song2020denoising,dhariwal2021diffusion,nichol2021glide,ho2022cascaded,rombach2022high,saharia2022photorealistic,ho2022classifier}, text generation \citep{austin2021structured,li2022diffusion,ramesh2022hierarchical}, speech synthesis,  \citep{popov2021grad,kim2022guided}, time series imputation \citep{tashiro2021csdi,alcaraz2022diffusion}, reinforcement learning \citep{pearce2023imitating,hansen2023idql}, and molecule modeling \citep{anand2022protein,xu2022geodiff,trippediffusion}. 
Remarkably, diffusion models have served as crucial components of mainstream content generators including Stable Diffusion \citep{rombach2022high}, DALL-E \citep{ramesh2022hierarchical}, and Imagen \citep{saharia2022photorealistic}, among others, achieving superior  performance in the now rapidly growing field of generative artificial intelligence. 
We refer the interested reader to \cite{yang2023diffusion} for a comprehensive survey of methods and applications pertinent to diffusion models, and to \cite{tang2024score,chen2024overview} for overviews of recent theoretical development. 

\subsection{Diffusion model overview}

On a high level, diffusion models take into consideration two processes: 

\begin{enumerate}
    \item[1)] a forward process 
    \begin{align*}
        X_{0} \to X_{1} \to \cdots \to X_{K}
    \end{align*}
    that sequentially diffuses the target data distribution into an easy-to-sample prior, typically chosen as a standard Gaussian distribution;

    \item[2)] a learned reverse process
    \begin{align*}
        Y_{0} \to Y_{1} \to \cdots \to Y_K 
    \end{align*}
    that transforms the prior  (e.g., standard Gaussian) back into a distribution that resembles the target distribution, with the aim of  achieving $X_k \overset{\mathrm{d}}{\approx} Y_k$ for all $k = 0, 1, \cdots, K$. 
    A key component that enables  the construction of a faithful reverse process is the estimated (Stein) score functions \citep{song2020score}, typically represented by pre-trained neural networks. 
    During the sampling phase, only the reverse process is implemented to generate new data instances. 
\end{enumerate}

\noindent 
Constructing the forward process is generally straightforward which often amounts to successively injecting noise into the data; 
in contrast,  the reverse process is far more complicated, which generally involves  evaluating large-scale denoising neural networks recursively (for the purpose of computing the estimated score functions) to restore the target distribution.  
Viewed in  this light, the number of function evaluations (NFE)---more precisely, the number of times needed to compute the output of, say, denoising neural networks---oftentimes dictates the efficiency of diffusion-based samplers. 
 
There are at least two primary approaches concerned with the construction of the reverse processes: stochastic differential equation (SDE)-based samplers, and ordinary differential equation (ODE)-based samplers \citep{song2020score}. 
These samplers are  based on discrete-time processes that approximate the dynamics of certain diffusion SDEs and ODEs, such that when initialized at the prior, the solutions of these differential equations are designed to have marginal distributions that match the target distribution.
Prominent examples of SDE-based and ODE-based samplers include the Denoising
Diffusion Probabilistic Model (DDPM) \citep{ho2020denoising} and the Denoising Diffusion Implicit Model (DDIM) \citep{song2020denoising}, respectively.
Empirically, ODE-based samplers offer faster sampling speeds compared to the SDE-based counterpart, 
while SDE-based samplers often generates higher-quality samples given sufficient runtime \citep{song2020denoising,nichol2021improved}. The respective advantages of these two approaches motivate researchers to explore both types of samplers.

\begin{figure}[t]
    \centering
    \includegraphics[width=0.9\linewidth]{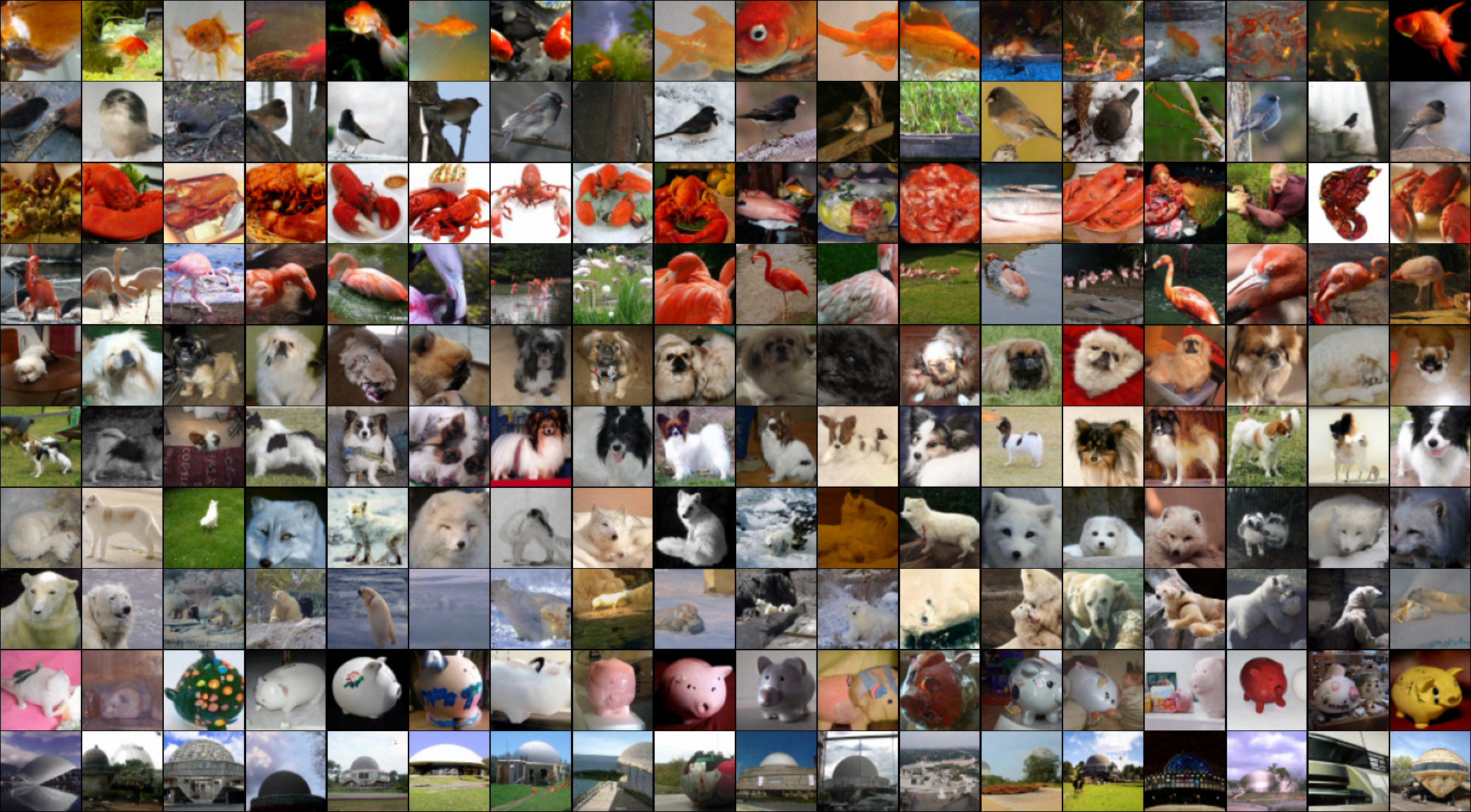}
    \caption{Class-conditional ImageNet $64 \times 64$ samples generated using 250 sampling steps with our method (\cref{alg:stochastic-RK}).}
    \label{fig:patch}
\end{figure}

\subsection{Accelerating diffusion models}

While mainstream diffusion-based samplers like DDPM are known to generate high-fidelity samples, they often suffer from low sampling speed, 
requiring a large number of 
score function evaluations (oftentimes being neural network evaluations) to generate samples. For this reason, diffusion models incur considerably higher computational costs compared to single-step generators like generative adversarial networks (GANs) \citep{goodfellow2014generative} or variational auto-encoders (VAEs) \citep{kingma2013auto}, thus constraining their practicality in real-world applications that demand real-time data generation.

To remedy this efficiency issue, researchers have proposed several acceleration schemes to speed up the sampling process of diffusion models. 
Prominent examples include the training-based method, such as model distillation \citep{luhman2021knowledge,salimansprogressive,meng2023distillation}, noise level or sample trajectory learning \citep{nichol2021improved,san2021noise}, and consistency models \citep{song2023consistency,li2024towards}. 
Despite their impressive performance, training-based acceleration methods incur enormous additional computational costs for training, and can be challenging to implement for large-scale pre-trained diffusion models.
In contrast, an alternative class of acceleration methods is based on modifying the original diffusion models without additional training, offering the flexibility to wrap around any pre-trained diffusion models (see,  e.g., \citet{lu2022dpm,zheng2023dpm,zhao2024unipc}). More detailed discussions about these training-free acceleration methods are deferred to Section~\ref{sec:related-work}.

Despite their empirical successes, most theoretical guarantees of diffusion acceleration
are established based on ODE-based algorithms (e.g., \citet{lee2023convergence,li2024accelerating,huang2024convergence}).
In comparison, rigorous convergence analysis for SDE-based acceleration remains largely underexplored, in spite of extensive theoretical investigation for the first-order unaccelerated solvers \citep{chen2023improved,chen2023sampling,lee2022convergence,benton2024nearly,li2023towards,liang2024non,li2024d,li2024adapting}. 
Given the popularity of stochastic samplers \citep{song2020score,lu2022dpm++,gonzalez2024seeds} and the fact that they tend to generate higher fidelity samples compared to their ODE-based analog, 
it is of great interest to design principled SDE-based acceleration schemes and demonstrate their provable advantages.



\subsection{Our contributions}

In this paper, we design a high-order SDE-based sampler, leveraging upon idea of stochastic Runge-Kutta methods.  
Our algorithm is training-free in nature. Each step only requires a single score function evaluation, introducing no extra per-step cost compared to DDPM.  %
For a broad family of target data distributions in $\mathbb{R}^d$, 
it only takes $\widetilde O(d^{3/2} / \varepsilon)$ score function evaluations for our proposed sampler to yield a distribution that is $\varepsilon^2$ close to the target distribution  in KL divergence, provided that the score estimates are sufficiently accurate and that $\varepsilon$ is sufficiently small. 
Compared to prior theory for accelerated SDE-based samplers, our result strengthens the state-of-the-art guarantees $\widetilde O(d^{3} / \varepsilon)$ in terms of dimensional dependency. 
More precise comparisons between our results and previous theory on SDE-based samplers are provided in Table~\ref{table:literature}. 
%
%
To demonstrate the practical efficiency of the proposed method, we conduct a series of numerical experiments, as illustrated in Figure~\ref{fig:patch}. More details can be found in Section~\ref{sec:experiments}.




\begin{table}[t]
\centering
\scalebox{0.97}{
\begin{tabular}{c  c  l  c  c}
    \toprule
    \textbf{Sampler} &\textbf{Distribution} & \textbf{Score estimation} &  \textbf{Complexity} & \textbf{Reference} \\
    \midrule
    SDE-based & Finite second moment  & \, $\ell_2$ score error &  $\widetilde{O}(d / {\eps^2})$ & \cite{benton2024nearly} \\ \midrule
    SDE-based & Bounded  & \, $\ell_2$ score error &  $\widetilde{O}(d^3 / {\eps})$ & \cite{li2024accelerating} \\ \midrule
    \textcolor{CornflowerBlue}{SDE-based} & \textcolor{CornflowerBlue}{Bounded} & \, \textcolor{CornflowerBlue}{$\ell_2$ score error} & \textcolor{CornflowerBlue}{$\widetilde O(d^{3 / 2}  / {\eps})$} & \textcolor{CornflowerBlue}{This work} \\ 
    \bottomrule
\end{tabular}}
\caption{ The number of score function evaluations  required to attain $\eps^2$ error measured in KL divergence. 
In this table, we ignore the impact of score estimation errors, and focus only on SDE-based samplers. 
We only emphasize the dependency on $d$ and $\eps$, omitting logarithmic factors and other constants. 
}
\label{table:literature}
\end{table}

\subsection{Other related works}
\label{sec:related-work}

Here, we briefly discuss several other prior theory on multiple aspects of diffusion models. 

\paragraph{Training-free acceleration schemes.}
A recent strand of works seeks to speed up ODE-based samplers via efficient ODE solvers. 
In particular, \citet{zhangfast} proposes DEIS, building on the semi-linear structure of the reverse process and utilizing the exponential integrator \citep{hochbruck2010exponential}. Similarly, \cite{lu2022dpm} introduces the DPM-solver by combining high-order ODE solvers with the semi-linear framework, 
and further develop DPM-Solver++ to enhance stability in guided sampling \citep{lu2022dpm++}.
Additionally, \cite{zhao2024unipc} establishes a predictor-corrector framework to accelerate diffusion sampling. 
In comparison, training-free acceleration for SDE-based samplers are considerably less explored. 
\cite{jolicoeur2021gotta} designs an SDE solver based on stochastic Improved Euler’s method. 
\cite{karras2022elucidating} proposes a a stochastic sampler that comines ODE integrator with a Langevin step. 
Motivated by Taylor expanding diffusion processes, \cite{lu2022dpm++} proposes SDE-DPM-Solver++. 
\cite{xue2024sa} presents the SA-solver, leveraging the stochastic Adams method to accelerate sampling speed. 
The theoretical underpinnings about these stochastic acceleration methods, however, remain far from complete.  

\paragraph{Theory for ODE-based acceleration.} 
In comparison to the theory for DDPM, the theoretical support for the ODE-based samplers has only been established fairly recently 
\citep{chen2023restoration,li2024towards,benton2023error,chen2024probability,li2024sharp,gao2024convergence,huang2024convergence}, 
where the state-of-the-art convergence guarantees for the probability flow ODE are established by \cite{li2024sharp}. 
A first attempt towards the design of provably accelerated training-free ODE-based methods is made by   \cite{li2024accelerating}, which proposes and analyzes both ODE- and SDE-based acceleration algorithms. 
The accelerated ODE sampler proposed therein leverages a momentum-like term to enhance sample efficiency, and their accelerated SDE sampler is constructed using higher-order expansions of the conditional density. Both of these samplers come with improved non-asymptotic convergence guarantees  compared to prior theory for the unaccelerated counterpart \citep{benton2024nearly,li2024sharp}. 
Furthermore, \cite{huang2024convergence} establish convergence guarantees for high-order ODE solvers in the context of diffusion models. 
\cite{gupta2024faster} proposes to accelerate ODE-based samplers by incorporating a randomized midpoint method, achieving state-of-the-art dependency on the problem dimension. 
To bypass the complexity of developing an end-to-end theory for diffusion models, these studies often establish non-asymptotic convergence results assuming access to accurate score function estimates. 

\paragraph{Theory for score matching/estimation.} 
In addition to the sampling phase, the score matching phase plays a crucial role in determining the sample quality \citep{hyvarinen2005estimation,lu2022maximum,koehler2022statistical}. 
To understand the finite-sample error of score function estimation, \cite{block2020generative} provides estimation guarantees under the $\ell_2$ metric in terms of the Rademacher complexity of a certain concept class. 
\cite{chen2023score}, \cite{oko2023diffusion}  and \cite{tang2024adaptivity} characterize the sample complexity of diffusion models when the target distribution resides within some low-dimensional linear space, the Besov space, and low-dimensional manifold, respectively. 
More broadly,  progress has been made within the theoretical community towards addressing multiple aspects arising in score estimation (see, e.g., \cite{oko2023diffusion,chen2024learning,wibisono2024optimal,dou2024optimal,zhang2024minimax,mei2023deep,feng2024optimal}).
From a more optimization perspective,
\cite{han2024neural}
studies the optimization error of using two-layer neural networks for score estimation.

\subsection{Notation}
\label{sec:notation}

For any positive integer $n$, we denote $[n]\coloneqq \{1,\dots,n\}$. 
For two sequences of non-negative real numbers $\{a_n\}_{n \in \NN_+}$ and $\{b_n\}_{n \in \NN_+}$, we employ the notation $a_n \lesssim b_n$  (resp.~$a_n \gtrsim b_n$) to indicate the existence of a universal constant $C$, such that $a_n \leq C b_n$  (resp.~$a_n \geq C b_n$) holds for all sufficiently large $n$.  
The notation $a_n=O(b_n)$ means $a_n\lesssim b_n$, and
$\widetilde O(\cdot)$ hides a factor that is polynomial in $(\log d, \log \eps^{-1}, \log \delta^{-1})$. 
For any tensor $T \in \RR^{d_1 \times d_2 \times d_3}$ and  matrix $M \in \RR^{d_2 \times d_3}$, we define $T[M]$ to be a vector in $\RR^{d_1}$, such that the $i$-th entry of this vector is given by
$$
	\big(T[M]\big)_i = \sum_{j \in [d_2], k \in [d_3]}T_{ijk}M_{jk} \eqqcolon \big\langle T(i,\cdot,\cdot),M \big\rangle.
$$
For any vector $v \in \RR^{d_3}$, we define $T v$ to be a matrix in $\RR^{d_1 \times d_2}$, such that the $(i, j)$-th entry of this matrix is 
$$
	\big(Tv\big)_{i,j} = \sum_{k \in [d_3]} T_{ijk} v_k \eqqcolon \big\langle T(i,j,\cdot), v \big\rangle.
$$
Similarly, for any fourth order tensor $T \in \RR^{d_1 \times d_2 \times d_3 \times d_4}$ and third order tensor $A \in \RR^{d_2 \times d_3 \times d_4}$, we define $T[A]$ to be a vector in $\RR^{d_1}$, with the $i$-th entry given by 
\begin{align*}
    \big(T[A] \big)_i = \sum_{j \in [d_2], k \in [d_3], \ell \in [d_4]} T_{ijk\ell} A_{jk\ell} =: \big\langle T(i,\cdot,\cdot,\cdot), A \big\rangle. 
\end{align*}
For any two random objects $X$ and $Y$, we say $X \indep Y$ if and only if they are statistically independent of each other.
For two distributions $\mu$ and $\nu$, we employ $\mu \otimes \nu$ to represent the product distribution of $\mu$ and $\nu$. 
For any random object $X$, we use $\cL(X)$ to denote its law (i.e., distribution). 
Moreover, for any vector-valued function $s(t,x): \RR \times \RR^d \rightarrow \RR^d$ whose two arguments are in $\RR$ and $\RR^d$, respectively,  
we denote by $\nabla_x s(t,x) \in \RR^{d\times d}$ (resp.~$\nabla_x^2 s(t,x) \in \RR^{d\times d\times d}$) the Jacobian matrix (resp.~Hessian) w.r.t.~the second argument. 
For any two distributions, 
we denote by $\KL(p \parallel q)$ the Kullback-Leibler(KL) divergence from $q$ to $p$, and use $\TV(p, q)$ to represent the total-variation (TV) distance between $p$ and $q$. 
For any positive integer $n$, we also use $\mathrm{perm}(n)$ to denote the set of permutations of $\{1,\dots,n\}$.


\section{Algorithm: a stochastic Runge-Kutta method}

In this section, 
we present the rationale underlying the design of stochastic Runge-Kutta methods, 
following some preliminaries about diffusion models from an SDE perspective.


\subsection{Background: diffusion models through the lens of SDEs}
\label{sec:diffusion-model}

As mentioned previously, the diffusion generative modeling comprises a forward process and a reverse process. 
A widely adopted choice of the forward process can be described via the Ornstein–Uhlenbeck (OU) process 
\begin{align}
\label{eq:forward}
	\dd X_t = -X_t \dd t + \sqrt{2}\, \dd B_t^{\mathsf{f}}, \qquad X_0 \sim q_0, \qquad 0 \leq t \leq T ,
\end{align}
with $q_0$ the target data distribution, and $(B_t^{\mathsf{f}})_{0 \leq t \leq T}$ a standard Brownian motion in $\RR^d$ independent from $X_0$.
As is well-known, the solution to \eqref{eq:forward} enjoys the following marginal distribution
\begin{align}
\label{eq:forward-dist-lambda-sigma}
	X_t \overset{\mathrm{d}}{=} \lambda_t X_0 + \sigma_t Z \qquad \text{with } \lambda_t \coloneqq e^{-t} ~\text{and}~\sigma_t \coloneqq \sqrt{1 - e^{-2t}}
\end{align}
for any $0 \leq t \leq T$, where $X_0 \sim q_0$ and $Z\sim \mathcal{N}(0,I_d)$ are independently generated. 
Throughout this paper, we shall denote by $q_t$ the distribution of $X_t$.


How to reverse the above forward process \eqref{eq:forward} 
can be illuminated via a classical result in the SDE literature.  
To be precise, consider the following SDE
\begin{align}
\label{eq:reverse-OU}
    \dd Y_t = \big[Y_t + 2s(t, Y_t)\big] \dd t + \sqrt{2}\, \dd B_t, \qquad   \qquad 0 \leq t \leq T,  
\end{align}
where $(B_t)_{0 \leq t \leq T}$ is also a standard Brownian motion in $\RR^d$ independent of $Y_0$, 
and 
\begin{align}
	s(t, x) \coloneqq \nabla_x \log q_{T - t}(x) 
\end{align}
stands for the (Stein) score function.  Classical results \citep{anderson1982reverse,haussmann1986time} tell us that the distribution match between the above two stochastic processes in the sense that $$Y_t \overset{\mathrm{d}}{=} X_{T - t}, \qquad  0 \leq t \leq T$$ 
as long as $Y_0 \sim q_T$, thus unveiling that \eqref{eq:reverse-OU} forms the reverse process of \eqref{eq:forward}. 
As an important implication, if we have exact access to the score functions $\{s(t,\cdot)\}_{0\leq t\leq T}$ as well as $q_T$,  
then running the SDE \eqref{eq:reverse-OU} from $Y_0$ suffices to yield a point $Y_T$ that exhibits the target distribution $q_0$. 

Nevertheless, there are multiple implementation issues that prevent one from running the reverse process \eqref{eq:reverse-OU} in an exact manner. To begin with, it is unrealistic to assume exact access to the score functions; instead, one only has, for the most part, imperfect score estimates at hand. Secondly, due to the computational cost of evaluating each score function (which might involve, say, computing the output of a large neural network or transformer), it is preferable to solve the SDE \eqref{eq:reverse-OU} approximately with only a small number of score function evaluations; 
as a consequence, time-discretization of the SDE \eqref{eq:reverse-OU} is oftentimes necessary, in spite of the discretization error it inevitably incurs.  Thirdly, the SDE \eqref{eq:reverse-OU} is typically not initialized to $Y_0\sim q_T$, but instead, some generic data-independent distribution like $\mathcal{N}(0,I_d)$ (given that $q_T$ can be fairly close to $\mathcal{N}(0,I_d)$ for large enough $T$). 
These issues constitute three sources that result in the discrepancy between $q_0$ and the distribution of $Y_T$, 
with the first two sources (i.e., the score estimation error and the discretization error) having the most significant effects.

\subsection{A stochastic Runge-Kutta method}
\label{sec:runge-kutta}

Runge-Kutta methods are a widely used family of iterative algorithms for approximating solutions to differential equations \citep{runge1895numerische,kutta1901beitrag},  
which enable the construction of high-order numerical solvers without requiring higher-order derivatives of the functions involved.
Stochastic Runge-Kutta methods refer to a family of specialized adaptation of the general Runge-Kutta methods, designed specifically for solving SDEs \citep{burrage1996high}. 
Motivated by the practical effectiveness of Runge-Kutta-type algorithms, we propose a high-order stochastic Runge-Kutta method---in conjunction with the use of the exponential integrator---for solving the reverse process described in \cref{eq:reverse-OU}. 
Here and throughout, we select $K$ discretization time points $0 = t_0 < t_1 < \cdots < t_K < T$, and define 
\begin{align}
	\Delta_k \coloneqq t_{k + 1} - t_k \qquad  \text{for }k \in \{0,1,\dots,K-1\}.
\end{align}
It is assumed that for each $t_k$ ($0\leq k\leq K$), 
we only have access to the estimate $\widehat{s}(t_k,\cdot)$ of the true score function $s(t_k,\cdot)=\nabla_x \log q_{T-t_k}(\cdot)$. 

\paragraph{Preliminaries: exponential integrator, and SDE for scores.} 
The use of the exponential integrator arises as a common algorithmic trick in SDE to cope with linear drift components. 
Reformulating the SDE in \cref{eq:reverse-OU}, we obtain an equivalent form
\begin{align}
\label{eq:SDE-exp}
    \dd \big[e^{-t} Y_t \big] = 2e^{-t} s(t, Y_t) \dd t + \sqrt{2}\, e^{-t} \dd B_t, \qquad 0 \leq t \leq T,    
\end{align}
whose drift term does not contain a linear component as in \cref{eq:reverse-OU}. 
As a result, for any sequence of discretization time points $0 = t_0 < t_1 < \cdots < t_K < T$, 
we can take the integral  to derive, for any $t \in [t_k, t_{k + 1}]$, 
\begin{align}
\label{eq:Xt-sde-integral}
    Y_{t} = e^{t - t_k} Y_{t_k} + 2\int_0^{t - t_k}  e^{t - t_k - r} s\big(t_k + r, Y_{t_k + r}\big) \dd r + \sqrt{2} \int_0^{t - t_k} e^{t - t_k - r} \dd B_{t_k + r}. 
\end{align}

In addition, the evolution of $s(t, Y_t)$ can be characterized by means of another SDE. 
More specifically, applying the \ito formula \citep{oksendal2003stochastic} to $s(t, Y_t)$ yields 
\begin{align}
\label{eq:ito-S}
    \dd s(t, Y_t) = & \partial_t s(t, Y_t) \dd t+ \nabla_x s(t, Y_t) \big(Y_t + 2s(t, Y_t)\big) \dd t + \sqrt{2}\, \nabla_x s(t, Y_t) \dd B_t + \nabla_x^2 s(t, Y_t) [ I_d ] \dd t,
\end{align}
where we recall the notation of $\nabla^2_x s(t, Y_t) [ I_d ]$ in Section~\ref{sec:notation}.

\paragraph{Prelude: DDPM as a first-order Runge-Kutta solver.} 
%
We now turn to designing SDE solvers through the idea of the Runge-Kutta method. 
%
%
Let us begin with first-order score approximation and use it to describe the first-order Runge-Kutta solver. 
As a natural starting point, one can approximate $s(t_k + r, Y_{t_k + r})$ in \cref{eq:Xt-sde-integral} by $s(t_k, Y_{t_k})$ for every $r \in [0, \Delta_{k}]$. 
With this strategy in place, we arrive at the approximation
\begin{align*}
    Y_{t_{k + 1}} \approx e^{\Delta_k} Y_{t_k} + 2 (e^{\Delta_k} - 1) s\big(t_k, Y_{t_k}\big) + \sqrt{2} \int_0^{\Delta_k} e^{\Delta_k - r} \dd B_{t_k + r} 
\end{align*}
at the endpoint $t = t_{k + 1}$. 
This motivates the following first-order SDE solver that computes 
\begin{align}
	\wh Y_{t_{k + 1}} = e^{\Delta_k}\wh Y_{t_k} + 2 \big(e^{\Delta_k} - 1\big) \widehat{s} \big(t_k, \wh Y_{t_k}\big) + \sqrt{2} \int_0^{\Delta_k} e^{\Delta_k - r} \dd B_{t_k + r}
\end{align}
iteratively for $k=0,\dots,K-1$,  
which coincides with the exponential integrator solver tailored to the  DDPM sampler \citep{zhangfast}.


\paragraph{Our algorithm: a higher-order Runge-Kutta solver.} In order to further speed up the DDPM-type sampler, we seek to exploit higher-order approximation. Rearrange \cref{eq:Xt-sde-integral} as follows:
\begin{align*}
	Y_{t_{k + 1}} &=  e^{\Delta_k} Y_{t_k} + 2 (e^{\Delta_k} - 1) s\big(t_k, Y_{t_k}\big) + \sqrt{2} \int_0^{\Delta_k} e^{\Delta_k - r} \dd B_{t_k + r} \\
    & \qquad + 2 \int_0^{\Delta_k} e^{\Delta_k - r} \Big(s\big(t_k + r, Y_{t_k + r}\big) - s\big(t_k, Y_{t_k}\big)\Big) \dd r.
\end{align*}
The idea is to approximate the score difference $s(t_k + r, Y_{t_k + r}) - s(t_k, Y_{t_k})$ in the last term of the above display (as opposed to approximating the score $s(t_k + r, Y_{t_k + r})$ as in the first-order solver). 
Let us first present the update rule of the proposed Runge-Kutta solver as follows, whose rationale will be elucidated momentarily: 
%
%
\begin{subequations}
\label{eq:hat-X-tk-full}
\begin{align}
    \wh  Y_{t_{k + 1}} 
	&  =  e^{\Delta_k} \wh  Y_{t_k} + \big(e^{\Delta_k} - e^{-\Delta_k}\big) \wh s\big(t_k, \wh Y_{t_k} + g_{t_k, t_{k + 1}}\big) + \sqrt{2} \int_0^{\Delta_k} e^{\Delta_k - r} \dd W_{t_k + r},
	\label{eq:hat-X-tk}
\end{align}
where $(W_t)_{0\leq t\leq T}$ is a standard Brownian motion in $\mathbb{R}^d$ to be used for the Runge-Kutta solver in discrete time 
(in order to differentiate it from the process $(B_t)_{0\leq t\leq T}$ used for the reverse process in \cref{eq:reverse-OU}), and $g_{t_k, t_{k + 1}}$ is a Gaussian vector defined as
\begin{align}
\label{eq:double-index-g}
    g_{t_k, t_{k + 1}} \coloneqq \frac{2\sqrt{2}}{e^{\Delta_k} - e^{-\Delta_k}}\int_0^{\Delta_k} e^{\Delta_k - r}\big(W_{t_k + r} - W_{t_k}\big) \dd r.
\end{align}
\end{subequations}
%
We highlight several key properties of this algorithm.
\begin{itemize}
    \item Firstly, each iteration of \eqref{eq:hat-X-tk} only requires a single score function evaluation. This feature is in stark contrast with acceleration algorithms that demand higher-order computation. 
    
    \item If we set $\alpha_k = e^{-2\Delta_k}$, then the update rule \eqref{eq:hat-X-tk} reduces to 
    \begin{align}
    \label{eq:alpha-form}
        \wh Y_{t_{k + 1}} = \frac{1}{\sqrt{\alpha_k}} \Big( \wh Y_{t_{k}} + (1 - \alpha_k) \hat s(t_k, \wh Y_{t_k} + \zeta_{k, 1} g_{k, 1}) \Big) + \zeta_{k, 2} g_{k, 1} + \zeta_{k, 3} g_{k, 3}, 
    \end{align}
    where $g_{k, 1}, g_{k, 2}, g_{k, 3} \stackrel{\text{i.i.d.}}{\sim}
     \cN(0, I_d)$ are independent of $\hat Y_{t_k}$, and $\zeta_{k, 1}, \zeta_{k, 2}, \zeta_{k, 3} \in \RR$ are certain functions of $\Delta_k$ defined as follows: 
\begin{align}
\label{eq:zetas-main-txt}
\begin{split}
    &\qquad \zeta_{k, 1} = \frac{2\sqrt{2} f_1(\Delta_k)^{1/2}}{e^{\Delta_k} - e^{-\Delta_k}}, \qquad \zeta_{k, 2} = \frac{\sqrt{2} f_3(\Delta_k)}{f_1(\Delta_k)^{1/2}}, \qquad \zeta_{k, 3} = \sqrt{2f_2(\Delta_k) - \frac{2f_3(\Delta_k)^2}{f_1(\Delta_k)}}, \\
    & \text{with }~ f_1(\Delta) = e^{2\Delta} / 2 - 2 e^{\Delta} + \Delta + 3 / 2,  \quad f_2(\Delta) = e^{2\Delta} / 2 - 1 / 2, \quad f_3(\Delta) = e^{2\Delta} / 2 - e^{\Delta} + 1 / 2. 
\end{split}
\end{align}
 %
     In addition, as $\Delta_k \to 0^+$, we have 
    \begin{align}
        \zeta_{k, 1} \Delta_k^{-1/2} \to \sqrt{2/3}, \qquad \zeta_{k, 2} \Delta_k^{-1/2} \to \sqrt{3 / 2}, \qquad \zeta_{k, 3} \Delta_k^{-1/2} \to \sqrt{1/2}. 
    \end{align}
    We observe that our acceleration method shares similarities with that proposed by \cite{li2024accelerating}. They also adopt an algorithm of the form \eqref{eq:alpha-form}, but with different choices of $\zeta_{k, 1}, \zeta_{k, 2}, \zeta_{k, 3}$: in their formulation,  $\zeta_{k, 1} \Delta_k^{-1/2} \to 1$, $\zeta_{k, 2} \Delta_k^{-1/2} \to 1$, and $\zeta_{k, 3} \Delta_k^{-1/2} \to 1$ as $\Delta_k \to 0^+$. 
    In terms of technical motivation, their approach is based on high-order expansion of the probability density function, while our algorithm is motivated by the Runge-Kutta method for SDEs.

\end{itemize}
The whole procedure is summarized in Algorithm~\ref{alg:stochastic-RK}, described using the implementation-friendly form \eqref{eq:alpha-form}.

\paragraph{Rationale behind the construction of our Runge-Kutta solver \eqref{eq:hat-X-tk-full}.}
To understand the rationale underlying the above construction, we first note that in the SDE in \cref{eq:ito-S}, the term that dominates is 
$\sqrt{2} \nabla_x s(t, Y_t) \dd B_t$, 
and hence it is tempting to approximate $s(t_k + r, Y_{t_k + r}) - s(t_k, Y_{t_k})$ by $\sqrt{2}\nabla_x s(t_k, Y_{t_k}) (B_{t_k + r} - B_{t_k})$ to reach
\begin{align}
\label{eq:tilde-process-00}
\begin{split}	
	Y_{t_{k + 1}} &\approx  e^{\Delta_k} Y_{t_k} + 2 \big(e^{\Delta_k} - 1\big) s\big(t_k, Y_{t_k}\big) + \sqrt{2} \int_0^{\Delta_k} e^{\Delta_k - r} \dd B_{t_k + r}  \\
    &  \qquad + 2 \sqrt{2} \int_0^{\Delta_k} e^{\Delta_k - r}  \nabla_x s\big(t_k, Y_{t_k}\big) \big(B_{t_k + r} - B_{t_k}\big) \dd r . 
\end{split}
\end{align}
%
%
Note, however, that an approach designed directly based on \eqref{eq:tilde-process-00} could be computationally expensive in practice,  given that it requires evaluating the Jacobian of the score function --- a $(d\times d)$-dimensional object that is in general either inaccessible or too costly to estimate.

To remedy this issue, we propose an alternative solution. 
Observe from the Taylor expansion that
\begin{align}
	\notag
     \big(e^{\Delta_k } - e^{-\Delta_k}\big) \Big( s\big(t_k, Y_{t_k} + g_{t_k, t_{k + 1}}\big) - s\big(t_k, Y_{t_k}\big) \Big)  
	&\approx \big(e^{\Delta_k } - e^{-\Delta_k}\big) \nabla_x s(t_k, Y_{t_k})  g_{t_k, t_{k + 1}} \notag\\
	&= 2\sqrt{2}\,\nabla_x s(t_k, Y_{t_k}) \int_0^{\Delta_k} e^{\Delta_k - r}  \big(B_{t_k + r} - B_{t_k}\big) \dd r ,
	\label{eq:approx-score-diff-grad}
\end{align}
where we take 
   $ g_{t_k, t_{k + 1}} = \frac{2\sqrt{2}}{e^{\Delta_k} - e^{-\Delta_k}}\int_0^{\Delta_k} e^{\Delta_k - r}\big(B_{t_k + r} - B_{t_k}\big) \dd r$. 
This suggests that the last term in \cref{eq:tilde-process-00} can be well approximated by the difference of two score functions, 
without the need of computing the gradient of the score functions. 
Substituting \cref{eq:approx-score-diff-grad} into \cref{eq:tilde-process-00} gives
\begin{align*}
\begin{split}	
	Y_{t_{k + 1}} &\approx  e^{\Delta_k} Y_{t_k} + 2 \big(e^{\Delta_k} - 1\big) s\big(t_k, Y_{t_k}\big) + \sqrt{2} \int_0^{\Delta_k} e^{\Delta_k - r} \dd B_{t_k + r}  + \big(e^{\Delta_k } - e^{-\Delta_k}\big) \Big( s\big(t_k, Y_{t_k} + g_{t_k, t_{k + 1}}\big) - s\big(t_k, Y_{t_k}\big) \Big)
	\\ 
	& =
	e^{\Delta_k} Y_{t_k} +  \big(e^{\Delta_k } - e^{-\Delta_k}\big)  s\big(t_k, Y_{t_k} + g_{t_k, t_{k + 1}}\big)
	+ \sqrt{2} \int_0^{\Delta_k} e^{\Delta_k - r} \dd B_{t_k + r} 
	+ \big\{2(e^{\Delta_{k}}-1)-  e^{\Delta_{k}} + e^{-\Delta_{k}} \big\} s\big(t_k, Y_{t_k}\big)
	\\ 
	& \approx
	e^{\Delta_k} Y_{t_k} +  \big(e^{\Delta_k } - e^{-\Delta_k}\big)  s\big(t_k, Y_{t_k} + g_{t_k, t_{k + 1}}\big)
	+ \sqrt{2} \int_0^{\Delta_k} e^{\Delta_k - r} \dd B_{t_k + r}  ,
\end{split}
\end{align*}
where the last line drops a higher-order term in view of the following approximation
\begin{align}
\label{eq:coefficients-close}
    2(e^{\Delta_{k}}-1)-  e^{\Delta_{k}} + e^{-\Delta_{k}} 
	=e^{\Delta_{k}}+e^{-\Delta_{k}}-2=O\big(\Delta_{k}^{2}\big) .
\end{align}
Consequently, we arrive at the proposed approximation scheme as described in \cref{eq:hat-X-tk}. 
Here, we use the coefficient $(e^{\Delta_k} - e^{-\Delta_k})$ instead of $2(e^{\Delta_k} - 1)$ (as suggested by the exponential integrator) for two main reasons: (1) when $\Delta_k$ is sufficiently small, the two coefficients are approximately equivalent, as shown in \cref{eq:coefficients-close}, and (2) the first coefficient is more commonly used in mainstream diffusion models (e.g., the DDPM sampler \citep{ho2020denoising}). 



\begin{algorithm}[t]
	\DontPrintSemicolon
		\textbf{inputs:} score estimates $\{\widehat{s}(t,\cdot)\}$, 
		time interval $[0,T]$, discretization time points $0 = t_0  < \cdots < t_K < T$. 
		  \\
		generate $Y_0 \sim \mathcal{N}(0,I_d)$. \\
	\For{$k=0,1,\dots,K-1$}{
		compute
		\begin{align}
 \wh Y_{t_{k + 1}} = \frac{1}{\sqrt{\alpha_k}} \Big( \wh Y_{t_{k}} + (1 - \alpha_k) \hat s\big(t_k, \wh Y_{t_k} + \zeta_{k, 1} g_{k, 1}\big) \Big) + \zeta_{k, 2} g_{k, 1} + \zeta_{k, 3} g_{k, 3}, 
		\end{align}
		where $\alpha_k = e^{-2\Delta_k}$ with $\Delta_k = t_{k + 1} - t_k$, $g_{k, 1}, g_{k, 2}, g_{k, 3} \sim_{i.i.d.} \cN(0, I_d)$ are independent of $\hat Y_{t_k}$, and $\zeta_{k, 1}, \zeta_{k, 2}, \zeta_{k, 3} \in \RR$ are functions of $\Delta_k$, as defined in \cref{eq:zetas-main-txt}.
	} 
    \caption{A stochastic Runge-Kutta method for diffusion models.}
 \label{alg:stochastic-RK}
\end{algorithm}

\section{Theoretical guarantees}
\label{sec:theory}

In this section, we present a convergence theory for the proposed stochastic Runge-Kutta solver in Algorithm~\ref{alg:stochastic-RK}. 
Let us begin by stating a couple of key assumptions, 
with the first one concerning the boundedness of the target data distribution. 
\begin{assumption}[bounded support]
\label{assumption:moments}
    Assume the target distribution $q_0$ obeys
    \begin{align*}
        \PP_{Y \sim q_0}\big(\|Y\|_2 \leq R\big) = 1
    \end{align*}
	for some quantity $R>0$. 
    Without loss of generality, we assume throughout that $R = \sqrt{d}$. 
	Note that for any distribution with bounded support,  we can always achieve this by properly rescaling the data. 
\end{assumption}

The next assumption below imposes a few conditions on the choice of the discretization time points $0=t_0<\dots<t_K<T$, where we recall that $\Delta_k=t_{k+1}-t_k$.  A concrete choice of valid discretization time points shall be provided momentarily in  Corollary~\ref{cor:example}. 
%
\begin{assumption}[discretization time points]
\label{assumption:step-size}
Suppose that there exists $\kappa \in (0, 1 / 4)$, such that 
\begin{align}
\Delta_k \leq \kappa \min \{1, (T - t_{k + 1})^2\},
\quad k = 0, 1,  \cdots, K - 1
\qquad \text{and}\qquad 
d^2 \kappa \lesssim 1.
\end{align}
%
It is further assumed that 
\begin{align}
1.3\Delta_k +   {\big(53\Delta_k + 10 \Delta_k^2\big) \big(\sigma_{T - t_k}^{-2} + \lambda_{T - t_k}^2 \sigma_{T - t_k}^{-4}\big) {d}} \leq 1 / 2, 
\qquad 
k = 0, 1,  \cdots, K - 1,
\label{eq:Delta-k-size-small}
\end{align}
where we recall the definition of $\lambda_t$ and $\sigma_t$ in \eqref{eq:forward-dist-lambda-sigma}. In addition, we assume that 
 $\delta = T - t_K > 0$.  
\end{assumption}
\begin{remark}
One can often interpret $\kappa$ as a proxy for the step size. 
 The upper bound $1 / 4$ in Assumption~\ref{assumption:step-size} is not crucial and can be replaced by any positive numerical constant. 
 \end{remark}
\begin{remark}
The assumption of $\delta$ being positive implies early stopping when tracking the reverse process. 
This step is essential, as for non-smooth target distributions, the score function $s(t, \cdot)$ can diverge as $t \to T$. 
In effect, our algorithm approximates a slightly noised distribution $q_{\delta}$ rather than the exact target distribution $q_0$, which is acceptable for a sufficiently small $\delta$.
Moreover, $q_{\delta}$ and $q_0$ are close in Wasserstein distance. 
This early stopping technique is commonly used in both practical applications and theoretical analysis of diffusion models \citep{song2020score,benton2024nearly}.
\end{remark}


Furthermore, we are still in need of assumptions that capture the accuracy of the estimated score functions, as stated below. 
\begin{assumption}[score estimation error]
\label{assumption:score}
	Suppose that the score estimates $\{\widehat s(t,\cdot)\}$ satisfy the following properties: 
    \begin{enumerate}
        \item For every $k = 0, 1, \cdots, K - 1$, it holds that
        \begin{align*}
            \sup_{a_k \in \sI_{k}, b_k \in \sI_{k}'}\EE\left[ \big\|
            \widehat s(t_k, a_k Y_{t_{k + 1}} + b_k g)
            -
            s(t_k, a_k Y_{t_{k + 1}} + b_k g)  \big\|_2^2 \right] \leq \eps_{\mathsf{score},k}^2, 
        \end{align*}
        where $\sI_{k} = [1 - 3.1 \sqrt{\Delta_k \kappa}, 1 + 3.1 \sqrt{\Delta_k \kappa}]$, $\sI_{k}' = [0, 3.5 \Delta_k^{1/2}]$, and $Y_{t_{k + 1}} \sim q_{T - t_{k + 1}}$ and $g  \sim \cN(0, I_d)$ are independently generated. 
        Recall that $\kappa$ is the stepsize proxy, as defined in Assumption \ref{assumption:step-size}.

		    
        \item For all $k = 0, 1, \cdots, K - 1$,  it holds that 
        \begin{align*}
            \sup_{y \in \RR^d} 
            \frac{\sigma_{T - t_k}^2}{\lambda_{T - t_k}} \big\|  \widehat s(t_k, y) - s(t_k, y) \big\|_2 \leq 2\sqrt{d}. 
        \end{align*}
    \end{enumerate} 
\end{assumption}

\begin{remark}
\label{remark:m}
    To interpret the second point of Assumption \ref{assumption:score}, observe that $\lambda_{T - t}^{-1}(\sigma_{T - t}^2 s(t, y) + y) = m(t, y)$, where $m(t, y) = \EE[\theta \mid \lambda_{T - t} \theta + \sigma_{T - t} g = y]$ with $(\theta, g) \sim q_0 \otimes \cN(0, I_d)$. 
    Under Assumption \ref{assumption:moments}, it holds that $\|m(t, y)\|_2 \leq \sqrt{d}$ for all $t \in [0, T]$ and $y \in \RR^d$. 
    Similarly, we define $\widehat m(t, y) = \lambda_{T - t}^{-1}(\sigma_{T - t}^2 \widehat s(t, y) + y)$. 
    Then Assumption \ref{assumption:score} is equivalent to assuming 
    \begin{align}
    \sup_{y \in \RR^d} \big\|\widehat m(t, y) - m(t, y) \big\|_2 \leq 2\sqrt{d},\qquad \text{for all }
    t \in \{t_0, t_1, \cdots, t_{K - 1}\}.
    \end{align}
    %
    which is valid as long as $\|\widehat m(t, y)\|_2 \leq \sqrt{d}$. 
    For a general $\widehat s$, the second point of Assumption \ref{assumption:score} is always satisfied by projecting $\widehat m(t, y) = \lambda_{T - t}^{-1}(\sigma_{T - t}^2 \widehat s(t, y) + y)$ onto the $d$-dimensional ball $\{x \in \RR^d: \|x\|_2 \leq \sqrt{d}\}$. 
    Furthermore, this projection reduces the score estimation error. 
    Specifically, if we denote the projection operator by $\mathsf{P}$, and let $\widehat s^{\mathsf{P}}(t, y) := \sigma_{T - t}^{-2} (\lambda_{T - t}\mathsf{P} (\widehat m(t, y)) - y )$, then it holds that $\|\widehat s^{\mathsf{P}}(t, y) - s(t, y)\|_2 \leq \|\widehat s(t, y) - s(t, y)\|_2$. 
\end{remark}


We are now positioned to present our convergence guarantees for the proposed Runge-Kutta method, as stated in the following theorem. 
Here and throughout, 
$p_{\mathsf{output}}$ stands for the distribution of the output $\widehat{Y}_{t_K}$ of Algorithm~\ref{alg:stochastic-RK}. 

\begin{theorem}
\label{thm:main}
    Under Assumptions \ref{assumption:moments}, \ref{assumption:step-size} and \ref{assumption:score}, Algorithm~\ref{alg:stochastic-RK} achieves
    \begin{align*}
        \KL (q_{\delta} \parallel p_{\mathsf{output}}) \lesssim  d^3  \kappa^2 T  + d^7  \kappa^3 (\delta^{-1} + T) + \kappa^{1/2} d\sum_{k = 0}^{K - 1}  {\eps_{\mathsf{score},k}^{1 / 2} \sigma_{T - t_k} }{\lambda_{T - t_k}^{-1 / 2}} + de^{-2T}.  
    \end{align*}
\end{theorem}
%



The proof of \cref{thm:main} is postponed to \cref{sec:analysis}.
In the next corollary, we derive the upper bound for KL divergence based on a specific stepsize selection. The proof of this corollary is provided in Appendix \ref{sec:proof-cor:example}.

\begin{corollary}
\label{cor:example}
    Consider any early stopping point $\delta \in (0, 1 / 2)$ and any desired accuracy level $\eps^2$. 
    \begin{enumerate}
        \item If $\eps \leq 1 / \sqrt{d}$, then there exist $T$ and $0 = t_0 < t_1 < \cdots < t_K = T - \delta$ such that 
Algorithm~\ref{alg:stochastic-RK} achieves 
    \begin{align*}
        \KL (q_{\delta} \parallel p_{\mathsf{output}}) \lesssim \eps^2 + \eps^3 d^{5/2} \delta^{-1} + \sum_{k = 0}^{K - 1} d^{1/2}\eps_{\mathsf{score},k}^{1/2} 
    \end{align*}
    with $K = \widetilde O \big(d^{3/2} (\eps \delta)^{-1}\big)$. 
    
    \item If $\sqrt{d} / 2 \geq \eps \geq 1 / \sqrt{d}$, then there exist $T$ and $0 = t_0 < t_1 < \cdots < t_K = T - \delta$ such that Algorithm~\ref{alg:stochastic-RK} yields
    \begin{align*}
        \KL (q_{\delta} \parallel p_{\mathsf{output}}) \lesssim \eps^2 + \eps^3 d^{5/2} \delta^{-1} + \sum_{k = 0}^{K - 1} d^{1/2}\eps_{\mathsf{score},k}^{1/2}
    \end{align*}
    with $K = \widetilde O(d^2 \delta^{-1})$. 
    \end{enumerate}
    When attaining the above upper bounds, we take $\Delta_{K - 1} = \kappa \delta^2$, and $\Delta_{k - 1} = \min \{\kappa, \Delta_k (1 + \sqrt{\kappa \Delta_k})^2\}$ for $k = K - 1, K - 2, \cdots, 1$. More details are given in Appendix \ref{sec:proof-cor:example}.
\end{corollary}
%

When the score estimation errors are negligible (i.e., $\eps_{\mathsf{score},k} \approx 0$) and $\eps$ is sufficiently small, Corollary \ref{cor:example} implies that  \cref{alg:stochastic-RK} requires $$\widetilde{O}\big(d^{3/2}(\eps \delta)^{-1}\big)$$ steps---or equivalently,  $\widetilde{O}(d^{3/2}(\eps \delta)^{-1})$ score function evaluations---to 
yield a distribution that is within $\eps^2$-KL-divergence to an early-stopped target distribution.
In contrast, 
(1) the state-of-the-art convergence rate for unaccelerated diffusion model is $\widetilde{O}(d / \eps^2)$ \citep{benton2024nearly}, so that our theory exhibits improved $\varepsilon$-dependency. Also,  
the iteration complexity established for the SDE-based accelerated algorithm in \cite{li2024accelerating} is $\widetilde{O}(d^3 / \eps)$, and hence our result exhibits better $d$-dependency than the theory presented in \cite{li2024accelerating}. 
It is also noteworthy that the analyses in both \citet{benton2024nearly,li2024accelerating} offer a more favorable dependency on the score estimation error. 
We believe that our less favorable dependency on the score errors stems from our technical limitations as opposed to the algorithm drawback; improving this dependency calls for new techniques that we leave for future research.

\begin{remark}
It is worth noting that our theory is stated in terms of the KL divergence between the algorithm output and the target distribution. While one can simply invoke the Pinsker inequality (i.e., $\TV(q_{\delta} , p_{\mathsf{output}}) \leq \sqrt{\KL (q_{\delta} \parallel p_{\mathsf{output}})/2}$\,) to obtain an upper bound on the total-variation distance, this approach has been shown to be sub-optimal for stochastic samplers like DDPM. In fact, the concurrent work \cite{li2024d} has established the striking result that  $\TV(q_{\delta} , p_{\mathsf{output}})$ could be  order-of-magnitudes better than $\sqrt{\KL (q_{\delta} \parallel p_{\mathsf{output}})/2}$ for DDPM. How to obtain the desired control on the TV metric for our proposed sampler is left for future investigation. 
\end{remark}




\section{Analysis}
\label{sec:analysis}

In this section, we provide an overview of the proof strategies for \cref{thm:main}, with detailed proofs deferred to the appendix. 

\paragraph{Step 1: constructing an auxiliary process.} 
To facilitate analysis, we introduce the following auxiliary stochastic process, obtained by replacing the estimated score $\widehat{s}(t_k,\cdot)$ with the true score $s(t_k,\cdot)$ in the update rule \eqref{eq:hat-X-tk} and initializing it to the distribution $q_T$: 
\begin{align}
\label{eq:bar-X-tk}
\begin{split}
	\overline Y_{t_{0}} &\sim q_T,\\
    \overline Y_{t_{k + 1}} &=  e^{\Delta_k} \overline Y_{t_k} + \big(e^{\Delta_k} - e^{-\Delta_k}\big) s\big(t_k, \overline Y_{t_k} + g_{t_k, t_{k + 1}}\big) + \sqrt{2} \int_0^{\Delta_k} e^{\Delta_k - r} \dd W_{t_k + r}, 
~~ 0\leq k< K, 
\end{split}
\end{align}
where we recall that $g_{t_k, t_{k + 1}} \in \RR^d$ is a Gaussian random vector defined in \cref{eq:double-index-g}. 
Given that the process \eqref{eq:bar-X-tk} is defined only at the discretization time points $\{t_k\}$, 
we find it convenient to introduce a natural interpolation of \cref{eq:bar-X-tk} to cover all time instances: 
for any $t \in (t_k, t_{k + 1})$, take 
\begin{align}
\label{eq:bar-X-t-process}
\begin{split}
    \overline Y_{t} =  e^{t - t_k} \overline Y_{t_k} + \big(e^{t - t_k} - e^{-t + t_k}\big) s\big(t_k, \overline Y_{t_k} + g_{t_k, t}\big) + \sqrt{2} \int_0^{t - t_k} e^{t - t_k - r} \dd W_{t_k + r}. 
\end{split}
\end{align}
where
\begin{align}
\label{eq:double-index-g-general}
    g_{t_k, t} \coloneqq \frac{2\sqrt{2}}{e^{t-t_k} - e^{-t+t_k}}\int_0^{t-t_k} e^{t-t_k - r}\big(W_{t_k + r} - W_{t_k}\big) \dd r.
\end{align}

Before proceeding, let us compare the auxiliary process $(\overline Y_t)$ 
with the exact reverse process $(Y_t)$.  
Recall that the true reverse process can be rearranged as (cf.~\cref{eq:Xt-sde-integral}):
%
\begin{align}
\label{eq:X-t-process-ref}
\begin{split}
	Y_t &=  e^{t - t_k} Y_{t_k} + 2 (e^{t - t_k} - 1) s(t_k, Y_{t_k}) + \sqrt{2} \int_0^{t - t_k} e^{t - t_k - r} \dd B_{t_k + r} \\
    & \qquad + 2 \int_0^{t - t_k} e^{t - t_k - r} \big(s(t_k + r, Y_{t_k + r}) - s(t_k, Y_{t_k})\big) \dd r.
\end{split}
\end{align}
Taking the differential of processes \eqref{eq:bar-X-t-process} and \eqref{eq:X-t-process-ref} and rearranging terms reveal that: 
for any $t \in (t_k, t_{k + 1}]$, 
\begin{subequations}
\label{eq:full-Ft-Fbart}
\begin{align}
	\dd \overline Y_t &= \Big[ e^{t - t_k} \overline Y_{t_k} + (e^{t - t_k} + e^{-t + t_k}) s\big(t_k, \overline Y_{t_k} + g_{t_k, t}\big) + \sqrt{2} \int_0^{t - t_k} e^{t - t_k - r} \dd W_{t_k + r} \Big] \dd t + \sqrt{2}\, \dd W_t  \nonumber  \\
    & ~~+ 2\sqrt{2}\,\nabla_x s\big(t_k, \overline Y_{t_k} + g_{t_k, t}\big) \big(W_t - W_{t_k}\big) \dd t + 2\sqrt{2} \,\nabla_x s\big(t_k, \overline Y_{t_k} + g_{t_k, t}\big) \int_0^{t - t_k} e^{t - t_k - r} \big(W_{t_k + r} - W_{t_k}\big) \dd r \dd t \nonumber \\ 
    & ~~- \frac{2\sqrt{2}(e^{t - t_k} + e^{-t + t_k})}{e^{t - t_k} - e^{-t + t_k}} \nabla_x s\big(t_k, \overline Y_{t_k} + g_{t_k, t}\big) \int_0^{t - t_k} e^{t - t_k - r} \big(W_{t_k + r} - W_{t_k}\big) \dd r \dd t, \label{eq:dbarXt-long-equation}  \\
	\dd Y_t &=  e^{t - t_k} \Big[ Y_{t_k} + 2 s(t_k, Y_{t_k}) + \sqrt{2} \int_0^{t - t_k} e^{-r} \dd B_{t_k + r} \Big] \dd t + 2 \big(s(t, Y_t) - s(t_k, Y_{t_k})\big) \dd t + \sqrt{2}\, \dd B_t \nonumber \\
    & ~~+ 2 \int_0^{t - t_k} e^{t - t_k - r} \big(s(t_k + r, Y_{t_k + r}) - s(t_k, Y_{t_k})\big) \dd r \dd t. \label{eq:dXt-long-equation} 
\end{align}
\end{subequations}
For notational simplicity, we shall often abbreviate \cref{eq:full-Ft-Fbart} by 
\begin{subequations}
\label{eq:abbrev-Ft-Fbart}
\begin{align}
& \dd \overline Y_t = \overline \cF\Big(t, \overline Y_{t_k}, (W_s - W_{t_k})_{t_k \leq s \leq t}\Big) \dd t + \sqrt{2} \dd W_t \\
& \dd Y_t = \cF\Big(t, 
(Y_s)_{t_k \leq s \leq t}, 
(B_s - B_{t_k})_{t_k \leq s \leq t}\Big) \dd t + \sqrt{2} \dd B_t
\end{align}
\end{subequations}
in the sequel, 
where the definitions of the mappings $\cF$ and $\overline\cF$ are clear from the context. 
From the original definition of the reverse process in \cref{eq:reverse-OU}, 
it can be easily shown that (which we omit here for brevity)
\begin{align}
	\label{eq:equiv-form-Ft}
        \cF\Big(t, (Y_s)_{t_k \leq s \leq t}, (B_s - B_{t_k})_{t_k \leq s \leq t}\Big) = Y_t + 2 s(t, Y_t). 
\end{align}
As a result, we shall often adopt the following more concise notation
\begin{align}
\label{eq:equiv-form-Ft-simple}
        \cF(t, Y_t) = 
        \cF\Big(t, (Y_s)_{t_k \leq s \leq t}, (B_s - B_{t_k})_{t_k \leq s \leq t}\Big)
        =
        Y_t + 2 s(t, Y_t). 
\end{align}

\paragraph{Step 2: characterizing the effect of time-discretization errors.}
With the aforementioned properties in mind, this step seeks to upper bound  the KL divergence between the two processes $(Y_t)$ and $(\overline Y_t)$---both initialized by $Y_0 \overset{\mathrm{d}}{=} \overline Y_0 \sim q_T$---by means of Girsanov's Theorem. 
On a high level,  the primary purpose of this step is to characterize the impact of the time-discretization error and the approximation error due to the application of the Runge-Kutta method, 
given that both the reverse process $(Y_t)$
 and its time-discretized counterpart $(\overline{Y}_t)$ are constructed using exact score functions. 

To begin with, we upper bound the KL divergence between $(\overline Y_t)$ and $(Y_t)$ using the differences between the mappings $\mathcal{F}$ and $\overline{\mathcal{F}}$ introduced in Step 1, as stated 
below. 
\begin{lemma}
\label{lemma:KL-barX-X}
    Denote by $Q_{T - \delta}$ (resp.~$\overline Q_{T - \delta}$) the distribution of $(Y_{t})_{0 \leq t \leq T - \delta}$ (resp.~$(\overline Y_{t})_{0 \leq t \leq T - \delta}$), which we recall are respectively defined in Eqs.~\eqref{eq:X-t-process-ref} and \eqref{eq:bar-X-t-process} from initialization $Y_0 \overset{\mathrm{d}}{=} \overline Y_0 \sim q_T$. 
    Then under Assumption \ref{assumption:moments}, it holds that 
    \begin{align*}
     \KL (Q_{T - \delta} \parallel \overline Q_{T - \delta}) \leq \sum_{k = 0}^{K - 1} \liminf_{\tau \to 0^+} \int_{t_k + \tau}^{t_{k + 1}} \EE\Big[ \big\| \cF(t, Y_t) - \overline \cF\big(t, Y_{t_k}, (H_s^{\tau})_{t_k \leq s \leq t}\big) \big\|_2^2  \Big] \dd t.
\end{align*}
In the above display, $(Y_t)_{0 \leq t \leq T - \delta} \sim Q_{T - \delta}$, and $(H_s^{\tau})_{t_k \leq s \leq t_{k + 1}}$ is a stochastic process satisfying
\begin{align}
\label{eq:process-H}
\begin{split}
    \dd H_t^{\tau} &= \frac{1}{\sqrt{2}} \Big( \cF(t, Y_t) - \overline \cF\big(t, Y_{t_k}, (H_s^{\tau})_{t_k \leq s \leq t}\big) \Big) \dd t + \dd B_t, \qquad t_k + \tau \leq t \leq t_{k + 1} \\
    (H_{s}^{\tau})_{t_k \leq s \leq t_k + \tau} &= (B_s - B_{t_k})_{t_k \leq s \leq t_k + \tau}. 
\end{split}
\end{align}
%
%
\end{lemma}
\noindent 
Note that the distribution of $(H_s^{\tau})_{t_k \leq s \leq t}$ depends on the value of $\tau \in (0, \Delta_k)$. 
%
%
The proof of \cref{lemma:KL-barX-X} is deferred to Appendix~\ref{sec:proof-lemma:KL-barX-X}. 
The existence and the uniqueness of process \eqref{eq:process-H} are also established therein.

In the next lemma, we show the proximity of the process $(H_s^{\tau})_{t_k \leq s \leq t}$ is and the Brownian motion increment process $(B_s - B_{t_k})_{t_k \leq s \leq t}$. The proof of this lemma is postponed to Appendix \ref{sec:proof-lemma:upper-bound-H}.
\begin{lemma}
\label{lemma:upper-bound-H}
    Consider any $\tau \in (0, \Delta_k)$. Under Assumptions \ref{assumption:moments} and \ref{assumption:step-size}, we have the following upper bounds: 
    \begin{align*}
    & \EE\Big[ \sup_{t_k \leq t \leq t_{k + 1}} \big\|B_t - B_{t_k} - H_t^{\tau} \big\|_2^2 \Big] \\
    &\qquad\lesssim  \sigma_{T - t_{k + 1}}^{-2} \Delta_k^4 d + (\sigma_{T - t_k}^{-2} + \lambda_{T - t_k}^2 \sigma_{T - t_k}^{-4})^2 \Delta_k^3 d^3 + \lambda_{T - t_{k + 1}}^4 \sigma_{T - t_{k + 1}}^{-12} \Delta_k^4 d^3 + \lambda_{T - t_{k + 1}}^2 \sigma_{T - t_{k + 1}}^{-8} \Delta_k^4 d, \\
    & \EE\Big[ \sup_{t_k \leq t \leq t_{k + 1}} \big\|B_t - B_{t_k} - H_t^{\tau}\big\|_2^4 \Big] \\
    &\qquad \lesssim  \sigma_{T - t_{k + 1}}^{-4} \Delta_k^8 d^2 + (\sigma_{T - t_k}^{-2} + \lambda_{T - t_k}^2 \sigma_{T - t_k}^{-4})^4 \Delta_k^6 d^6 + \lambda_{T - t_{k + 1}}^8 \sigma_{T - t_{k + 1}}^{-24} \Delta_k^8 d^6 + \lambda_{T - t_{k + 1}}^4 \sigma_{T - t_{k + 1}}^{-16} \Delta_k^8 d^2. 
\end{align*}
\end{lemma}
%

With \cref{lemma:KL-barX-X,lemma:upper-bound-H} in mind, 
we can readily develop a more concise upper bound on $\KL (Q_{T - \delta} \parallel \overline Q_{T - \delta})$, as stated below. 
The proof of this lemma can be found in Appendix \ref{sec:proof-lemma:Q-barQ}. 
\begin{lemma}
\label{lemma:Q-barQ}
    Under Assumptions \ref{assumption:moments} and \ref{assumption:step-size}, it holds that 
\begin{align*}
 \KL (Q_{T - \delta} \parallel \overline Q_{T - \delta}) \lesssim \,  d^3  \kappa^2 T  + d^7  \kappa^3 (\delta^{-1} + T). 
\end{align*}
\end{lemma}
%


%
%
%
%
%
%

\paragraph{Step 3: bounding the effect of score estimation errors.}
We still need to take into account the impact of the score estimation error. 
In this regard, we recall process \eqref{eq:hat-X-tk}, 
%
%
and denote by $\wh Q^{\dis}_{T - \delta}$ (resp.~$Q^{\dis}_{T - \delta}$ ) the distribution of $(\wh Y_{t_k})_{0 \leq k \leq K}$ (resp.~$(Y_{t_k})_{0 \leq k \leq K}$). 
The next lemma attempts to upper bound $\KL(Q^{\dis}_{T - \delta} \parallel \wh Q^{\dis}_{T - \delta})$; its proof is deferred to Appendix~\ref{sec:proof-lemma:discretized-KL}. 
\begin{lemma}
\label{lemma:discretized-KL}
	Suppose that Assumptions \ref{assumption:moments}, \ref{assumption:step-size} and \ref{assumption:score} hold, and that both processes $(\widehat{Y}_t)$ and $(Y_t)$ are initialized to $Y_0 \overset{\mathrm{d}}{=} \wh Y_0 \sim q_T$. 
 Then, it holds that
\begin{align*}
		& \KL(Q^{\dis}_{T - \delta} \parallel \wh Q^{\dis}_{T - \delta}) 
  \lesssim  d^3  \kappa^2 T  + d^7  \kappa^3 (\delta^{-1} + T) +  \sum_{k = 0}^{K - 1}  \frac{\eps_{\mathsf{score},k}^{1 / 2} \kappa^{1/2}\sigma_{T - t_k} d}{\lambda_{T - t_k}^{1 / 2}}.  
\end{align*}
\end{lemma}
%
%

\paragraph{Step 4: determining the impact of initialization errors.}
In practice, we typically have no access to $q_T$, and a common strategy is to replace it with $\cN(0, I_d)$.  
Consider process \eqref{eq:hat-X-tk}, but with initialization $\wh Y_0 \sim \cN(0, I_d)$ instead of $\wh Y_0 \sim q_T$. 
In this case, we denote the distribution of $\wh Y_{T - \delta}$ by $p_{\mathsf{output}}$. 
With $\wh Y_0 \sim \cN(0, I_d)$, we denote the distribution of $(\wh Y_{t_k})_{0 \leq k \leq K}$ that follows the update rule \eqref{eq:alpha-form} by $\wh P_{T - \delta}^{\dis}$, in contrast to $\wh Q_{T - \delta}^{\dis}$ which is the distribution of the same process with $\wh Y_0 \sim q_T$.
Note that for $(y_0, y_1, \cdots, y_K) \in \RR^{d(K + 1)}$, 
\begin{align*}
    \frac{\dd Q_{T - \delta}^{\dis}(y_0, \cdots, y_K)}{\dd \wh P_{T - \delta}^{\dis}(y_0, \cdots, y_K)} = \frac{\dd Q_{T - \delta}^{\dis}(y_0, \cdots, y_K)}{\dd \wh Q_{T - \delta}^{\dis}(y_0, \cdots, y_K)} \cdot \frac{\dd \wh Q_{T - \delta}^{\dis}(y_0, \cdots, y_K)}{\dd \wh P_{T - \delta}^{\dis}(y_0, \cdots, y_K)}  = \frac{\dd Q_{T - \delta}^{\dis}(y_0, \cdots, y_K)}{\dd \wh Q_{T - \delta}^{\dis}(y_0, \cdots, y_K)} \cdot \frac{\dd q_T(y_0)}{\dd \pi_d(y_0)}, 
\end{align*}
where $\pi_d$ represents the distribution of a $d$-dimensional standard Gaussian random vector. 
Using the above distribution, we obtain 
\begin{align}
\label{eq:KL-final-decomp}
    \KL(q_{\delta} \parallel p_{\mathsf{output}}) \leq \KL\big(Q_{T - \delta}^{\dis} \parallel \wh P_{T - \delta} \big) = \KL\big(Q_{T - \delta}^{\dis} \parallel \wh Q_{T - \delta}^{\dis}\big) + \KL(q_T \parallel \pi_d).
\end{align}
%
Further, recall that $q_T$ has the same distribution as $\lambda_T \theta + \sigma_T g$, where $(\theta, g) \sim q_0 \otimes \cN(0, I_d)$. Hence,  the data processing inequality tells us that 
\begin{align}
    \KL(q_T \parallel \pi_d) \leq \KL\big(q_0 \otimes q_T \parallel q_0 \otimes \cN(0, I_d)\big) = \frac{1}{2} \left( -d \log \sigma_T^2 - d + d \sigma_T^2 + e^{-2T} \EE_{\theta \sim q_0}[\|\theta\|_2^2] \right) \lesssim d e^{-2T}.
    \label{eq:KL-init}
\end{align}
%
%
%
Substituting the above upper bound \cref{eq:KL-init} and \cref{lemma:discretized-KL} into \cref{eq:KL-final-decomp}, we arrive at
\begin{align*}
    \KL(q_{\delta} \parallel p_{\mathsf{output}}) \lesssim d^3  \kappa^2 T  + d^7  \kappa^3 (\delta^{-1} + T) +  \sum_{k = 0}^{K - 1}  \frac{\eps_{\mathsf{score},k}^{1 / 2} \kappa^{1/2}\sigma_{T - t_k} d}{\lambda_{T - t_k}^{1 / 2}} + d e^{-2T}
\end{align*}
as claimed. 

\section{Numerical experiments}
\label{sec:experiments}

In this section, we illustrate the practical performance of \cref{alg:stochastic-RK} on various image generation tasks. 
For benchmarking, we resort to the original DDPM sampling scheme \citep{ho2020denoising} along with the SDE-based acceleration method proposed in \cite{li2024accelerating}, ensuring that all methods adopt the same pre-trained score estimates.

More specifically, we utilize the pre-trained score functions from \cite{nichol2021improved} and focus on two datasets: ImageNet-64 \citep{chrabaszcz2017downsampled} and CIFAR-10 \citep{krizhevsky2009learning}.
Note that we have not attempted to optimize the generative modeling performance with additional techniques (e.g., employing better score functions or training with higher quality datasets), as our primary goal is to evaluate the effectiveness of the proposed acceleration method.
Our approach is compatible with a variety of diffusion model codebases and datasets, where we anticipate observing similar acceleration effects. 

In our experiments, we compare the Fréchet inception distance (FID) \citep{heusel2017gans} of the images generated by the vanilla DDPM, the SDE acceleration method proposed in \cite{li2024accelerating}, and our proposed method (the Stochastic Runge-Kutta method). 
FID quantifies the similarity between the distribution of the generated images and the target distribution, with lower FID values indicating greater similarity.
For each method and step size combination, we generate $10^4$ images and compute the corresponding FID. The numerical results are summarized below. 

\paragraph{CIFAR-10.} 
Figure \ref{image:cifar} presents the simulation results for the CIFAR-10 dataset. 
The left panel shows images generated by our method and the vanilla DDPM, while the right panel illustrates the evolution of FID across different NFEs, ranging from 10 to 100. 
Note that here NFE is equal to the number of diffusion steps, as each step requires only one score evaluation for all methods. The generated images suggest that our method produces less noisy outputs. Moreover, our method consistently outperforms the other two methods in terms of FID across all step sizes.

\begin{figure}[htbp]
  \centering
  \begin{minipage}[b]{0.4\textwidth}
    \includegraphics[width=\textwidth]{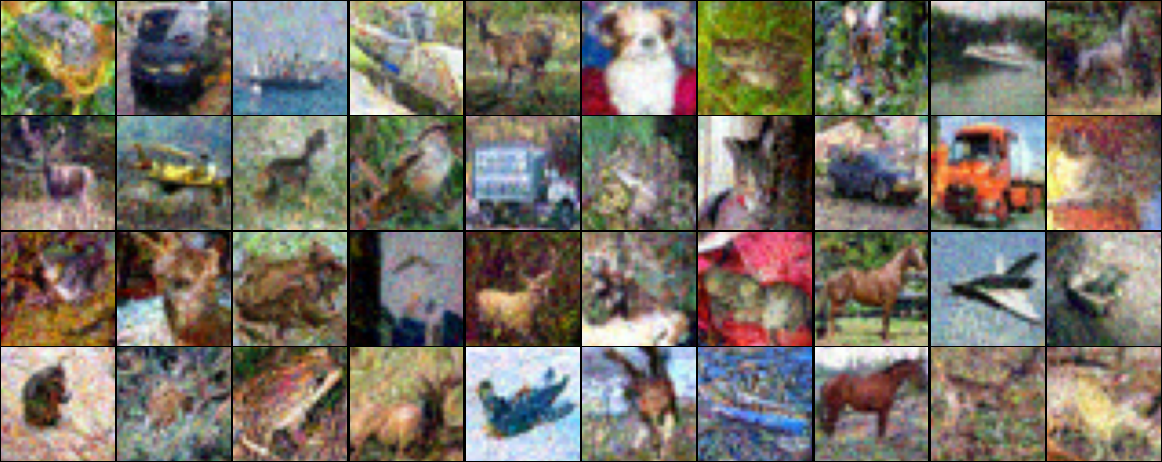}
    \caption*{Images generated using vanilla DDPM.} 
    \includegraphics[width=\textwidth]{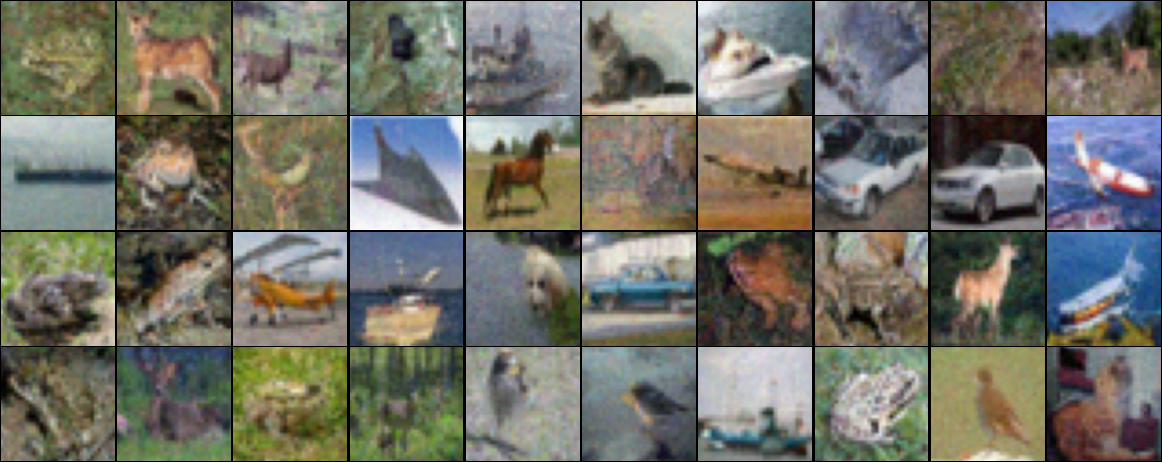}
    \caption*{Images generated using our method.}
  \end{minipage}
  \hspace{1cm}
  \begin{minipage}[b]{0.515\textwidth}
    \includegraphics[width=\textwidth]{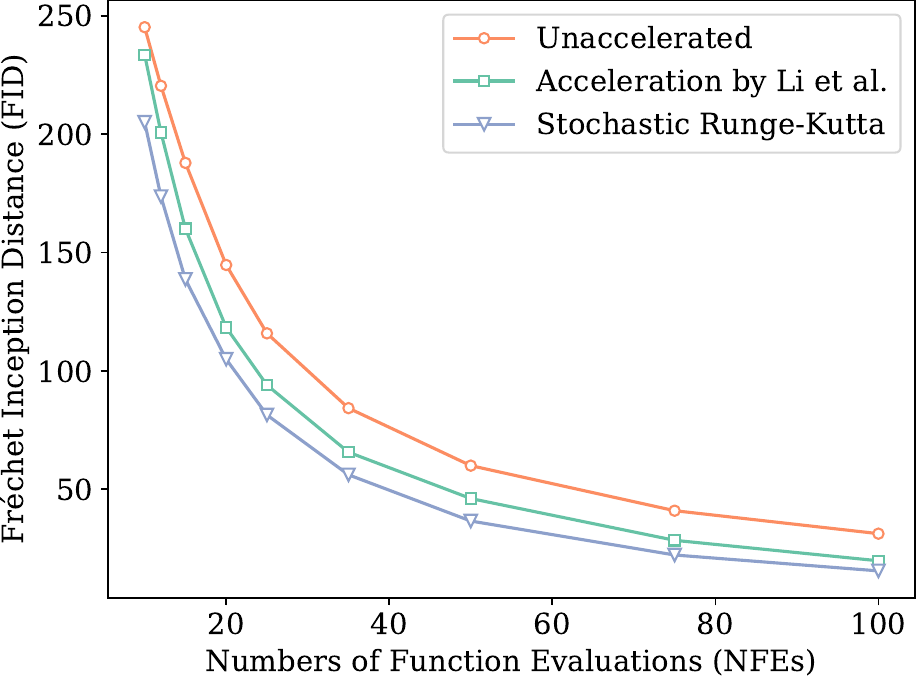}
  \end{minipage}
  \caption{Simulation results using pre-trained score functions for the CIFAR-10 dataset. The left panel shows images generated by the vanilla DDPM and our method with 35 NFEs. The right panel plots the FID scores for all three methods across different NFEs. }
  \label{image:cifar}
\end{figure}

\paragraph{ImageNet-64.}
Figure \ref{image:imagenet} shows the simulation results for the ImageNet-64 dataset, where we observe similar improvements as what happens for the CIFAR-10 dataset. 

\begin{figure}[htbp]
  \centering
  \begin{minipage}[b]{0.4\textwidth}
    \includegraphics[width=\textwidth]{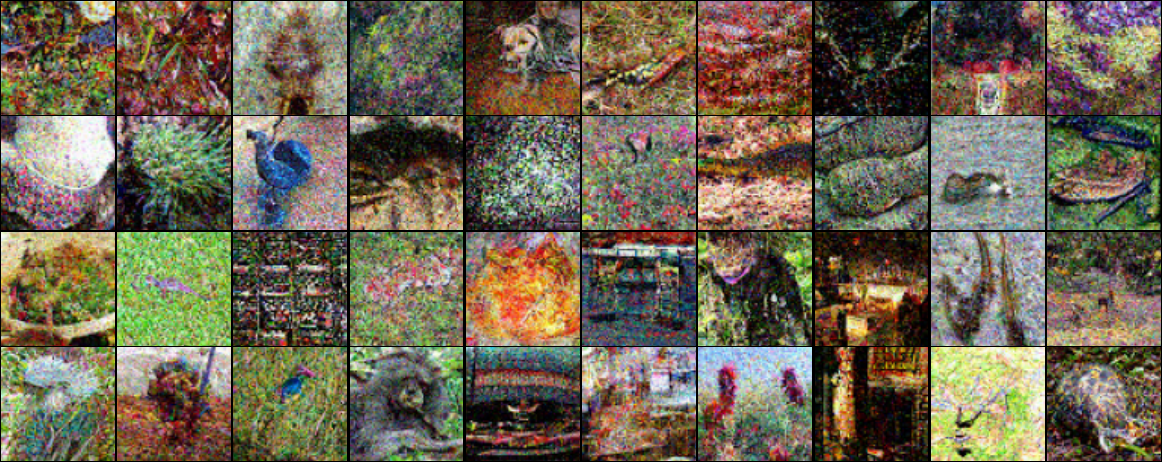}
    \caption*{Images generated using vanilla DDPM.} 
    \includegraphics[width=\textwidth]{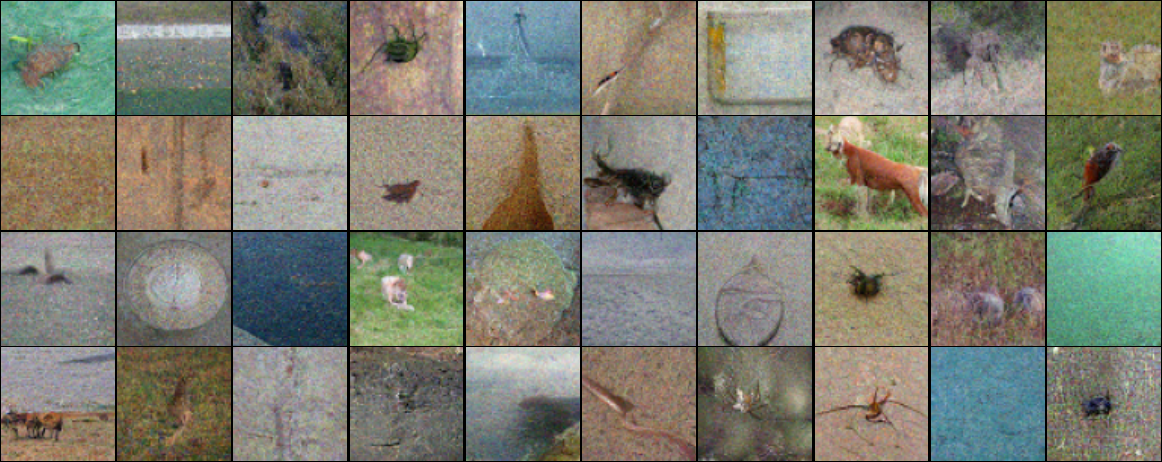}
    \caption*{Images generated using our method.}
  \end{minipage}
  \hspace{1cm}
  \begin{minipage}[b]{0.515\textwidth}
    \includegraphics[width=\textwidth]{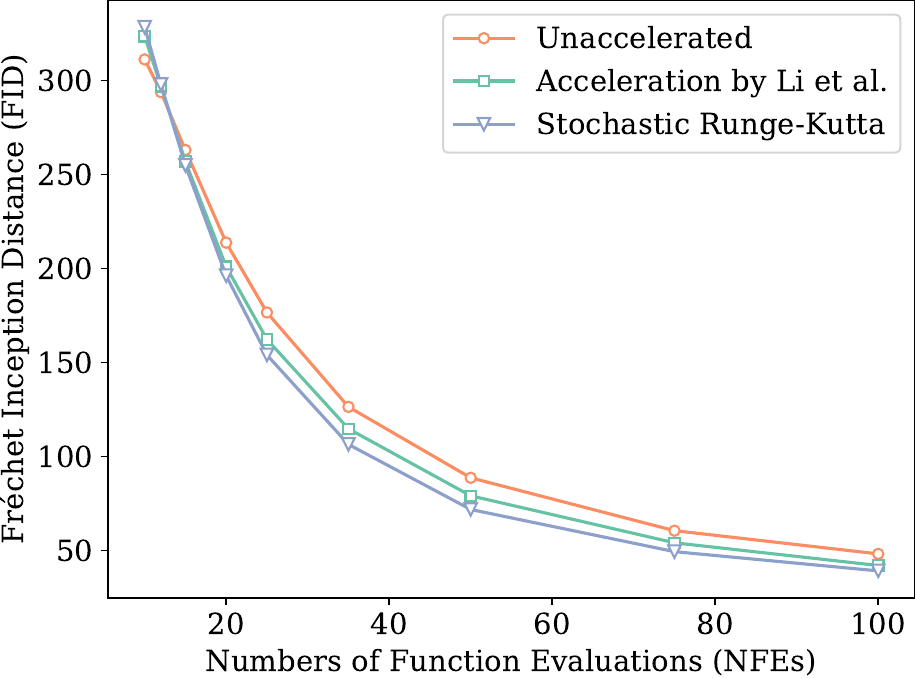}
  \end{minipage}
  \caption{Simulation results using pre-trained score functions for the ImageNet-64 dataset. The left panel shows images generated by the vanilla DDPM and our method with 35 NFEs. The right panel plots the FID scores for all three methods across different NFEs. }
  \label{image:imagenet}
\end{figure}

\section{Discussion}
\label{sec:discussion}

In this paper, we have made progress in provably speeding up SDE-based diffusion samplers. In comparison to prior results, the convergence guarantees of our accelerated algorithm enjoy improved dimension-dependency, shedding light on the advantages of the stochastic Runge-Kutta approach. 
Remarkably, our algorithm paves the way for designing even higher-order SDE-based diffusion solvers, the advantages of which will be explored in future research.

Moving forward, there are plenty of directions that are worth pursuing. 
For instance, the dependency on the dimension $d$ and the score estimation error remains sub-optimal, and more refined analyses are needed in order to tighten our result. 
Also, as mentioned previously, establishing sharp TV-type upper bound for our proposed sampler (as in \cite{li2024d} for DDPM) would be an interesting direction and call for new techniques, as the Girsanov-type arguments might not be applicable for analyzing the TV-distance. 
Furthermore, the recent work \cite{li2024adapting} has demonstrated the remarkable capability of DDPM in adapting to unknown low-dimensional structure; whether this appealing feature is inherited by our accelerated stochastic sampler is worth investigating. 
Finally, it would be important to develop fast and principled diffusion-based samplers that allow one to sample with guidance in a provably efficient manner (see, e.g., \cite{wu2024theoretical,chidambaram2024does}).


\section*{Acknowledgements}

Y.~Chen is supported in part by the Alfred P.~Sloan Research Fellowship, the AFOSR grant FA9550-22-1-0198, the ONR grant N00014-22-1-2354,  and the NSF grant CCF-2221009. 
Y. Wei is supported in part by the NSF grants DMS-2147546/2015447, CAREER award DMS-2143215, CCF-2106778, CCF-2418156 and the Google Research Scholar Award.
The authors gratefully acknowledge Timofey Efimov for his generous assistance with the numerical experiments.


\appendix 

\section{Technical lemmas}
\label{sec:technical-lemmas}

We collect in this section a couple of technical lemmas that are useful in establishing our main results.

\begin{lemma}
\label{lemma:BM-integral-cov}
    Denote by $(W_t)_{t \geq 0}$ a standard Brownian motion in $\RR^d$. 
    Then for all $t_k \leq t < t_{k + 1}$, the covariance matrices of the following vectors are given by
    \begin{align*}
        & \Cov\left[ \int_0^{t - t_k} e^{t - t_k - r}  (W_{t_k + r} - W_{t_k}) \dd r \right] = \big[  e^{2t - 2t_k} / 2 - 2 e^{t - t_k} + (t - t_k + 3 / 2) \big] \cdot  I_d \eqqcolon f_1(t-t_k)\cdot I_d, \\
        & \Cov \left[ \int_0^{t - t_k} e^{t - t_k - r}  (W_{t_k + r} - W_{t_k}) \dd r,\, W_t - W_{t_k} \right] =  \big[e^{t - t_k} - t + t_k - 1 \big] \cdot I_d, \\
        & \Cov\left[ \int_0^{t - t_k} e^{t - t_k - r} \dd W_{t_k + r} \right] = \big[e^{2(t - t_k)} / 2 - 1 / 2\big] \cdot I_d
        \eqqcolon f_2(t-t_k)\cdot I_d, \\
        & \Cov\left[ \int_0^{t - t_k} e^{t - t_k - r}  (W_{t_k + r} - W_{t_k}) \dd r,  \int_0^{t - t_k} e^{t - t_k - r} \dd W_{t_k + r}  \right] = \big[  e^{2(t - t_k)} / 2 - e^{t - t_k} + 1 / 2 \big] \cdot I_d
        \eqqcolon f_3(t-t_k)\cdot I_d. 
    \end{align*}
    Here, we define, for notational simplicity, the following functions: for any $\Delta > 0$, let 
    \begin{align}
    \label{eq:covariance_functions}
    \begin{split}
        & f_1(\Delta) = e^{2\Delta} / 2 - 2 e^{\Delta} + \Delta + 3 / 2, \\
        & f_2(\Delta) = e^{2\Delta} / 2 - 1 / 2, \\
        & f_3(\Delta) = e^{2\Delta} / 2 - e^{\Delta} + 1 / 2. 
    \end{split}
    \end{align}
\end{lemma}
\begin{proof}[Proof of \cref{lemma:BM-integral-cov}]
    For $t_k \leq t < t_{k + 1}$, set 
    \begin{align*}
        \cH(t) = \Cov\left[ \int_0^{t - t_k} e^{t - t_k - r}  (W_{t_k + r} - W_{t_k}) \dd r \right] = \EE\left[ \left( \int_0^{t - t_k} e^{t - t_k - r}  (W_{t_k + r} - W_{t_k}) \dd r \right)^{\otimes 2} \right]. 
    \end{align*}
    %
Observe that $\cH(t_k) = 0_{d \times d}$. Taking the derivative of $\cH(t)$ with respect to $t$, we reach 
    \begin{align*}
        \cH'(t) &=  \EE\left[ \int_0^{t - t_k} e^{t - t_k - r}  (W_{t_k + r} - W_{t_k}) \dd r \otimes \Big( W_t - W_{t_k} + \int_0^{t - t_k} e^{t - t_k - r} (W_{t_k + r} - W_{t_k}) \dd r \Big) \right] \\
        & ~+ \EE\left[  \Big( W_t - W_{t_k} + \int_0^{t - t_k} e^{t - t_k - r} (W_{t_k + r} - W_{t_k}) \dd r \Big) \otimes \int_0^{t - t_k} e^{t - t_k - r}  (W_{t_k + r} - W_{t_k}) \dd r  \right] \\
        &= 2 \cH(t) + \bigg(2 \int_0^{t - t_k} e^{t - t_k - r} r \, \dd r \bigg) I_d  \\
        &=  2 \cH(t) + 2(e^{t - t_k} - t + t_k - 1)\cdot I_d. 
    \end{align*}
    Solving the above ordinary differential equation yields 
    \begin{align*}
        \cH(t) = \Big( \frac{1}{2} e^{2t - 2t_k} - 2 e^{t - t_k} + (t - t_k + 3 / 2) \Big) I_d, 
    \end{align*}
    thereby completing the proof of the first advertised identity. 

    As for the second  claimed identity, we make the observation that
    \begin{align*}
        \Cov \left[ \int_0^{t - t_k} e^{t - t_k - r}  (W_{t_k + r} - W_{t_k}) \dd r,\, W_t - W_{t_k} \right] = \bigg(\int_0^{t - t_k} e^{t - t_k - r} r \dd r \bigg)\, I_d = \big(e^{t - t_k} - t + t_k - 1\big) \cdot I_d. 
    \end{align*}
    With regards to the third claimed identity, we have 
    \begin{align*}
        \Cov\left[ \int_0^{t - t_k} e^{t - t_k - r} \dd W_{t_k + r} \right] = \bigg(\int_0^{t - t_k} e^{2(t - t_k - r)} \dd r\bigg)\, I_d = \frac{1}{2} \big(e^{2(t - t_k)} - 1\big)I_d. 
    \end{align*}
    To prove the last result, it is seen that  
    \begin{align*}
        & \Cov\left[ \int_0^{t - t_k} e^{t - t_k - r}  (W_{t_k + r} - W_{t_k}) \dd r,   \int_0^{t - t_k} e^{t - t_k - r} \dd W_{t_k + r} \right] =  \bigg( \int_0^{t - t_k} \int_0^{t - t_k} e^{2(t - t_k) - r - s} \mathbbm{1}_{s \leq r} \dd s \dd r \bigg) I_d \\
         &\qquad =\bigg(\int_0^{t - t_k} \int_0^r e^{2(t - t_k) - r - s} \dd s \dd r \bigg) I_d = \Big(\frac{1}{2} e^{2(t - t_k)} - e^{t - t_k} + \frac{1}{2}\Big)  I_d. 
    \end{align*}
    The proof is thus complete. 
\end{proof}

\begin{lemma}
\label{lemma:alpha-M}
    Consider any random object $M \in \RR^{d \times d}$ and any random variable $\alpha \in \RR$,  as well as a filtration $\cF$. Then, it holds that 
    \begin{align*}
        \EE\left[ \big\| \EE[\alpha M \mid \cF] \big\|_{\mathrm{F}}^4 \right] \leq 
        \sqrt{ \EE[\alpha^8] \cdot \EE\big[\|M\|_{\mathrm{F}}^8\big] }. 
    \end{align*}
\end{lemma}
%
\begin{proof}[Proof of \cref{lemma:alpha-M}]
    It follows from Cauchy–Schwarz  that $\|\EE[\alpha M \mid \cF]\|_{\mathrm{F}}^2 \leq \EE[\alpha^2 \mid \cF] \EE[\|M\|_{\mathrm{F}}^2 \mid \cF]$, which in turn yields
    \begin{align*}
        \EE\left[ \big\| \EE[\alpha M \mid \cF] \big\|_{\mathrm{F}}^4 \right] \leq & \EE\left[\EE[\alpha^2 \mid \cF]^2 \EE\big[\|M\|_{\mathrm{F}}^2 \mid \cF \big]^2 \right] \leq  \EE\left[\EE[\alpha^2 \mid \cF]^4 \right]^{1/2} \EE\left[\EE[\|M\|_{\mathrm{F}}^2 \mid \cF]^4   \right]^{1/2} \\
        \leq & \sqrt{ \EE[\alpha^8] \EE\big[\|M\|_{\mathrm{F}}^8\big]}. 
    \end{align*}
\end{proof}

\begin{lemma}
\label{lemma:W-B-diff}
    For any $0 \leq t_0 < t_0 + t < T$, 
    the reverse process $(Y_t)_{0\leq t\leq T}$ (cf.~\cref{eq:reverse-OU}) obeys
    \begin{align*}
        & \EE\big[ \big\| Y_{t_0 + t} - Y_{t_0} - \sqrt{2}B_{t_0 + t} + \sqrt{2}B_{t_0} \big\|_2^2 \big] \lesssim d \sigma_{T - t_0 - t}^{-2}t^2, \\
        & \EE\big[ \big\| Y_{t_0 + t} - Y_{t_0} - \sqrt{2}B_{t_0 + t} + \sqrt{2}B_{t_0} \big\|_2^4 \big] \lesssim d^2 \sigma_{T - t_0 - t}^{-4}t^4. 
    \end{align*}
\end{lemma}
\begin{proof}[Proof of \cref{lemma:W-B-diff}]
    For notational convenience, we define, for any $t \geq 0$,  $\cD(t) \coloneqq \EE\big[ \| Y_{t_0 + t} - Y_{t_0} - \sqrt{2}B_{t_0 + t} + \sqrt{2}B_{t_0} \|_2^2 \big]$, which clearly obeys $\cD(0) = 0$. 
     It then follows that 
    \begin{align*}
        |\cD'(t)| &=  2 \Big| \EE\Big[ \big\langle Y_{t_0 + t} + 2s(t_0 + t, Y_{t_0 + t}), Y_{t_0 + t} - Y_{t_0} - \sqrt{2} B_{t_0 + t} + \sqrt{2} B_{t_0} \big\rangle \Big] \Big| \\
        &\leq  2 \EE\big[\big\|Y_{t_0 + t} + 2s(t_0 + t, Y_{t_0 + t})\big\|_2^2\big]^{1/2} \cD(t)^{1/2} \lesssim \sqrt{d}\, \sigma_{T - t_0 - t}^{-1} \cD(t)^{1/2}.
    \end{align*}
    Here, the first inequality arises from the Cauchy-Schwarz inequality, and the second inequality holds since
\begin{align*}
\EE\big[\big\| Y_{t_{0}+t}+2s(t_{0}+t,Y_{t_{0}+t})\big\|_{2}^{2}\big] & \lesssim\EE\big[\big\| Y_{t_{0}+t}\big\|_{2}^{2}\big]+\EE\big[\big\| s(t_{0}+t,Y_{t_{0}+t})\big\|_{2}^{2}\big]=\EE\big[\big\| X_{T-t_{0}-t}\big\|_{2}^{2}\big]+\EE\big[\big\| s(t_{0}+t,X_{T-t_{0}-t})\big\|_{2}^{2}\big]\\
 & \lesssim d+d\sigma_{T-t_{0}-t}^{-2}\asymp d\sigma_{T-t_{0}-t}^{-2}, 
\end{align*}
where the last line invokes \citet[Lemma~6]{benton2024nearly} as well as Assumption~\ref{assumption:moments} with $R=\sqrt{d}$. 
    In view of the ODE comparison theorem, we see that $\cD(t) \lesssim d \sigma_{T - t_0 - t}^{-2}t^2 $, 
    thus establishing the first result of this lemma.

    Similarly, we define $\overline \cD(t) \coloneqq \EE[ \| Y_{t_0 + t} - Y_{t_0} - \sqrt{2}B_{t_0 + t} + \sqrt{2}B_{t_0} \|_2^4]$ with $\overline \cD(0) = 0$. Taking the derivative of $\overline\cD$, we see from H\"older's inequality that 
    \begin{align*}
         |\overline\cD'(t)| 
        &=  4 \Big| \EE\Big[ \big\langle Y_{t_0 + t} + 2s(t_0 + t, Y_{t_0 + t}), Y_{t_0 + t} - Y_{t_0} - \sqrt{2} B_{t_0 + t} + \sqrt{2} B_{t_0} \big\rangle \big\| Y_{t_0 + t} - Y_{t_0} - \sqrt{2}B_{t_0 + t} + \sqrt{2}B_{t_0} \big\|_2^2 \Big] \Big| \\
        &\leq  4 \EE\big[ \big\|Y_{t_0 + t} + 2 s(t_0 + t, Y_{t_0 + t})\big\|_2^4 \big]^{1/4} \EE\big[ \big\| Y_{t_0 + t} - Y_{t_0} - \sqrt{2}B_{t_0 + t} + \sqrt{2}B_{t_0} \big\|_2^4\big]^{3/4} \\
        &=  4 \EE\big[ \big\|Y_{t_0 + t} + 2 s(t_0 + t, Y_{t_0 + t})\big\|_2^4 \big]^{1/4} \overline \cD(t)^{3/4} \\
        &\lesssim  \sqrt{d} \sigma_{T - t_0 - t}^{-1} \overline \cD(t)^{3/4}, 
    \end{align*}
    which in turn implies that $\overline\cD(t) \lesssim d^2 \sigma_{T - t_0 - t}^{-4} t^4$. This concludes the proof. 
\end{proof}

%

\section{Properties of the score function}

In this section, we gather several useful properties of the ground-truth score function $\{s(t,\cdot)\}$.  
Recall that the true score functions admit the following expression 
\begin{align}
\label{eq:parametric-form}
    s(T - t, x) = \frac{\lambda_t m(t, x) - x}{\sigma_t^2}, \qquad \text{with }
    m(t, x) = \mathop{\EE}\limits_{(\theta, g) \sim q_0 \otimes \cN(0, I_d)} \big[\theta \mid \lambda_t \theta + \sigma_t g = x \big],
\end{align}
where we recall that $\lambda_t = e^{-t}$ and $\sigma_t = \sqrt{1 - e^{-2t}}$. 
In addition, we find it convenient to define the function
    \begin{align}
    \label{eq:B2-cF}
        f_0(\theta, x, t) \coloneqq  \frac{\lambda_t^2}{\sigma_t^4} \|\theta\|_2^2 - \frac{\lambda_t + \lambda_t^3}{\sigma_t^4} \langle x, \theta \rangle. 
    \end{align}

\begin{lemma}
\label{lemma:partial-t-score}
Recall that $q_t$ is the distribution of $\lambda_t \theta + \sigma_t g$ for $(\theta, g) \sim q_0 \otimes \cN(0, I_d)$. 
     Under Assumption \ref{assumption:moments},   
    \begin{align*}
        \partial_t \nabla_x \log q_t (x) = - \frac{\lambda_t(2 - \sigma_t^2)}{\sigma_t^4} m(t, x) + \frac{2 \lambda_t^2}{\sigma_t^4} x  - \frac{\lambda_t^2(2 - \sigma_t^2)}{\sigma_t^6} C_t(x) x + \frac{\lambda_t^3}{\sigma_t^6} v_t(x)
    \end{align*}
    holds for any $x \in \RR^d$ and $t \in (0, T]$, 
    where 
    \begin{align*}
         m(t, x) &= \EE[\theta \mid \lambda_t \theta + \sigma_t g = x] \in \RR^d, \\
         C_t(x) &= \Cov[\theta \mid \lambda_t \theta + \sigma_t g = x] \in \RR^{d \times d}, \\
         v_t(x) &= \EE\big[\|\theta\|_2^2 \big(\theta - m(t, x)\big) \mid \lambda_t \theta + \sigma_t g = x\big] \in \RR^d. 
    \end{align*}
\end{lemma}

\begin{proof}[Proof of \cref{lemma:partial-t-score}]

Recall from \cref{eq:parametric-form} that
\begin{align*}
    \nabla_x \log q_t(x) = s(T - t, x) = \frac{\lambda_t m(t, x) - x}{\sigma_t^2}, 
\end{align*}
where $m(t, x) = \EE[\theta \mid \lambda_t \theta + \sigma_t g = x]$ for $(\theta, g) \sim q_0 \otimes \cN(0, I_d)$. 
To begin with, it is easily seen that
\begin{align}
\label{eq:partial-t-constants}
\partial_t(\sigma_t^{-2}) = -2 \lambda_t^2 \sigma_t^{-4}, \qquad \partial_t(\lambda_t \sigma_t^{-2}) = - \lambda_t(2 - \sigma_t^2) \sigma_t^{-4}. 
\end{align}
Next, we turn to $\partial_t m(t, x)$,  
and observe that 
\begin{align*}
    m(t, x) = \frac{\int \theta \exp (\lambda_t \sigma_t^{-2} \langle x, \theta \rangle - \lambda_t^2 \sigma_t^{-2} \|\theta\|_2^2 / 2) q_0 (\dd \theta)}{\int \exp (\lambda_t \sigma_t^{-2} \langle x, \theta \rangle - \lambda_t^2 \sigma_t^{-2} \|\theta\|_2^2 / 2) q_0 (\dd \theta)}. 
\end{align*}
Given our bounded fourth-moment assumption on $q_0$, we can readily apply Fubini's theorem and exchange the order of differentiation and integration, which leads to the following equation
\begin{align}
\label{eq:partial-t-m-proof}
    \partial_t m(t, x) = & \frac{\lambda_t^2}{\sigma_t^4} \EE\left[ \|\theta\|_2^2 \big(\theta - m(t, x)\big) \mid \lambda_t \theta + \sigma_t g = x \right] - \frac{\lambda_t (2 - \sigma_t^2)}{\sigma_t^4} \EE\big[  \langle x, \theta \rangle \big(\theta - m(t, x)\big) \mid \lambda_t \theta + \sigma_t g = x \big] \nonumber \\
    = & \frac{\lambda_t^2}{\sigma_t^4} v_t(x) - \frac{\lambda_t(2 - \sigma_t^2)}{\sigma_t^4} C_t(x) x. 
\end{align}
The proof can thus be completed by putting  together \cref{eq:partial-t-constants,eq:partial-t-m-proof}. 
\end{proof}

\begin{lemma}
\label{lemma:score-properties}
    Suppose that the target distribution $q_0$ has finite fourth moment. Then, for all $t \in (0, T]$ and $x \in \RR^d$, the following identities hold:  
    \begin{enumerate}
        \item $\nabla_x s(T - t, x) = - \sigma_t^{-2} I_d + \lambda_t^2 \sigma_t^{-4} \Cov[\theta \mid \lambda_t \theta + \sigma_t g = x]$.
        
        \item $\nabla_x^2 s(T - t, x) = \lambda_t^3 \sigma_t^{-6} \EE\big[\big(\theta - m(t, x) \big) \otimes \big(\theta - m(t, x)\big) \otimes \big(\theta - m(t, x)\big) \mid \lambda_t \theta + \sigma_t g = x\big]$. 
        
        \item $\partial_t s(T - t, x) = - (\lambda_t + \lambda_t^3) \sigma_t^{-4} m(t, x) + 2 \lambda_t^2 \sigma_t^{-4} x + \lambda_t \sigma_t^{-2} \EE\big[(\theta - m(t, x))f_0(\theta, x, t) \mid \lambda_t \theta + \sigma_t g = x\big]$, with $f_0(\theta, x, t)$ defined in \cref{eq:B2-cF}. This expression can also be equivalently rewritten as
        \begin{align*}
            \partial_t s(T - t, x) =& - \frac{\lambda_t}{\sigma_t^2} \EE\big[\theta \mid \lambda_t \theta + \sigma_t g = x\big] + \frac{2\lambda_t^2}{\sigma_t^3} \EE\big[g \mid \lambda_t \theta + \sigma_t g = x\big] \\
            & + \frac{\lambda_t}{\sigma_t^2} \EE\big[\big(\theta - m(t, x)\big)f_0(\theta, x, t) \mid \lambda_t \theta + \sigma_t g = x \big]. 
        \end{align*}
    \end{enumerate}
\end{lemma}
\begin{proof}[Proof of \cref{lemma:score-properties}]
    The first identity has been established in, e.g., \citet[Lemma 5]{benton2024nearly}. 

    To establish the second identity claimed in the lemma, we observe that: from our assumption that $q_0$ has bounded fourth moment,  we can apply Fubini's theorem and exchange the order of differentiation and integration to obtain
    \begin{align*}
        \nabla_x^2 s(T - t, x) &=  \lambda_t^2 \sigma_t^{-4} \nabla_x \Cov\big[\theta \mid \lambda_t \theta + \sigma_t g = x\big] \\ 
        &=  \lambda_t^3 \sigma_t^{-6} \EE\big[\big(\theta - m(t, x)\big) \otimes \big(\theta - m(t, x)\big) \otimes \big(\theta - m(t, x)\big) \mid \lambda_t \theta + \sigma_t g = x\big].
    \end{align*}
    This proves the second point of the lemma. 

    Finally, we prove the third identity claimed in the lemma. Invoking  \cref{lemma:partial-t-score} gives 
    %
    \begin{align*}
        \partial_t s(T - t, x) &=  \partial_t \nabla_x \log q_t(x) = - \frac{\lambda_t(2 - \sigma_t^2)}{\sigma_t^4} m(t, x) + \frac{2 \lambda_t^2}{\sigma_t^4} x  - \frac{\lambda_t^2(2 - \sigma_t^2)}{\sigma_t^6} C_t(x) x + \frac{\lambda_t^3}{\sigma_t^6} v_t(x) \\
        &=  - \frac{\lambda_t + \lambda_t^3}{\sigma_t^4} m(t, x) + \frac{2\lambda_t^2}{\sigma_t^4} x + \frac{\lambda_t}{\sigma_t^2} \EE\big[\big(\theta - m(t, x)\big)f_0(\theta, x, t) \mid \lambda_t \theta + \sigma_t g = x\big],
    \end{align*}
    where 
    $C_{t}(x) \coloneqq \Cov[\theta\mid\lambda_{t}\theta+\sigma_{t}g=x]$ and 
    $v_{t}(x)\coloneqq\EE\big[\|\theta\|_{2}^{2}\big(\theta-m(t,x)\big)\mid\lambda_{t}\theta+\sigma_{t}g=x\big]$ with 
     $(\theta, g) \sim q_0 \otimes \cN(0, I_d)$. 
    The proof is complete by taking advantage of the identity $x = \lambda_t m(t, x) + \sigma_t \EE[g \mid \lambda_t \theta + \sigma_t g = x]$. 
\end{proof}

\begin{lemma}
\label{lemma:partial-t-grad-s}
    Assume that the target distribution $q_0$ has finite fourth moment. Then, for all $t \in (0, T]$ and $x \in \RR^d$, it holds that 
    \begin{align*}
        & \partial_t \nabla_x s(T - t, x) = \frac{2\lambda_t^2}{\sigma_t^4} I_d + \Big( \frac{2\lambda_t^2}{\sigma_t^4} - \frac{4\lambda_t^2}{\sigma_t^6} \Big) \Cov[\theta \mid \lambda_t \theta + \sigma_t g = x] \\
        & \qquad + \frac{\lambda_t^2}{\sigma_t^4} \EE\left[ \big(\theta - m(t, x)\big) \big(\theta - m(t, x) \big)^{\top} \big( f_0(\theta, x, t) - \EE\big [ f_0(\theta, x, t) \mid \lambda_t \theta + \sigma_t g = x\big] \big)  \mid \lambda_t \theta + \sigma_t g = x \right]. 
    \end{align*}
\end{lemma}
\begin{proof}[Proof of \cref{lemma:partial-t-grad-s}]
    Since $q_0$ has bounded fourth moment, by Fubini's theorem we can exchange the order of differentiation and integration. In addition, by the third point of \cref{lemma:score-properties}, we know that $$\partial_t s(T - t, x) = - (\lambda_t + \lambda_t^3) \sigma_t^{-4} m(t, x) + 2 \lambda_t^2 \sigma_t^{-4} x + \lambda_t \sigma_t^{-2} \EE\big[\big(\theta - m(t, x)\big)f_0(\theta, x, t) \mid \lambda_t \theta + \sigma_t g = x\big].$$ 
    As a consequence, 
    we can compute
    \begin{align*}
        & \nabla_x \partial_t s(T - t, x) =   \frac{2\lambda_t^2}{\sigma_t^4} I_d - \frac{2\lambda_t^2 + 2\lambda_t^4}{\sigma_t^6} \Cov[\theta \mid \lambda_t \theta + \sigma_t g = x] \\
        & + \frac{\lambda_t^2}{\sigma_t^4} \EE\big[\big(\theta - m(t, x)\big) \big(\theta - m(t, x)\big)^{\top} \big(f_0(\theta, x, t) - \EE[f_0(\theta, x, t) \mid \lambda_t \theta + \sigma_t g = x ]\big) \mid \lambda_t \theta + \sigma_t g = x\big]
    \end{align*}
    as claimed. 
\end{proof}

\begin{lemma}
\label{lemma:partial-t-hessian-s}
    Assume that the target distribution $q_0$ has finite fourth moment. Then, for all $t \in (0, T]$ and $x \in \RR^d$, it holds that
    \begin{align*}
        \partial_t \nabla_x^2 s(T - t, x) =& -\frac{3(\lambda_t^3 + \lambda_t^5)}{\sigma_t^8} \EE\big[\big(\theta - m(t, x)\big)^{\otimes 3} \mid \lambda_t \theta + \sigma_t g = x\big] - \frac{\lambda_t^3}{2\sigma_t^6} \sum_{\pi \in \mathrm{perm}(3)} M_{\pi} \\
        & + \frac{\lambda_t^3}{\sigma_t^6} \EE\left[ \big(\theta - m(t, x)\big)^{\otimes 3}\big(f_0(\theta, x, t) - \EE[f_0(\theta, x, t) \mid \lambda_t \theta + \sigma_t g = x]\big) \mid \lambda_t \theta + \sigma_t g = x \right],
    \end{align*}
    where for $\pi \in \mathrm{perm}(3)$ and $i_1, i_2, i_3 \in [d]$, we take
    \begin{align*}
        (M_{\pi})_{i_1 i_2 i_3} &=  \EE[(\theta_{i_{\pi(1)}} - m_{i_{\pi(1)}}) (\theta_{i_{\pi(2)}} - m_{i_{\pi(2)}}) \mid \lambda_t \theta + \sigma_t g = x] \\
        &~~\times  \EE[(\theta_{i_{\pi(3)}} - m_{i_{\pi(3)}})(f_0(\theta, x, t) - m_{f_0}) \mid \lambda_t \theta + \sigma_t g = x].
    \end{align*}
    In the above display, 
    \begin{align*}
        m_i = \EE[\theta_i \mid \lambda_t \theta + \sigma_t g = x] \qquad \text{and}\qquad
        m_{f_0} = \EE[f_0(\theta, x, t) \mid \lambda_t \theta + \sigma_t g = x]. 
    \end{align*}
\end{lemma}
\begin{proof}[Proof of \cref{lemma:partial-t-hessian-s}]
	With the bounded fourth-moment assumption on $q_0$ in place, one can apply Fubini's theorem to swap the order of differentiation and integration and obtain
 \begin{align*}
     & \partial_t \nabla_x^2 s(T - t, x) = \nabla_x \partial_t \nabla_x s(T - t, x) \\
     &=  - \frac{2\lambda_t^2 + 2 \lambda_t^4}{\sigma_t^6} \nabla_x \Cov\big[\theta \mid \lambda_t \theta + \sigma_t g = x\big] \\
     &~~~ + \frac{\lambda_t^2}{\sigma_t^4} \nabla_x \EE\left[ \big(\theta - m(t, x)\big) \big(\theta - m(t, x) \big)^{\top} \big( f_0(\theta, x, t) - \EE\big [ f_0(\theta, x, t) \mid \lambda_t \theta + \sigma_t g = x\big] \big)  \mid \lambda_t \theta + \sigma_t g = x \right] \\
     &=  -\frac{3(\lambda_t^3 + \lambda_t^5)}{\sigma_t^8} \EE\big[\big(\theta - m(t, x)\big)^{\otimes 3} \mid \lambda_t \theta + \sigma_t g = x\big] - \frac{\lambda_t^3}{2\sigma_t^6} \sum_{\pi \in \mathrm{perm}(3)} M_{\pi} \\
        & ~~~+ \frac{\lambda_t^3}{\sigma_t^6} \EE\left[ (\theta - m(t, x))^{\otimes 3}(f_0(\theta, x, t) - \EE[f_0(\theta, x, t) \mid \lambda_t \theta + \sigma_t g = x]) \mid \lambda_t \theta + \sigma_t g = x \right]
 \end{align*}
 as claimed.  
\end{proof}

\begin{lemma}
\label{lemma:grad3-s}
	Assume that the target distribution $q_0$ has finite fourth moment. Then, for all $t \in (0, T]$ and $x \in \RR^d$, it holds that
	\begin{align*}
		\nabla_x^3 s(T - t, x) = \frac{\lambda_t^4}{\sigma_t^8} \left( \EE\big[\big(\theta - m(t, x)\big)^{\otimes 4} \mid \lambda_t \theta + \sigma_t g = x\big] - \cT\big(\Cov[\theta \mid \lambda_t \theta + \sigma_t g = x]^{\otimes 2}\big) \right) \in \RR^{d^4}. 
	\end{align*}
 Here, for any tensor $X \in \RR^{d \times d \times d \times d}$, we take $\cT(X) = \sum_{\pi \in \mathrm{perm}(4)} \cT_{\pi}(X) / 8$ and $\cT_{\pi}(X) \in \RR^{d \times d \times d \times d}$, such that $\cT_{\pi}(X)_{i_1 i_2 i_3 i_4} = X_{i_{\pi(1)}i_{\pi(2)}i_{\pi(3)}i_{\pi(4)}}$. 
\end{lemma}
\begin{proof}[Proof of \cref{lemma:grad3-s}]

Let us invoke Fubini's theorem to swap the order of differentiation and integration, thus leading to
\begin{align*}
    \nabla_x^3 s(T - t, x) &= \lambda_t^3 \sigma_t^{-6} \nabla_x \EE[(\theta - m(t, x)) \otimes (\theta - m(t, x)) \otimes (\theta - m(t, x)) \mid \lambda_t \theta + \sigma_t g = x] \\
    &=  \lambda_t^4 \sigma_t^{-8} \cdot \Big( \EE[(\theta - m(t, x))^{\otimes 4} \mid \lambda_t \theta + \sigma_t g = x] - \cT(\Cov[\theta \mid \lambda_t \theta + \sigma_t g = x]^{\otimes 2}) \Big)
\end{align*}
The proof is thus complete. 
\end{proof}

\begin{lemma}
\label{lemma:T2}
    We assume Assumption \ref{assumption:moments}. Recall that $X_{T - t}$ is defined in \cref{eq:reverse-OU} and has marginal distribution $q_t$. Then, for $t \in (0, T]$, it holds that 
    \begin{align*}
        \EE\left[ \big\| \partial_t \nabla_x \log q_t(X_{T - t}) \big\|_2^2\right] \lesssim \frac{d^3}{\sigma_t^6}. 
    \end{align*}
\end{lemma}

\begin{proof}[Proof of \cref{lemma:T2}]
    
Invoking \cref{lemma:partial-t-score}, we obtain the following upper bound: 
    \begin{align}
        & \EE\left[ \big\| \partial_t \nabla_x \log q_t(X_{T - t}) \big\|_2^2 \right] \label{eq:decomp-i-ii} \\
        &\lesssim  \underbrace{\EE\left[ \Big\| - \frac{\lambda_t (2 - \sigma_t^2)}{\sigma_t^4} m(t, X_{T - t}) + \frac{2\lambda_t^2}{\sigma_t^4} X_{T - t} \Big\|_2^2 \right]}_{(i)} + \underbrace{\EE \left[ \Big\|  - \frac{\lambda_t^2(2 - \sigma_t^2)}{\sigma_t^6} C_t(X_{T - t}) X_{T - t} + \frac{\lambda_t^3}{\sigma_t^6} v_t(X_{T - t}) \Big\|_2^2 \right]}_{(ii)}. \nonumber
    \end{align}
    We shall bound terms $(i)$ and $(ii)$ in \cref{eq:decomp-i-ii} separately in the sequel.
    
    Let us start with term $(i)$,  for which we observe that 
    \begin{align}
    \label{eq:reformulate-(i)}
        - \frac{\lambda_t (2 - \sigma_t^2)}{\sigma_t^4} m(t, X_{T - t}) + \frac{2\lambda_t^2}{\sigma_t^4} X_{T - t} = - \frac{\lambda_t}{\sigma_t^2} m(t, X_{T - t}) + \frac{2\lambda_t^2}{\sigma_t^3} \EE\big[g \mid \lambda_t \theta + \sigma_t g = X_{T - t}\big]. 
    \end{align}
    By Jensen's inequality, we have
    \begin{align}
    \label{eq:B-d}
        \EE\left[\|m(t, X_{T - t})\|_2^2\right] \leq \EE\left[\|\theta\|_2^2\right] \leq d, \qquad \EE\big[\|\EE[g \mid \lambda_t \theta + \sigma_t g = X_{T - t}]\|_2^2\big] \leq \EE\left[\|g\|_2^2\right] = d, 
    \end{align}
    where $(\theta, g) \sim q_0 \otimes \cN(0, I_d)$.
    Substituting \cref{eq:reformulate-(i),eq:B-d} into term $(i)$, we arrive at
    \begin{align}
    \label{eq:(i)}
        (i) \lesssim \frac{\lambda_t^2 d}{\sigma_t^4}  + \frac{\lambda_t^4 d}{\sigma_t^6}. 
    \end{align}

    Next, we turn attention to term $(ii)$. 
    To this end,  write $X_{T - t} = \lambda_t \Theta + \sigma_t G$, where $(\Theta, G) \sim q_0 \otimes \cN(0, I_d)$. Note that 
    \begin{align}
        & - \frac{\lambda_t^2(2 - \sigma_t^2)}{\sigma_t^6} C_t(X_{T - t}) X_{T - t} + \frac{\lambda_t^3}{\sigma_t^6} v_t(X_{T - t}) \nonumber \\
        = & - \frac{\lambda_t^2 + \lambda_t^4}{\sigma_t^5} C_t(X_{T - t}) G + \frac{\lambda_t^3}{\sigma_t^6} \EE\left[ \big(\theta - m(t, X_{T - t})\big) \big(\theta - m(t, X_{T - t})\big)^{\top} \theta \mid \lambda_t \theta + \sigma_t g = X_{T - t} \right] \nonumber\\
        & - \frac{\lambda_t^3}{\sigma_t^6} C_t(X_{T - t})\big(\Theta - m(t, X_{T - t})\big) - \frac{\lambda_t^5}{\sigma_t^6} C_t(X_{T - t}) \Theta \nonumber \\
        \begin{split}\label{eq:four-terms}
        = & - \frac{\lambda_t^2 + \lambda_t^4}{\sigma_t^5} C_t(X_{T - t}) G + \frac{\lambda_t^3}{\sigma_t^6} \EE \left[ \big(\theta - m(t, X_{T - t})\big) \big\|\theta - m(t, X_{T - t})\big\|_2^2 \mid \lambda_t \theta + \sigma_t g = X_{T - t} \right] \\
        & - \frac{\lambda_t^3 + \lambda_t^5}{\sigma_t^6} C_t(X_{T - t})\big(\Theta - m(t, X_{T - t})\big) + \frac{\lambda_t^3}{\sigma_t^4} C_t(X_{T - t}) m(t, X_{T - t}). 
        \end{split}
    \end{align}
    We then separately upper bound the terms in the lase line of \cref{eq:four-terms}. 
    To this end, we find the following expressions useful (recall that $X_{T - t} = \lambda_t \Theta + \sigma_t G$ for $(\Theta, G) \sim q_0 \otimes \cN(0, 1)$): 
    \begin{align}
    \label{eq:G-expression}
    \begin{split}
        & C_t(X_{T - t}) = \frac{\sigma_t^2}{\lambda_t^2} \EE\left[ \big(g - \EE[g \mid \lambda_t \theta + \sigma_t g = X_{T - t}]\big) \big(g - \EE[g \mid \lambda_t \theta + \sigma_t g = X_{T - t}]\big)^{\top} \mid \lambda_t \theta + \sigma_t g = X_{T - t} \right], \\
        & \EE \left[ \big(\theta - m(t, X_{T - t})\big) \big\|\theta - m(t, X_{T - t})\big\|_2^2 \mid \lambda_t \theta + \sigma_t g = X_{T - t} \right] \\
        & \qquad = - \frac{\sigma_t^3}{\lambda_t^3} \EE\left[ \big(g - \EE[g \mid \lambda_t \theta + \sigma_t g = X_{T - t}]\big)\big\|g - \EE[g \mid \lambda_t \theta + \sigma_t g = X_{T - t}]\big\|_2^2 \mid \lambda_t \theta + \sigma_t g = X_{T - t}\right], \\
        & \Theta - m_t(X_{T - t}) = - \frac{\sigma_t}{\lambda_t}\big(G - \EE[g \mid \lambda_t \theta + \sigma_t g = X_{T - t}]\big). 
    \end{split}
    \end{align}
    \begin{itemize}
    \item 
    Let us look at the first summand in the last line of \cref{eq:four-terms}, namely, the term $- \sigma_t^{-5} (\lambda_t^2 + \lambda_t^4) C_t(X_{T - t}) G$. 
In view of \cref{eq:G-expression}, we can deduce that 
    \begin{align}
    \label{eq:ab-upper-bound}
        \EE\left[ \big\| C_t(X_{T - t}) G \big\|_2^2 \right] \overset{(a)}{\leq} \EE\left[ \big\| C_t(X_{T - t}) \big\|_\mathrm{F}^4 \right]^{1/2} \EE\left[ \big\| G \big\|_2^4 \right]^{1/2} \overset{(b)}{\lesssim} \frac{\sigma_t^4 d^3}{\lambda_t^4},
    \end{align}
    where step $(a)$ is by the Cauchy–Schwarz inequality, and step $(b)$ arises from the Jensen inequality. 
    Using \cref{eq:ab-upper-bound}, we see that 
    \begin{align}
    \label{eq:four-term-1}
        \frac{(\lambda_t^2 + \lambda_t^4)^2}{\sigma_t^{10}} \EE\left[ \big\| C_t(X_{T - t}) G \big\|_2^2\right] \lesssim \frac{d^3}{\sigma_t^6}. 
    \end{align}

    \item 
    Regarding the second summand in \cref{eq:four-terms}, note that  ${\lambda_t^3}{\sigma_t^{-6}} \EE [ (\theta - m(t, X_{T - t})) \|\theta - m(t, X_{T - t})\|_2^2 \mid \lambda_t \theta + \sigma_t g = X_{T - t} ]$. Applying \cref{eq:G-expression}, the Cauchy–Schwarz inequality and Jensen's inequality yields 
    \begin{align*}
        & \EE\left[ \big\| \EE \left[ (\theta - m(t, X_{T - t})) \|\theta - m(t, X_{T - t})\|_2^2 \mid \lambda_t \theta + \sigma_t g = X_{T - t} \right] \big\|_2^2 \right] \\
        & \leq  \EE\left[ \EE[\|\theta - m(t, X_{T - t})\|_2^2  \mid \lambda_t \theta + \sigma_t g = X_{T - t}] \cdot \EE[\|\theta - m(t, X_{T - t})\|_2^4 \mid \lambda_t \theta + \sigma_t g = X_{T - t}] \right] \\
        & \leq   \EE\left[ \|\Theta - m(t, X_{T - t})\|_2^4 \right]^{1/2} \cdot \EE\left[ \|\Theta - m(t, X_{T - t})\|_2^8 \right]^{1/2} \lesssim  \frac{\sigma_t^6 d^3}{\lambda_t^6}, 
    \end{align*}
    which further implies that 
    \begin{align}
    \label{eq:four-term-2}
        \frac{\lambda_t^6}{\sigma_t^{12}}\EE\left[ \| \EE [ (\theta - m(t, X_{T - t})) \|\theta - m(t, X_{T - t})\|_2^2 \mid \lambda_t \theta + \sigma_t g = X_{T - t} ]\|_2^2 \right] \lesssim \frac{d^3}{\sigma_t^6}. 
    \end{align}

    \item 
    The remaining two terms in \cref{eq:four-terms} can be controlled in a similar manner. The proof idea is similar to that for the first two terms, and we skip a detailed explanation for the compactness of presentation. 
    Specifically, we obtain the following upper bounds: 
    \begin{align}
         \frac{(\lambda_t^3 + \lambda_t^5)^2}{\sigma_t^{12}} \EE\left[ \| C_t(X_{T - t})(\Theta - m(t, X_{T - t})) \|_2^2 \right] &\lesssim \frac{d^3}{\sigma_t^6}, \label{eq:four-term-3} \\
         \frac{\lambda_t^6}{\sigma_t^8} \EE\left[ \| C_t(X_{T - t}) m(t, X_{T - t}) \|_2^2 \right] &\lesssim \frac{\lambda_t^2 d^3}{\sigma_t^4}. \label{eq:four-term-4}
    \end{align}
    \end{itemize}
    Finally, we put together \cref{eq:four-term-1,eq:four-term-2,eq:four-term-3,eq:four-term-4}. Invoking \cref{eq:four-terms}, we can then conclude that 
    \begin{align}
    \label{eq:(ii)}
        (ii) \lesssim \frac{d^3}{\sigma_t^6}.
    \end{align}

    To finish up,  combine \cref{eq:(i),eq:(ii)} to obtain   
    \begin{align*}
        \EE\left[ \big\| \partial_t \nabla_x \log q_t(X_{T - t}) \big\|_2^2\right] \lesssim \frac{d^3}{\sigma_t^6}, 
    \end{align*}
    thus concluding the proof. 
\end{proof}

\section{Bounding the KL divergence between diffusion processes}

\subsection{Proof of \cref{lemma:KL-barX-X}}
\label{sec:proof-lemma:KL-barX-X}


First, we show that for any $\tau \in (0, t_{k + 1} - t_k)$, there exists a unique strong solution to the SDE \eqref{eq:process-H} on the interval $[t_k + \tau, t_{k + 1}]$.
To this end, we shall introduce an augmented process, and show that $(H_s^{\tau})_{t_k \leq s \leq t}$ is a subset of this process. 
Let us begin by determining the drift function of this process and proving its Lipschitz continuity.  
Recall that $\overline \cF$ is defined as the drift functional of process \eqref{eq:dbarXt-long-equation}. 
For all $t \in (t_k, t_{k + 1}]$, observe that we can write 
$$\overline \cF\big(t, Y_{t_k}, (H_s^{\tau})_{t_k \leq s \leq t}\big) = \cG_t\left(Y_{t_k}, H_t^{\tau}, \int_0^{t - t_k} e^{t - t_k - r} \dd H_{t_k + r}^{\tau}, \int_0^{t - t_k} e^{t - t_k - r} H_{t_k + r}^{\tau} \dd r\right).$$
for some continuous mapping $\cG_t: \RR^{4d} \rightarrow \RR^d$. 
By the first point of \cref{lemma:score-properties}, we see that under Assumption \ref{assumption:moments}, for all $t \in [0, T)$ the mapping $x \mapsto s(t, x)$ is Lipschitz continuous. 
As a consequence, for any $\tau \in (0, t_{k + 1} - t_k)$ and all $t \in [t_k + \tau, t_{k + 1}]$, $\cG_t$ is $C_{\tau}$-Lipschitz continuous for some $C_{\tau} \in (0, \infty)$ that depends only on $\tau$.

We then introduce an augmented process $L_t = (L_{1, t}, L_{2, t}, L_{3, t}, L_{4, t}) \in \RR^{4d}$, defined as the solution to the following SDE: 
\begin{align}
\label{eq:SDE-L}
\begin{split}
    \dd L_t &=  \left[ \begin{array}{l}
         L_{1, t} + 2 s(t, L_{1, t})  \\
         \big( L_{1, t} + 2 s(t, L_{1, t}) - \cG_t(L_{1, t}, L_{2, t}, L_{3, t}, L_{4, t})  \big) / \sqrt{2} \\
         L_{3, t} + \big( L_{1, t} + 2 s(t, L_{1, t}) - \cG_t(L_{1, t}, L_{2, t}, L_{3, t}, L_{4, t})  \big) / \sqrt{2} \\
         L_{2, t} + L_{4, t} 
    \end{array} \right] \dd t + \left[ \begin{array}{l}
        \sqrt{2} \dd B_t \\
        \dd B_t \\
        \dd B_t \\
        0_d
    \end{array} \right] \\
    &=  b(t, L_t) \dd t + \sigma(t, L_t) \dd B_t. 
\end{split}
\end{align}
Here, $b(t, L) \in \RR^{4d}$ and $\sigma(t, L) \in \RR^{4d \times d}$.
Since $\cG_t$ is $C_{\tau}$-Lipschitz continuous for all $t \in [t_k + \tau, t_{k + 1}]$, 
we obtain that the mappings $L \mapsto b(t, L)$ and $L \mapsto \sigma(t, L)$ are Lipschitz continuous with a uniformly upper bounded Lipschitz constant for all $t \in [t_k + \tau, t_{k + 1}]$. 
By  \citet[Theorem 8.3]{le2016brownian}, SDE \eqref{eq:SDE-L} has a unique strong solution, regardless of the initialization. 
This establishes the existence and uniqueness of process $(H_t^{\tau})_{t_k \leq t \leq t_{k + 1}}$ as a solution to \eqref{eq:process-H}. 

In what follows,  denote by $Q_k(y)$ the law of $(Y_t)_{t_k \leq t \leq t_{k + 1}}$, and $\overline Q_k(y)$ the law of $(\cA_t(Y_{t_k}))_{t_k \leq t \leq t_{k + 1}}$, conditioned on $Y_{t_k} = y$,  where 
\begin{align*}
    \cA_t(y) := e^{t - t_k} y + (e^{t - t_k} - e^{-t + t_k}) s(t_k, y + g_{t_k, t}) + \sqrt{2} \int_0^{t - t_k} e^{t - t_k - r} \dd W_{t_k + r}.
\end{align*}
In the above display, we recall that $(W_t)_{t \geq 0}$ represents a $d$-dimensional standard Brownian motion, and $g_{t_k, t}$ is defined in \cref{eq:double-index-g}. 
Using the decomposition of KL divergence, we obtain 
\begin{align*}
     \KL\big( Q_{T - \delta} \parallel \overline Q_{T - \delta} \big) =  \sum_{k = 0}^{K - 1}\EE_{Q_{T - \delta}} \Big[ \KL \big( Q_k(Y_{t_k}) \parallel \overline Q_k(Y_{t_k}) \big) \Big],
\end{align*}
where the expectation is taken with respect to $(Y_{t})_{0 \leq t \leq T - \delta} \sim Q_{T - \delta}$. 

For $\tau \in (0, t_{k + 1} - t_k)$, we define the process $(U_t^{\tau})_{t_k \leq t \leq t_{k + 1}}$ with $U_{t_k}^{\tau} = Y_{t_k}$, and 
\begin{align*} 
    & \dd U_t^{\tau} = \big(U_t^{\tau} + 2s(t, U_t^{\tau})\big) \dd t + \sqrt{2} \dd W_t, \qquad \qquad \qquad \qquad \,\,\,  \mbox{for }t_k \leq t \leq t_{k} + \tau, \\
    & \dd U_t^{\tau} = \overline\cF\big(t, Y_{t_k}, (W_s - W_{t_k})_{t_k \leq s \leq t}\big) \dd t + \sqrt{2}\, \dd W_t, \qquad \,\,\,\,\, ~~\mbox{for }t_k + \tau \leq t \leq t_{k + 1}. 
\end{align*}
As $\tau \to 0^+$, it holds that $\sup_{t_k \leq t \leq t_{k + 1}} \|U_t^{\tau} - \cA_t(Y_{t_k})\|_2 \overset{\mathrm{a.s.}}{\to} 0$. 
Denote  by $\overline Q_k^{\tau}(Y_{t_k})$ the conditional distribution of $(U_t^{\tau})_{t_k \leq t \leq t_{k + 1}}$ given $U_{t_k}^{\tau} = Y_{t_k}$. 
Therefore, for all $Y_{t_k} \in \RR^d$ the distribution $\overline Q_k^{\tau} (Y_{t_k})$ converges weakly 
 to $\overline Q_k (Y_{t_k})$. 
Invoking the same Girsanov-type arguments as in   \citet[Section 5.2]{chen2023sampling}, 
we see that 
\begin{align}
\label{eq:KL-tau}
    \KL \big( Q_k(Y_{t_k}) \parallel \overline Q_k^{\tau}(Y_{t_k}) \big) = \int_{t_k + \tau}^{t_{k + 1}} \EE\Big[ \big\| \cF(t, Y_t) - \overline \cF(t, Y_{t_k}, (H_s^{\tau})_{t_k \leq s \leq t}) \big\|_2^2 \, \big| \, Y_{t_k}  \Big] \dd t. 
\end{align}
Leveraging the lower semicontinuity of KL divergence, we obtain
\begin{align*}
    \EE_{Q_{T - \delta}} \Big[ \KL \big( Q_k(Y_{t_k}) \parallel \overline Q_k(Y_{t_k}) \big) \Big] &\leq  \EE_{Q_{T - \delta}} \Big[ \liminf_{\tau \to 0^+} \KL \big( Q_k(Y_{t_k}) \parallel \overline Q_k^{\tau}(Y_{t_k}) \big) \Big] \\
    &\leq  \liminf_{\tau \to 0^+} \EE_{Q_{T - \delta}} \Big[ \KL \big( Q_k(Y_{t_k}) \parallel \overline Q_k^{\tau}(Y_{t_k}) \big) \Big] \\
    &=  \liminf_{\tau \to 0^+} \int_{t_k + \tau}^{t_{k + 1}} \EE\Big[ \big\| \cF(t, Y_t) - \overline \cF(t, Y_{t_k}, (H_s^{\tau})_{t_k \leq s \leq t}) \big\|_2^2 \, \big| \, Y_{t_k}  \Big] \dd t,
\end{align*}
where the second inequality above follows from Fatou's Lemma, and the last equality arises from  \cref{eq:KL-tau}. 
This completes the proof.

\subsection{Proof of \cref{lemma:upper-bound-H}}
\label{sec:proof-lemma:upper-bound-H}

    By virtue of \cref{eq:process-H}, for all $t \in [t_k + \tau, t_{k + 1}]$ we have
    \begin{align}
    \label{eq:H-B-diff}
    H_t^{\tau} - H_{t_k}^{\tau} = B_t - B_{t_k} + v_{a, t} + v_{b, t} + v_{c, t} + v_{d, t} + v_{e, t} + v_{f, t}, 
    \end{align}
    where the residual terms $v_{a, t}, v_{b, t}, v_{c, t}, v_{d, t}, v_{e, t}, v_{f, t} \in \RR^d$ are defined respectively as follows: 
    \begin{align*}
        & v_{a, t} = \frac{1}{\sqrt{2}} \int_{t_k + \tau}^t (e^{\zeta - t_k} - e^{-\zeta + t_k}) s(t_k, Y_{t_k}) \dd \zeta, \\
        & v_{b, t} =  \int_{t_k + \tau}^t \int_0^{\zeta - t_k} e^{\zeta - t_k - r} (\dd B_{t_k + r} - \dd H_{t_k + r}^{\tau}) \dd \zeta, \\
        & v_{c, t} = - \frac{1}{\sqrt{2}} \int_{t_k + \tau}^{t}  (e^{\zeta - t_k} + e^{-\zeta + t_k}) \big(s(t_k, Y_{t_k} + h_{t_k, \zeta}) - s(t_k, Y_{t_k}) - \nabla_x s(t_k, Y_{t_k} + h_{t_k, \zeta}) h_{t_k, \zeta} \big) \dd \zeta, \\
        & v_{d, t} = \sqrt{2} \int_{t_k + \tau}^t \int_0^{\zeta - t_k} e^{\zeta - t_k - r} \big(s(t_k + r, Y_{t_k + r}) - s(t_k, Y_{t_k}) \big) \dd r \dd \zeta, \\
        & v_{e, t} = - 2 \int_{t_k + \tau}^{t} \nabla_x s(t_k, Y    _{t_k} + h_{t_k, \zeta}) \int_0^{\zeta - t_k} e^{\zeta - t_k - r} H_{t_k + r}^{\tau} \dd r \dd \zeta, \\
        & v_{f, t} = \sqrt{2} \int_{t_k + \tau}^t \big(s(\zeta, Y_{\zeta}) - s(t_k, Y_{t_k}) - \sqrt{2} \nabla_x s(t_k, Y_{t_k} + h_{t_k, \zeta})H_{\zeta}^{\tau} \big) \dd \zeta,
    \end{align*}
    %
    %
    In the above display, we let
    \begin{align*}
    h_{t_k, \zeta} = \frac{2\sqrt{2}}{e^{\zeta - t_k} - e^{-\zeta + t_k}}\int_0^{\zeta - t_k} e^{\zeta - t_k - r}H_{t_k + r}^{\tau} \dd r,
\end{align*}
which essentially replaces $W_{t_k + r} - W_{t_k}$ with $H_{t_k + r}^{\tau}$ in the definition of $g_{t_k, \zeta}$. 

Before proceeding to bounding the terms in Eq.~\eqref{eq:H-B-diff}, we find it helpful to first make some observations. 
For $t \in [t_k, t_{k + 1}]$,  define $\gamma_{t_k, t} = \|H_t^{\tau} - B_t + B_{t_k}\|_2$, which clearly obeys 
$\gamma_{t_k, t_k} = 0$. 
In view of the first point of \cref{lemma:score-properties} and Assumption \ref{assumption:moments}, we see that 
$$
\big\|\nabla_x s(t, x)\big\|_2 \leq 4 \big(\sigma_{T - t}^{-2} + \lambda_{T - t}^2 \sigma_{T - t}^{-4}\big) {d}$$ for all $x \in \RR^d$ and $t \in [0, T)$. 
Similarly, by the second point of \cref{lemma:score-properties}, we can deduce that
$$ \big\|\nabla_x^2 s(t, x)\big\|_2 \leq 8 \lambda_{T - t}^3 \sigma_{T - t}^{-6} d^{3/2}$$ for all $x \in \RR^d$ and $t \in [0, T)$.

Next, we adopt these upper bounds to analyze the terms on the right-hand side of \cref{eq:H-B-diff}. 
We first look at $v_{c, t}$. 
By the fundamental theorem of calculus, we have 
\begin{align*}
    s(t_k, Y_{t_k} + h_{t_k, \zeta}) - s(t_k, Y_{t_k}) = \int_0^1 \nabla_x s(t_k, Y_{t_k} + \eta h_{t_k, \zeta}) h_{t_k, \zeta} \dd \eta. 
\end{align*}
Using this decomposition and the triangle inequality, we obtain that 
\begin{align*}
    & \big\|s(t_k, Y_{t_k} + h_{t_k, \zeta}) - s(t_k, Y_{t_k}) - \nabla_x s(t_k, Y_{t_k} + h_{t_k, \zeta}) h_{t_k, \zeta} \big\|_2 \\
    & \leq \int_0^1 \|\nabla_x s(t_k, Y_{t_k} + \eta h_{t_k, \zeta})  \|_2 \|h_{t_k, \zeta}\|_2 \dd \eta\\
    & \leq 8 (\sigma_{T - t_k}^{-2} + \lambda_{T - t_k}^2 \sigma_{T - t_k}^{-4}) {d} \|h_{t_k, \zeta}\|_2. 
\end{align*}
Recall that under Assumption \ref{assumption:step-size}, we have $\Delta_k \leq \kappa \leq 1 / 4$. 
Therefore, for all $\zeta \in [t_k, t_{k + 1}]$, it holds that $e^{\zeta - t_k} + e^{-\zeta + t_k} < 2.1$, and  as a consequence, 
\begin{align}
\label{eq:bound-vc}
\begin{split}
    \|v_{c, t}\|_2 &\leq  12 \int_{t_k + \tau}^t  (\sigma_{T - t_k}^{-2} + \lambda_{T - t_k}^2 \sigma_{T - t_k}^{-4}) {d} \|h_{t_k, \zeta}\|_2 \dd \zeta  \\
    &\leq  \int_{t_k + \tau}^t \frac{34 (\sigma_{T - t_k}^{-2} + \lambda_{T - t_k}^2 \sigma_{T - t_k}^{-4}) {d}}{e^{\zeta - t_k} - e^{-\zeta + t_k}} \Big[ \int_0^{\zeta - t_k} e^{\zeta - t_k - r} \big( \| B_{t_k + r} - B_{t_k} \|_2 + \gamma_{t_k, t_k + r} \big) \dd r \Big] \dd \zeta. 
\end{split}
\end{align}
We then proceed to upper bound the norms of $v_{a, t}$ and $v_{d, t}$. More specifically, 
\begin{align}
    & \|v_{a, t}\|_2 \leq \frac{1}{\sqrt{2}} (e^{t - t_k} + e^{-t + t_k} - 2) \, \big\| s(t_k, Y_{t_k})\big\|_2 \leq \frac{(t - t_k)^2}{4} \, \big\|s(t_k, Y_{t_k})\big\|_2, \label{eq:bound-va} \\
    & \| v_{d, t} \|_2 \leq 2 (t - t_k) \int_0^{t - t_k} \big\| s(t_k + r, Y_{t_k + r}) - s(t_k, Y_{t_k}) \big\|_2 \dd r. \label{eq:bound-vd}
\end{align}
With regards to $v_{e, t}$, 
it follows from the triangle inequality and the assumption  $\Delta_k \leq \kappa \leq 1 / 4$ that  
\begin{align}
\label{eq:bound-ve}
     \| v_{e, t} \|_2 \leq  10 (t - t_k) d \big(\sigma_{T - t_k}^{-2} + \lambda_{T - t_k}^{2} \sigma_{T - t_k}^{-4} \big)   \int_0^{t - t_k} (\gamma_{t_k, t_k + r} + \|B_{t_k + r} - B_{t_k}\|_2) \dd r. 
\end{align}

To bound the norm of $v_{f, t}$, we find it helpful to bound the norm of the vector $s(\zeta, Y_{\zeta}) - s(t_k, Y_{\zeta})$, towards which we resort to the third point of \cref{lemma:score-properties}. More precisely, for all $r \in [t_k, \zeta]$,  the third point of \cref{lemma:score-properties} tells us that
        \begin{align*}
            &\partial_r s(r, Y_{\zeta}) 
            = \frac{\lambda_{T - r} + \lambda_{T - r}^3}{\sigma_{T - r}^4} m(r, Y_{\zeta}) - \frac{2\lambda_{T - r}^2}{\sigma_{T - r}^4} Y_{\zeta}  - \frac{\lambda_{T - r}}{\sigma_{T - r}^2} \EE \big[\big(\theta - m(r, Y_{\zeta}) \big)\cF(\theta, Y_{\zeta}, r) \mid \lambda_{T - r} \theta + \sigma_{T - r} g = Y_{\zeta} \big], 
        \end{align*}
where $(\theta, g) \sim q_0 \otimes \cN(0, I_d)$, and $m(r, y) = \EE[\theta \mid \lambda_{T - r} \theta + \sigma_{T - r} g = y]$. 
Under Assumption \ref{assumption:moments}, it holds that $\|m(r, Y_{\zeta})\|_2 \leq \sqrt{d}$ and $$|\cF(\theta, Y_{\zeta}, r)| \leq \lambda_{T - r}^2 \sigma_{T - r}^{-4} d + (\lambda_{T - r} + \lambda_{T - r}^3) \sigma_{T - r}^{-4} \sqrt{d} \|Y_{\zeta}\|_2.$$ Since $s(\zeta, Y_{\zeta}) - s(t_k, Y_{\zeta}) = \int_{t_k}^{\zeta} \partial_r s(r, Y_{\zeta}) \dd r$, 
we see that for all $\zeta \in [t_k, t_{k + 1}]$, 
\begin{align}
\label{eq:szeta-stk-diff}
\begin{split}
    & \big\|s(\zeta, Y_{\zeta}) - s(t_k, Y_{\zeta}) \big\|_2 
    \leq \frac{\sqrt{d} (\zeta - t_k)(\lambda_{T - \zeta} + \lambda_{T - \zeta}^3)}{\sigma_{T - \zeta}^4} + \frac{2(\zeta - t_k)\lambda_{T - \zeta}^2\|Y_{\zeta}\|_2}{\sigma_{T - \zeta}^4} \\
    &\qquad  + \Big( \frac{d \lambda_{T - \zeta}^2 }{\sigma_{T - \zeta}^4} + \frac{\sqrt{d}(\lambda_{T - \zeta} + \lambda_{T - \zeta}^3) \|Y_{\zeta}\|_2}{\sigma_{T - \zeta}^4} \Big) \cdot \int_{t_k}^{\zeta} \frac{\lambda_{T - r} \EE[\|\theta - m(r, Y_{\zeta})\|_2 \mid \lambda_{T - r} \theta + \sigma_{T - r} g = Y_{\zeta}]}{\sigma_{T - r}^2} \dd r. 
\end{split}
\end{align}
Further, we make the observation that  
\begin{align}
\label{eq:s-zeta-s-tk-nabla-diff}
    & \|s(\zeta, Y_{\zeta}) - s(t_k, Y_{t_k}) - \sqrt{2} \nabla_x s(t_k, Y_{t_k} + h_{t_k, \zeta})H_{\zeta}^{\tau}\|_2 \nonumber \\
    &\overset{(i)}{\leq}  \|s(\zeta, Y_{\zeta}) - s(t_k, Y_{\zeta})\|_2 + \|s(t_k, Y_{\zeta}) - s(t_k, Y_{t_k}) -  \nabla_x s(t_k, Y_{t_k} + h_{t_k, \zeta})(Y_{\zeta} - Y_{t_k})\|_2 \nonumber \\
    & \qquad + \| \nabla_x s(t_k, Y_{t_k} + h_{t_k, \zeta})( Y_{\zeta} - Y_{t_k} - \sqrt{2} B_{\zeta} + \sqrt{2} B_{t_k})\|_2  + \|\sqrt{2} \nabla_x s(t_k, Y_{t_k} + h_{t_k, \zeta})(B_{\zeta} - B_{t_k} - H_{\zeta}^{\tau})\|_2 \nonumber \\
  & \overset{(ii)}{\leq} \|s(\zeta, Y_{\zeta}) - s(t_k, Y_{\zeta})\|_2 + 8 d(\sigma_{T - t_k}^{-2} + \lambda_{T - t_k}^{2} \sigma_{T - t_k}^{-4}) \|Y_{\zeta} - Y_{t_k}\|_2 \nonumber  \\
    & \qquad + 4 d(\sigma_{T - t_k}^{-2} + \lambda_{T - t_k}^{2} \sigma_{T - t_k}^{-4}) \|Y_{\zeta} - Y_{t_k} - \sqrt{2} B_{\zeta} + \sqrt{2} B_{t_k}\|_2  + 6d (\sigma_{T - t_k}^{-2} + \lambda_{T - t_k}^{2} \sigma_{T - t_k}^{-4}) \gamma_{t_k, \zeta}. 
\end{align}
In the above display, $(i)$ comes from the triangle inequality, whereas $(ii)$ is by the first point of \cref{lemma:score-properties} and Assumption \ref{assumption:moments}. 
Substituting the upper bounds in \cref{eq:szeta-stk-diff,eq:s-zeta-s-tk-nabla-diff} into the definition of $v_{f, t}$ yields 
\begin{align}
    & \|v_{f, t}\|_2 \nonumber  \\ 
    & \leq  \sqrt{2}  \int_{t_k + \tau}^t \Big( \frac{\sqrt{d} (\zeta - t_k)(\lambda_{T - \zeta} + \lambda_{T - \zeta}^3)}{\sigma_{T - \zeta}^4} + \frac{2(\zeta - t_k)\lambda_{T - \zeta}^2\|Y_{\zeta}\|_2}{\sigma_{T - \zeta}^4} \Big) \dd \zeta \nonumber \\
    & \leq \sqrt{2}  \int_{t_k + \tau}^t  \Big( \frac{d \lambda_{T - \zeta}^2 }{\sigma_{T - \zeta}^4} + \frac{\sqrt{d}(\lambda_{T - \zeta} + \lambda_{T - \zeta}^3) \|Y_{\zeta}\|_2}{\sigma_{T - \zeta}^4} \Big) \cdot \int_{t_k}^{\zeta} \frac{\lambda_{T - r} \EE[\|\theta - m(r, Y_{\zeta})\|_2 \mid \lambda_{T - r} \theta + \sigma_{T - r} g = Y_{\zeta}]}{\sigma_{T - r}^2} \dd r \dd \zeta \nonumber \\
    & + d \big(\sigma_{T - t_k}^{-2} + \lambda_{T - t_k}^{2} \sigma_{T - t_k}^{-4} \big) \int_{t_k + \tau}^t  (9 \gamma_{t_k, \zeta} + 6\|Y_{\zeta} - Y_{t_k} - \sqrt{2} B_{\zeta} + \sqrt{2} B_{t_k}\|_2 + 12 \|Y_{\zeta} - Y_{t_k}\|_2) \dd \zeta. 
    \label{eq:bound-vf}
\end{align}

The next step is to upper bound the norm of $v_{b, t}$, towards which we first make note of the following expression
\begin{align*}
    v_{b, t}& =  \int_{t_k + \tau}^t \int_0^{\zeta - t_k} \int_0^{e^{\zeta - t_k - r}} \dd x (\dd B_{t_k + r} - \dd H_{t_k + r}^{\tau}) \dd \zeta \\
    &=  \int_{t_k + \tau}^t \int_0^{\zeta - t_k} \int_0^{e^{\zeta - t_k}} \mathbbm{1}\{r \leq \zeta - t_k - \log x\} \dd x (\dd B_{t_k + r} - \dd H_{t_k + r}^{\tau}) \dd \zeta \\
    &=  \int_{t_k + \tau}^t \int_0^{e^{\zeta - t_k}} (B_{\zeta - \log x} - B_{t_k} - H_{\zeta - \log x}^{\tau}) \dd x \dd \zeta. 
\end{align*}
As a consequence,   we have 
\begin{align}
\label{eq:bound-vb}
    \|v_{b, t}\|_2 \leq  \int_{t_k + \tau}^t \int_0^{e^{\zeta - t_k}} \gamma_{t_k, \zeta - \log x} \dd x \dd \zeta. 
\end{align}

Finally, we can conclude the proof of the lemma using the upper bounds derived above. 
For $t \in [t_k, t_{k + 1}]$, we define $$\gamma_{\ast}(t) = \sup_{s \in [t_k, t]} \gamma_{t_k, s}.$$ 
We see that at $t = t_k$ we have $\gamma_{\ast}(t_k) = 0$, and our goal is to upper bound $\gamma_{\ast}(t_{k + 1})$. 
In addition, since the processes $(B_t - B_{t_k})_{t_k \leq t \leq t_{k + 1}}$ and $(H_t^{\tau})_{t_k \leq t \leq t_{k + 1}}$ have continuous sample paths, the mapping $t \mapsto \gamma_{\ast}(t)$ is continuous on $ t \in [t_k, t_{k + 1}]$. 
Substitution of \cref{eq:bound-vc,eq:bound-va,eq:bound-vd,eq:bound-ve,eq:bound-vf,eq:bound-vb} into \cref{eq:H-B-diff} gives 
\begin{align}
\label{eq:gamma-C0-C1}
    \gamma_{\ast}(t) \leq C_0 + C_1 \gamma_{\ast}(t),
\end{align}
where $C_0, C_1 \in \RR_{> 0}$ are defined as follows:
\begin{align*}
    & C_0 = \frac{(t - t_k)^2}{4} \|s(t_k, Y_{t_k})\|_2 + \int_{t_k + \tau}^t \frac{44 (\sigma_{T - t_k}^{-2} + \lambda_{T - t_k}^2 \sigma_{T - t_k}^{-4}) {d}}{e^{\zeta - t_k} - e^{-\zeta + t_k}}  \int_0^{\zeta - t_k} \| B_{t_k + r} - B_{t_k} \|_2 \dd r \dd \zeta  \\
    & + 2 (t - t_k) \int_0^{t - t_k} \hspace{-0.4cm}\big\| s(t_k + r, Y_{t_k + r}) - s(t_k, Y_{t_k}) \big\|_2 \dd r + 10 (t - t_k) d(\sigma_{T - t_k}^{-2} + \lambda_{T - t_k}^{2} \sigma_{T - t_k}^{-4})   \int_0^{t - t_k} \hspace{-0.4cm}  \|B_{t_k + r} - B_{t_k}\|_2 \dd r\\
    & + 2\sqrt{2}  \int_{t_k + \tau}^t  \Big( \frac{d \lambda_{T - \zeta}^2 }{\sigma_{T - \zeta}^4} + \frac{\sqrt{d}(\lambda_{T - \zeta} + \lambda_{T - \zeta}^3) \|Y_{\zeta}\|_2}{\sigma_{T - \zeta}^4} \Big) \cdot \int_{t_k}^{\zeta} \frac{\lambda_{T - r} \sqrt{d}}{\sigma_{T - r}^2} \dd r \dd \zeta  \nonumber \\
    & + \sqrt{2}  \int_{t_k + \tau}^t \Big( \frac{\sqrt{d} (\zeta - t_k)(\lambda_{T - \zeta} + \lambda_{T - \zeta}^3)}{\sigma_{T - \zeta}^4} + \frac{2(\zeta - t_k)\lambda_{T - \zeta}^2\|Y_{\zeta}\|_2}{\sigma_{T - \zeta}^4} \Big) \dd \zeta \\
    & + d (\sigma_{T - t_k}^{-2} + \lambda_{T - t_k}^{2} \sigma_{T - t_k}^{-4}) \int_{t_k + \tau}^t  (6 \|Y_{\zeta} - Y_{t_k} - \sqrt{2} B_{\zeta} + \sqrt{2} B_{t_k}\|_2 + 12 \|Y_{\zeta} - Y_{t_k}\|_2) \dd \zeta, \\
    & C_1 = 1.3(t - t_k) +   {(53 (t - t_k) + 10 (t - t_k)^2) (\sigma_{T - t_k}^{-2} + \lambda_{T - t_k}^2 \sigma_{T - t_k}^{-4}) {d}}.   
\end{align*}
Under Assumption \ref{assumption:step-size}, we know that $C_1 \leq 1 / 2$ for all $t \in [t_k, t_{k + 1}]$, and hence $\gamma_{\ast}(t) \leq 2 C_0$. 
Invoking  \cref{lemma:W-B-diff} tells us that: for all $\zeta \in [t_k, t_{k + 1}]$ we have 
\begin{align*}
    & \EE\big[ \|Y_{\zeta} - Y_{t_k}\|_2^2 \big] \lesssim d(\zeta - t_k) + d \sigma_{T - \zeta}^{-2}(\zeta - t_k)^2, \qquad \EE\big[ \|Y_{\zeta}\|_2^2 \big] \lesssim d, \qquad \EE\big[ \|B_{\zeta} - B_{t_k}\|_2^2 \big] \lesssim d(\zeta - t_k),  \\
    & \EE\big[ \|s(\zeta, Y_{\zeta})\|_2^2 \big] \lesssim d \sigma_{T - \zeta}^{-2}, \qquad \EE\big[ \|Y_{\zeta} - Y_{t_k} - \sqrt{2} B_{\zeta} + \sqrt{2} B_{t_k}\|_2^2 \big] \lesssim d \sigma_{T - \zeta}^{-2} (\zeta - t_k)^2, \\
    & \EE\big[ \|Y_{\zeta} - Y_{t_k}\|_2^4 \big] \lesssim d^2(\zeta - t_k)^2 + d^2 \sigma_{T - \zeta}^{-4}(\zeta - t_k)^4, \qquad \EE\big[ \|Y_{\zeta}\|_2^4 \big] \lesssim d^2, \qquad \EE\big[ \|B_{\zeta} - B_{t_k}\|_2^4 \big] \lesssim d^2(\zeta - t_k)^2,  \\
    & \EE\big[ \|s(\zeta, Y_{\zeta})\|_2^2 \big] \lesssim d^2 \sigma_{T - \zeta}^{-4}, \qquad \EE\big[ \|Y_{\zeta} - Y_{t_k} - \sqrt{2} B_{\zeta} + \sqrt{2} B_{t_k}\|_2^2 \big] \lesssim d^2 \sigma_{T - \zeta}^{-4} (\zeta - t_k)^4. 
\end{align*}
Taking the expectation of $C_0^2$ and $C_0^4$ implies that for all $t \in [t_k, t_{k + 1}]$, 
\begin{align*}
     \EE[\gamma_{\ast}(t)^2] 
    &\lesssim  \sigma_{T - t}^{-2} (t - t_k)^4 d + (\sigma_{T - t_k}^{-2} + \lambda_{T - t_k}^2 \sigma_{T - t_k}^{-4})^2 (t - t_k)^3 d^3 + \lambda_{T - t}^4 \sigma_{T - t}^{-12} (t - t_k)^4 d^3 + \lambda_{T - t}^2 \sigma_{T - t}^{-8} (t - t_k)^4 d, \\
     \EE[\gamma_{\ast}(t)^4] 
    &\lesssim \sigma_{T - t}^{-4} (t - t_k)^8 d^2 + (\sigma_{T - t_k}^{-2} + \lambda_{T - t_k}^2 \sigma_{T - t_k}^{-4})^4 (t - t_k)^6 d^6 + \lambda_{T - t}^8 \sigma_{T - t}^{-24} (t - t_k)^8 d^6 + \lambda_{T - t}^4 \sigma_{T - t}^{-16} (t - t_k)^8 d^2. 
\end{align*}
The proof is thus complete. 

\subsection{Proof of \cref{lemma:Q-barQ}}
\label{sec:proof-lemma:Q-barQ}

    According to \cref{lemma:KL-barX-X}, we see that: in order to upper bound $\KL(Q_{T - \delta} \parallel \overline Q_{T - \delta})$, it suffices to control 
    \begin{align*}
        \sum_{k = 0}^{K - 1} \int_{t_k}^{t_{k + 1}} \EE\Big[ \big\| \cF(t, Y_t) - \overline \cF(t, Y_{t_k}, (H_s^{\tau})_{t_k \leq s \leq t}) \big\|_2^2  \Big] \dd t. 
    \end{align*}
    In view of \cref{lemma:score-properties} and Assumption \ref{assumption:moments}, we know that for all $t \in [0, T)$ and $x \in \RR^d$, it holds that 
    \begin{align*}
         \|\nabla_x s(t, x)\|_2 \lesssim (\sigma_{T - t}^{-2} + \lambda_{T - t}^2 \sigma_{T - t}^{-4}) d, \qquad \|\nabla_x^2 s(t, x)\|_2 \lesssim \lambda_{T - t}^3 \sigma_{T - t}^{-6} d^{3/2},
    \end{align*}
    which in turn allow one to show that $\bar \cF$ is Lipschitz continuous with respect to the last input. More precisely,  
    \begin{align*}
        & \big\|\overline \cF(t, Y_{t_k}, (H_s^{\tau})_{t_k \leq s \leq t}) - \overline \cF(t, Y_{t_k}, (B_s - B_{t_k})_{t_k \leq s \leq t}) \big\|_2 \\
        & \lesssim \|\nabla_x s\big(t_k, \overline Y_{t_k} + h_{t_k, t}\big)\|_2 \|B_t - B_{t_k} - H_t^{\tau}\|_2 + \|\nabla_x s\big(t_k, \overline Y_{t_k} + h_{t_k, t}\big) - \nabla_x s\big(t_k, \overline Y_{t_k} + b_{t_k, t}\big)\|_2 \cdot \|B_t - B_{t_k}\|_2 \\
        & + \|\nabla_x s\big(t_k, \overline Y_{t_k} + h_{t_k, t}\big)\|_2 \|h_{t_k, t} - b_{t_k, t}\|_2 + \|\nabla_x s\big(t_k, \overline Y_{t_k} + h_{t_k, t}\big) - \nabla_x s\big(t_k, \overline Y_{t_k} + b_{t_k, t}\big)\|_2 \cdot \|b_{t_k, t}\|_2 \\
        & \lesssim \Big(d (\sigma_{T - t}^{-2} + \lambda_{T - t}^2 \sigma_{T - t}^{-4}) + d^{3/2} \lambda_{T - t}^3 \sigma_{T - t}^{-6} (\|B_t - B_{t_k}\|_2 +  \| b_{t_k, t} \|_2)\Big)  \sup_{t_k \leq t \leq t_{k + 1}} \big\|B_t - B_{t_k} - H_t^{\tau} \big\|_2, 
    \end{align*}
    with 
    \begin{align*}
        & b_{t_k, t} = 2\sqrt{2} (e^{t - t_k} - e^{-t + t_k})^{-1} \int_0^{t - t_k} e^{t - t_k - r}(B_{t_k + r} - B_{t_k}) \dd r, \\
        & h_{t_k, t} = 2\sqrt{2} (e^{t - t_k} - e^{-t + t_k})^{-1} \int_0^{t - t_k} e^{t - t_k - r}H_{t_k + r}^{\tau} \dd r.
    \end{align*} 
    %
    Putting the above inequality together with \cref{lemma:upper-bound-H}  gives
    \begin{align*}
        & \sum_{k = 0}^{K - 1} \int_{t_k}^{t_{k + 1}} \EE\Big[ \big\| \cF(t, Y_t) - \overline \cF(t, Y_{t_k}, (H_s^{\tau})_{t_k \leq s \leq t}) \big\|_2^2  \Big]  \dd t \\
        & \lesssim  \sum_{k = 0}^{K - 1} \int_{t_k}^{t_{k + 1}}\EE\Big[ \big\| \cF(t, Y_t) - \overline \cF(t, Y_{t_k}, (B_s - B_{t_k})_{t_k \leq s \leq t}) \big\|_2^2  \Big] \dd t \\
        & \quad + \sum_{k = 0}^{K - 1} \int_{t_k}^{t_{k + 1}} \EE\Big[ \big\| \overline \cF(t, Y_{t_k}, (H_s^{\tau})_{t_k \leq s \leq t}) - \overline \cF(t, Y_{t_k}, (B_s - B_{t_k})_{t_k \leq s \leq t}) \big\|_2^2  \Big] \dd t \\
        & \lesssim  \sum_{k = 0}^{K - 1} \int_{t_k}^{t_{k + 1}}\EE\Big[ \big\| \cF(t, Y_t) - \overline \cF(t, Y_{t_k}, (B_s - B_{t_k})_{t_k \leq s \leq t}) \big\|_2^2  \Big] \dd t \\
        & \quad + \sum_{k = 0}^{K - 1}\Delta_k \Big(d^2 (\sigma_{T - t}^{-2}  + \lambda_{T - t}^2 \sigma_{T - t}^{-4})^2 + d^{3} \lambda_{T - t}^6 \sigma_{T - t}^{-12} \Delta_k \Big)  \\
        & \quad \times \Big(\sigma_{T - t_{k + 1}}^{-2} \Delta_k^4 d + (\sigma_{T - t_k}^{-2} + \lambda_{T - t_k}^2 \sigma_{T - t_k}^{-4})^2 \Delta_k^3 d^3 + \lambda_{T - t_{k + 1}}^4 \sigma_{T - t_{k + 1}}^{-12} \Delta_k^4 d^3 + \lambda_{T - t_{k + 1}}^2 \sigma_{T - t_{k + 1}}^{-8} \Delta_k^4 d \Big). 
    \end{align*}
    It is seen from the triangle inequality that 
    \begin{align*}
        \sum_{k = 0}^{K - 1} \int_{t_k}^{t_{k + 1}} \EE\Big[ \big\| \cF(t, Y_t) - \overline \cF(t, Y_{t_k}, (B_s - B_{t_k})_{t_k \leq s \leq t}) \big\|_2^2  \Big] \dd t \lesssim \TT_1 + \TT_2 + \TT_3 + \TT_4, 
    \end{align*}
    where
    \begin{align*}
        & \TT_1 = \sum_{k = 0}^{K - 1} \int_{t_k}^{t_{k + 1}} (e^{t - t_k} - e^{-t + t_k})^2 \EE\big[ \|  s(t_k, Y_{t_k})  \|_2^2 \big] \dd t, \\
        & \TT_2 = \sum_{k = 0}^{K - 1} \int_{t_k}^{t_{k + 1}} \EE\big[ \big\| s(t, Y_t) - s(t_k, Y_{t_k}) - \sqrt{2} \nabla_x s(t_k, Y_{t_k} + b_{t_k, t})(B_t - B_{t_k}) \big\|_2^2 \big] \dd t, \\
        & \TT_3 = \sum_{k = 0}^{K - 1} \int_{t_k}^{t_{k + 1}} \int_0^{t - t_k} \EE \big[ \| s(t_k + r, Y_{t_k + r}) - s(t_k, Y_{t_k}) - \sqrt{2} \nabla_x s(t_k, Y_{t_k} + b_{t_k, t})(B_{t_k + r} - B_{t_k})  \|_2^2 \big] \dd r \dd t, \\
        & \TT_4 = \sum_{k = 0}^{K - 1} \int_{t_k}^{t_{k + 1}} \EE\big[ \| s(t_k, Y_{t_k} + b_{t_k, t}) - s(t_k, Y_{t_k}) - \nabla_x s(t_k, Y_{t_k} + b_{t_k, t}) b_{t_k, t} \|_2^2 \big] \dd t. 
    \end{align*}
The  terms $\TT_1$, $\TT_2$, $\TT_3$ and $\TT_4$ are upper bounded separately in \cref{lemma:T1234} below, whose proof can be found in Appendix \ref{sec:proof-lemma:T1234}.  

\begin{lemma}
\label{lemma:T1234}
    Under the assumptions of \cref{lemma:Q-barQ}, it holds that
    \begin{enumerate}
        \item $\TT_1 \lesssim \sum_{k = 0}^{K - 1} \sigma_{T - t_k}^{-2} \Delta_k^3 d$; 
        \item $\TT_2 \lesssim \sum_{k = 0}^{K - 1} \big( d^3 \sigma_{T - t_{k + 1}}^{-6} \Delta_k^3  + d^7 \lambda_{T - t_k}^{8} \sigma_{T - t_k}^{-16} \Delta_k^4 \big)$;
        \item $\TT_3 \lesssim \sum_{k = 0}^{K - 1} \big( d^3 \sigma_{T - t_{k + 1}}^{-6} \Delta_k^4  + d^7 \lambda_{T - t_k}^{8} \sigma_{T - t_k}^{-16} \Delta_k^5 \big)$;
        \item $\TT_4 \lesssim \sum_{k = 0}^{K - 1} \big( d^3 \sigma_{T - t_{k + 1}}^{-6} \Delta_k^3  + d^7 \lambda_{T - t_k}^{8} \sigma_{T - t_k}^{-16} \Delta_k^4 \big)$.
    \end{enumerate}
\end{lemma}


\noindent 
As a consequence of \cref{lemma:T1234}, we reach
\begin{align*}
    \KL(Q_{T - \delta} \parallel \overline Q_{T - \delta}) \lesssim  \sum_{k = 0}^{K - 1} \big( d^3 \sigma_{T - t_{k + 1}}^{-6} \Delta_k^3  + d^7 \lambda_{T - t_k}^{8} \sigma_{T - t_k}^{-16} \Delta_k^4 \big). 
\end{align*}

Let $K_0 = \inf\{k: T - t_k \leq 1\}$, look at the indices that are above and below $K_0$ separately. 
Specifically, for all $k \leq K_0 - 1$ we have $\sigma_{T - t_k}^{-2} \lesssim 1$. 
In addition, for all $K - 1 \geq k \geq K_0$ we have $\sigma_{T - t_k}^{-2} \lesssim (T - t_k)^{-1}$, and hence under Assumption \ref{assumption:step-size} we have $d^3 \sigma_{T - t_{k + 1}}^{-6} \Delta_k^3  + d^7 \lambda_{T - t_k}^{8} \sigma_{T - t_k}^{-16} \Delta_k^4 \lesssim d^3 \kappa^2 \Delta_k + d^7 \lambda_{T - t_k}^{8} \sigma_{T - t_k}^{-4} \kappa^3 \Delta_k$. Putting these together yields
\begin{align*}
    & \sum_{k = 0}^{K - 1} \big( d^3 \sigma_{T - t_{k + 1}}^{-6} \Delta_k^3  + d^7 \lambda_{T - t_k}^{8} \sigma_{T - t_k}^{-16} \Delta_k^4 \big) \\
    & = \sum_{k = 0}^{K_0 - 1} \big( d^3 \sigma_{T - t_{k + 1}}^{-6} \Delta_k^3  + d^7 \lambda_{T - t_k}^{8} \sigma_{T - t_k}^{-16} \Delta_k^4 \big) +  \sum_{k = K_0}^{K - 1} \big(  d^3 \kappa^2 \Delta_k + d^7 \lambda_{T - t_k}^{8} \sigma_{T - t_k}^{-4} \kappa^3 \Delta_k \big) \\
    & \lesssim  d^3  \kappa^2 T  + d^7  \kappa^3 (\delta^{-1} + T), 
\end{align*}
as claimed in the lemma. 


\subsection{Proof of \cref{lemma:discretized-KL}}       
\label{sec:proof-lemma:discretized-KL}


Denoting by $\overline Q_{T - \delta}^{\dis}$ the distribution of $(\overline Y_{t_k})_{0 \leq k \leq K}$, we observe that 
\begin{align}
\label{eq:term1-and-2}
    \KL(Q^{\dis}_{T - \delta} \parallel \wh Q^{\dis}_{T - \delta}) = \underbrace{ \int \dd Q_{T - \delta}^{\dis} \log \frac{\dd Q_{T - \delta}^{\dis}}{\dd \overline Q_{T - \delta}^{\dis}}}_{(\text{i})} + \underbrace{\int \dd Q_{T - \delta}^{\dis} \log \frac{\dd \overline Q_{T - \delta}^{\dis}}{\dd \wh Q_{T - \delta}^{\dis}}}_{(\text{ii})}, 
\end{align}
leaving us with two terms to control. 

\subsubsection*{Bounding the term (\text{i})}
Note that the term (\text{i}) is essentially $(\text{i}) = \KL(Q_{T - \delta}^{\dis} \parallel \overline Q_{T - \delta}^{\dis})$. 
Recall that the data processing inequality \citep[Theorem 7.2]{polyanskiy2020information} asserts that 
$\KL\big(P_{f(X)} \parallel P_{f(Y)}\big) \leq \KL(P_X \parallel P_Y)$ holds for  any two random objects $X$ and $Y$ on the same space and any mapping $f$, with $P_Z$ the distribution of $Z$. 
Taking this together with \cref{lemma:Q-barQ}, we reach
\begin{align*}
    \int \dd Q_{T - \delta}^{\dis} \log \frac{\dd Q_{T - \delta}^{\dis}}{\dd \overline Q_{T - \delta}^{\dis}} \leq \int \dd Q_{T - \delta} \log \frac{\dd Q_{T - \delta}}{\dd \overline Q_{T - \delta}}  \lesssim \,  d^3  \kappa^2 T  + d^7  \kappa^3 (\delta^{-1} + T). 
\end{align*}

\subsubsection*{Bounding the term (\text{ii})}
Recall that $\wh Q_{T - \delta}$ is the distribution of process \eqref{eq:hat-X-tk} and $\overline Q_{T - \delta}$ is that of process \eqref{eq:bar-X-tk}. 
By \cref{eq:hat-X-tk} and \cref{eq:bar-X-tk}, we see that
for each $k \in \{0,1,\dots,K - 1\}$, 
\begin{align}
\label{eq:bar-hat-updates}
\begin{split}
      & \overline Y_{t_{k + 1}} =  e^{\Delta_k} \overline  Y_{t_k} + (e^{\Delta_k} - e^{-\Delta_k}) s(t_k, \overline Y_{t_k} + g_{t_k, t_{k + 1}}) + \sqrt{2} \int_0^{\Delta_k} e^{\Delta_k - r} \dd W_{t_k + r}, \\
     & \wh Y_{t_{k + 1}} =  e^{\Delta_k} \wh  Y_{t_k} + (e^{\Delta_k} - e^{-\Delta_k}) \widehat s(t_k, \wh Y_{t_k} + g_{t_k, t_{k + 1}}) + \sqrt{2} \int_0^{\Delta_k} e^{\Delta_k - r} \dd W_{t_k + r}, 
\end{split}
\end{align}
where we recall that $g_{t_k, t_{k + 1}}$ is defined in \cref{eq:double-index-g}. 
Note that $g_{t_k,t_{k+1}}$ and 
$\sqrt{2} \int_0^{\Delta_k} e^{\Delta_k - r} \dd W_{t_k + r}$ are correlated Gaussian random vectors, 
which admit simpler expressions. 
Specifically, from \cref{lemma:BM-integral-cov}, we can write 
\begin{align}
    g_{t_k, t_{k + 1}} = \zeta_{k, 1} g_{k, 1}, \qquad  \sqrt{2} \int_0^{\Delta_k} e^{\Delta_k - r} \dd W_{t_k + r} = \zeta_{k, 2} g_{k, 1} + \zeta_{k, 3} g_{k, 2}, 
\end{align}
where $g_{k, 1}, g_{k, 2} \overset{\mathrm{i.i.d.}}{\sim} \cN(0, I_d)$, and 
\begin{align}
\label{eq:three-zetas}
    & \zeta_{k, 1} = \frac{2\sqrt{2} f_1(\Delta_k)^{1/2}}{e^{\Delta_k} - e^{-\Delta_k}}, \qquad \zeta_{k, 2} = \frac{\sqrt{2} f_3(\Delta_k)}{f_1(\Delta_k)^{1/2}}, \qquad \zeta_{k, 3} = \sqrt{2f_2(\Delta_k) - \frac{2f_3(\Delta_k)^2}{f_1(\Delta_k)}}. 
\end{align}
In the above display, we recall the definitions of $f_1, f_2$ and $f_3$  in \cref{eq:covariance_functions}. 

For every $0\leq k \leq K$,  denote by $Q_k$ the distribution of $(Y_{t_0}, Y_{t_1}, \cdots,  Y_{t_k})$, $\overline Q_k$ the distribution of $(\overline Y_{t_0}, \overline Y_{t_1}, \cdots, \overline Y_{t_k})$, and $\wh Q_k$ that of $(\wh Y_{t_0}, \wh Y_{t_1}, \cdots, \wh Y_{t_k})$. 
Therefore, it follows from \cref{eq:bar-hat-updates} that
\begin{align*}
    & \int \dd Q_{T - \delta}^{\dis} \log \frac{\dd \overline Q_{T - \delta}^{\dis}}{\dd \wh Q_{T - \delta}^{\dis}} \\
    & = \sum_{k = 0}^{K - 1} \int \dd Q_{k + 1} \log \frac{\dd \overline Q_{k + 1}(Y_{t_{k + 1}} \mid Y_{t_0}, \cdots, Y_{t_{k}}) }{\dd \wh Q_{k + 1} (Y_{t_{k + 1}} \mid Y_{t_0}, \cdots, Y_{t_k})} \\
    & = \sum_{k = 0}^{K - 1} \int \dd Q_{k + 1} \log \frac{\int \exp\big(-\|Y_{t_{k + 1}} - e^{\Delta_k} Y_{t_k} - (e^{\Delta_k} - e^{-\Delta_k}) s(t_k, Y_{t_k} + \zeta_{k, 1} g) - \zeta_{k, 2} g\|_2^2 / (2\zeta_{k, 3}^2)\big) \phi(g) \dd g}{\int \exp\big(-\|Y_{t_{k + 1}} - e^{\Delta_k} Y_{t_k} - (e^{\Delta_k} - e^{-\Delta_k}) \widehat s(t_k, Y_{t_k} + \zeta_{k, 1} g) - \zeta_{k, 2} g\|_2^2 / (2\zeta_{k, 3}^2)\big) \phi(g) \dd g},
\end{align*}
where $\phi(\cdot)$ denotes the probability density function of $\cN(0, I_d)$. 
Recalling that $s(t, x) = \sigma_{T - t}^{-2} (-x + \lambda_{T - t} m(t, x))$ and $\widehat s(t, x) = \sigma_{T - t}^{-2} (-x + \lambda_{T - t} \widehat m(t, x))$ (with $m$ and $\widehat m$ introduced in Remark \ref{remark:m}),  
we can further deduce that 
\begin{align}
\label{eq:(ii)-form}
    & \int \dd Q_{T - \delta}^{\dis} \log \frac{\dd \overline Q_{T - \delta}^{\dis}}{\dd \wh Q_{T - \delta}^{\dis}} \\
    & = \sum_{k = 0}^{K - 1} \int \dd Q_{k + 1} \log \frac{\int \exp \big( \eta_k \langle m(t_k, Y_{t_k} + \zeta_{k, 1} g), v_k - \kappa_k g \rangle - \gamma_k \|m( t_k, Y_{t_k} + \zeta_{k, 1} g)\|_2^2 / 2 \big) \phi_{\xi_k, \tau_k}(g) \dd g}{\int \exp \big( \eta_k \langle \widehat m(t_k, Y_{t_k} + \zeta_{k, 1} g), v_k - \kappa_k g \rangle - \gamma_k \|\widehat m( t_k, Y_{t_k} + \zeta_{k, 1} g)\|_2^2 / 2 \big) \phi_{\xi_k, \tau_k}(g) \dd g}, \nonumber
\end{align}
where $\phi_{\xi, \tau}(\cdot)$ denotes the probability density function for $\cN(\xi, \tau^2 I_d)$, 
 and we set
\begin{align}
\label{eq:lots-of-definitions}
\begin{split}
    & \eta_k = \zeta_{k, 3}^{-2} \sigma_{T - t_k}^{-2} \lambda_{T - t_k} (e^{\Delta_k} - e^{-\Delta_k}), \\
    & v_k = Y_{t_{k + 1}} + \big( \sigma_{T - t_k}^{-2} (e^{\Delta_k} - e^{-\Delta_k}) - e^{\Delta_k}  \big) Y_{t_k}, \\
    & \kappa_k = \zeta_{k, 2} - \sigma_{T - t_k}^{-2} (e^{\Delta_k} - e^{-\Delta_k}) \zeta_{k, 1}, \\
    & \gamma_k = \zeta_{k, 3}^{-2} \sigma_{T - t_k}^{-4} \lambda_{T - t_k}^2 (e^{\Delta_k} - e^{-\Delta_k})^2, \\
    & \tau_k^2 = (\zeta_{k, 3}^2 + \kappa_k^{2})^{-1} \zeta_{k, 3}^2, \\
    & \xi_k =  \zeta_{k, 3}^{-2} \tau_k^2 \kappa_k v_k. 
\end{split}
\end{align}
We find it helpful to single out the following useful upper bounds (the proofs are omitted as they follow from straightforward calculus techniques):
\begin{align}
\label{eq:many-bounds}
   & \eta_k \lesssim \sigma_{T - t_k}^{-2} \lambda_{T - t_k}, \qquad  |\kappa_k| \lesssim \Delta_k^{1/2}, \qquad  \gamma_k \lesssim \kappa \lambda_{T - t_k}^2, \qquad \tau_k^2 \lesssim 1.  
\end{align}
In addition, it can be verified that: if $g \sim \cN(\xi_k, \tau_k^2 I_d)$, then one can write
\begin{align}
\label{eq:def-nuk}
    Y_{t_k} + \zeta_{k, 1} g = Y_{t_k} + \zeta_{k, 1} \zeta_{k, 3}^{-2} \tau_k^2 \kappa_k \big(Y_{t_{k + 1}} - e^{\Delta_k} Y_{t_k} + \sigma_{T - t_k}^{-2} (e^{\Delta_k} - e^{-\Delta_k}) Y_{t_k} \big) + \zeta_{k, 1}\tau_k g', 
\end{align}
where $g' \sim \cN(0, I_d)$ is independent of $(Y_{t_k}, Y_{t_{k + 1}})$.

Now let us analyze the quantities $\zeta_{k, 1}, \zeta_{k, 2}, \zeta_{k, 3}$ and those defined in \cref{eq:lots-of-definitions} in the lemma below. 
\begin{lemma}
\label{lemma:bound-constants}
    Under the assumptions of \cref{lemma:discretized-KL}, it holds that 
    \begin{align*}
        |\zeta_{k, 1}\zeta_{k, 3}^{-2} \tau_k^2 \kappa_k|  &\leq 0.65, \\
         \sigma_{T - t_k}^{-2} \big(1 - e^{-2\Delta_k}\big)  &\leq 3.2 \sqrt{\Delta_k \kappa} , \\
        \sigma_{T - t_k}^{-2} \big(e^{\Delta_k} - e^{-\Delta_k}\big)   &\leq 3.25 \sqrt{\Delta_k \kappa}. 
    \end{align*}
\end{lemma}
\begin{proof}[Proof of \cref{lemma:bound-constants}]
    Recall that by Assumption \ref{assumption:step-size} we have $\Delta_k \leq \kappa < 1 / 4$.
    As a consequence, we have $1 - e^{-2\Delta_k} \leq 2 \Delta_k$ and $e^{\Delta_k} - e^{-\Delta_k} \leq 81 \Delta_k / 40$. 
    
    If $T - t_k \geq 1 / 2$, then $\sigma_{T - t_k}^{-2} \leq (1 - e^{-1})^{-1}$. 
    In this case, we have $\sigma_{T - t_k}^{-2} (1 - e^{-2\Delta_k}) \leq 3.2 \Delta_k$, and $\sigma_{T - t_k}^{-2} (e^{\Delta_k} - e^{-\Delta_k}) \leq 3.25 \Delta_k$. 
    On the other hand, if $T - t_k < 1 / 2$, then $\sigma_{T - t_k}^{-2} \leq 0.8 (T - t_k)^{-1}$, hence  
    $\sigma_{T - t_k}^{-2} (1 - e^{-2\Delta_k}) \leq 1.6 \Delta_k / (T - t_k) \leq 1.6 \sqrt{\kappa \Delta_k}$ and $\sigma_{T - t_k}^{-2} (e^{\Delta_k} - e^{-\Delta_k}) \leq 1.62 \Delta_k / (T - t_k) \leq 1.62 \sqrt{\kappa \Delta_k}$.
    This establishes the second and the third inequalities. 

    As for the first inequality, observe that $|\zeta_{k, 1}\zeta_{k, 3}^{-2} \tau_k^2 \kappa_k| = |\zeta_{k, 1} \kappa_k / (\zeta_{k, 3}^2 + \kappa_k^2)| \leq |\zeta_{k, 1} \zeta_{k, 3}^{-1}| / 2$.  
    When $\Delta_k \leq 1 / 4$, it holds that $\zeta_{k, 1} \in [0.8 \Delta_k^{1/2}, 0.9 \Delta_k^{1/2}]$ and $\zeta_{k, 3} \in [0.7 \Delta_k^{1/2}, 0.75 \Delta_k^{1/2}]$. 
    Therefore, we have $|\zeta_{k, 1}\zeta_{k, 3}^{-2} \tau_k^2 \kappa_k| \leq |\zeta_{k, 1} \zeta_{k, 3}^{-1}| / 2 \leq 0.65$. 
    %
\end{proof}

Denoting by $\nu_k$ the marginal distribution of the random vector in \cref{eq:def-nuk}, 
we provide an important property about $\nu_k$ in the next lemma. 
\begin{lemma}
\label{lemma:nuk-distribution}
    Under the assumptions of \cref{lemma:discretized-KL}, it holds that $\nu_k \overset{\mathrm{d}}{=} a_k Y_{t_{k + 1}} + b_k g$, where $a_k, b_k$ are quantities satisfying $|a_k - 1| \leq 3.1 \sqrt{\Delta_k \kappa}$ and $|b_k| \leq 3.5 \sqrt{\Delta_k}$. Here, $g \sim \cN(0, I_d)$ is independent of $Y_{t_{k + 1}}$, and we recall that $Y_{t_{k + 1}} \sim q_{T - t_{k + 1}}$.  
\end{lemma}
\begin{proof}[Proof of \cref{lemma:nuk-distribution}]
    Note that we can write 
    $Y_{t_{k}} = e^{-\Delta_k} Y_{t_{k + 1}} + \sqrt{1 - e^{-2\Delta_k}} G$, where $G \sim \cN(0, I_d)$ is independent of $Y_{t_{k + 1}}$. 
    Substituting this equation into \cref{eq:def-nuk} yields 
    \begin{align*}
        \nu_k & \overset{\mathrm{d}}{=}  \big(e^{-\Delta_k} + \sigma_{T - t_k}^{-2} (1 - e^{-2\Delta_k}) \zeta_{k, 1} \zeta_{k, 3}^{-2} \tau_k^2 \kappa_k \big) Y_{t_{k + 1}} + \zeta_{k, 1} \zeta_{k, 3}^{-2} \tau_k^2 \kappa_k \sqrt{1 - e^{-2\Delta_k}} \big(\sigma_{T - t_k}^{-2} (e^{\Delta_k} - e^{-\Delta_k}) - e^{\Delta_k} \big) G \\
        & \qquad + \sqrt{1 - e^{-2\Delta_k}} G + \zeta_{k, 1}\tau_k g' \\
       & \overset{\mathrm{d}}{=}  a_k Y_{t_{k + 1}} + b_k G,  
    \end{align*}
    where
    \begin{align*}
        & a_k = e^{-\Delta_k} + \sigma_{T - t_k}^{-2} (1 - e^{-2\Delta_k}) \zeta_{k, 1} \zeta_{k, 3}^{-2} \tau_k^2 \kappa_k, \\
        & b_k = \sqrt{\zeta_{k, 1}^2 \tau_k^2 + ({1 - e^{-2\Delta_k}}) \big(\zeta_{k, 1} \zeta_{k, 3}^{-2} \tau_k^2 \kappa_k  (\sigma_{T - t_k}^{-2} (e^{\Delta_k} - e^{-\Delta_k}) - e^{\Delta_k}) + 1\big)^2}. 
    \end{align*}
    Note that $|1 - e^{-2\Delta_k}| \leq 2\Delta_k \leq 2\kappa$ and $\tau_k^2 \leq 1$. Using these upper bounds and \cref{lemma:bound-constants},  we reach 
    \begin{align*}
         |a_k - 1| &\leq |e^{-\Delta_k} - 1| + \sigma_{T - t_k}^{-2} (1 - e^{-2\Delta_k}) |\zeta_{k, 1} \zeta_{k, 3}^{-2} \tau_k^2 \kappa_k| \leq 3.1 \sqrt{\Delta_k \kappa} \\
         |b_k| &\leq \sqrt{0.9^2 \Delta_k + 2 \Delta_k \big(1 + 0.65 \times (e^{1/4} + 3.25 \kappa) \big)^2 } \leq 3.5 \Delta_k^{1/2}
    \end{align*}
    as claimed. 
\end{proof}

By virtue of Assumption \ref{assumption:score} and \cref{lemma:nuk-distribution}, we know that 
\begin{align*}
    \EE_{y \sim \nu_k} \big[ \| m(t_k, y) - \widehat m(t_k, y) \|_2^2 \big] = \sigma_{T - t_k}^{4} \lambda_{T - t_k}^{-2} \EE_{y \sim \nu_k} \big[ \|s(t_k, y) - \widehat s(t_k, y)\|_2^2 \big]  \leq \sigma_{T - t_k}^{4} \lambda_{T - t_k}^{-2}\eps_{\mathsf{score},k}^2. 
\end{align*}
In the sequel, we make the convention that conditional on $(Y_{t_k}, Y_{t_{k + 1}})$, $g \sim \cN(\xi_k, \tau_k^2 I_d)$, where we recall that $(\xi_k, \tau_k^2)$ are defined in \cref{eq:lots-of-definitions}. Note that $\xi_k$ is a function of $(Y_{t_k}, Y_{t_{k + 1}})$. 
For $y_k, y_{k + 1} \in \RR^d$, we define
\begin{align*}
    p_k(y_k, y_{k + 1}) =  \PP\big( \|m(t_k, Y_{t_k} + \zeta_{k, 1} g) - \widehat m(t_k, Y_{t_k} + \zeta_{k, 1} g)\|_2 \geq \sigma_{T - t_k} \lambda_{T - t_k}^{-1 / 2}\eps_{\mathsf{score},k}^{1 / 2} \mid Y_{t_k} = y_k, Y_{t_{k + 1}} = y_{k + 1} \big). 
\end{align*}
Then by Chebyshev's inequality, one has $$\EE[p_k(Y_{t_k}, Y_{t_{k + 1}})] \leq  \sigma_{T - t_k}^{2} \lambda_{T - t_k}^{-1}\eps_{\mathsf{score},k}.$$ 
%
%
Conditioning on $(Y_{t_k}, Y_{t_{k + 1}})$, we introduce the conditional event 
\begin{align*}
    \cS_{Y_{t_k}, Y_{t_{k + 1}}}  =\,   \Big\{ g: \|m(t_k, Y_{t_k} + \zeta_{k, 1} g) - \widehat m(t_k, Y_{t_k} + \zeta_{k, 1} g)\|_2 \geq \sigma_{T - t_k} \lambda_{T - t_k}^{-1 / 2}\eps_{\mathsf{score},k}^{1/2} \Big\}. 
\end{align*}
Per the discussions above, we see that $\PP(\cS_{Y_{t_k}, Y_{t_{k + 1}}}^c) \leq p_k(Y_{t_k}, Y_{t_{k + 1}})$. 
For notational simplicity, we define 
\begin{align*}
    & N_k \coloneqq \int \mathbbm{1}\{g \in {\cS_{Y_{t_k}, Y_{t_{k + 1}}}}\} \exp \big( \eta_k \langle m(t_k, Y_{t_k} + \zeta_{k, 1} g), v_k - \kappa_k g \rangle - \gamma_k \|m(t_k, Y_{t_k} + \zeta_{k, 1} g)\|_2^2 / 2 \big) \phi_{\xi_k, \tau_k}(g) \dd g, \\
    & \widehat N_k \coloneqq \int \mathbbm{1}\{g \in {\cS_{Y_{t_k}, Y_{t_{k + 1}}}}\} \exp \big( \eta_k \langle \widehat m(t_k, Y_{t_k} + \zeta_{k, 1} g), v_k - \kappa_k g \rangle - \gamma_k \|\widehat m(t_k, Y_{t_k} + \zeta_{k, 1} g)\|_2^2 / 2 \big) \phi_{\xi_k, \tau_k}(g) \dd g, \\
    & N_k^c \coloneqq \int \mathbbm{1}\{g \in {\cS^c_{Y_{t_k}, Y_{t_{k + 1}}}}\} \exp \big( \eta_k \langle m(t_k, Y_{t_k} + \zeta_{k, 1} g), v_k - \kappa_k g \rangle - \gamma_k \|m(t_k, Y_{t_k} + \zeta_{k, 1} g)\|_2^2 / 2 \big) \phi_{\xi_k, \tau_k}(g) \dd g, \\
    & \widehat N_k^c \coloneqq \int \mathbbm{1}\{g \in {\cS^c_{Y_{t_k}, Y_{t_{k + 1}}}}\} \exp \big( \eta_k \langle \widehat m(t_k, Y_{t_k} + \zeta_{k, 1} g), v_k - \kappa_k g \rangle - \gamma_k \|\widehat m(t_k, Y_{t_k} + \zeta_{k, 1} g)\|_2^2 / 2 \big) \phi_{\xi_k, \tau_k}(g) \dd g, \\
    & \widehat D_k \coloneqq \int \exp \big( \eta_k \langle \widehat m(t_k, Y_{t_k} + \zeta_{k, 1} g), v_k - \kappa_k g \rangle - \gamma_k \|\widehat m(t_k, Y_{t_k} + \zeta_{k, 1} g)\|_2^2 / 2 \big) \phi_{\xi_k, \tau_k}(g) \dd g, 
\end{align*}
which clearly obey $\widehat D_k^{-1} (\widehat N_k + \widehat N_k^c) = 1$, and 
%
\begin{align*}
     \log \frac{\int \exp \big( \eta_k \langle m(t_k, Y_{t_k} + \zeta_{k, 1} g), v_k - \kappa_k g \rangle - \gamma_k \|m( t_k, Y_{t_k} + \zeta_{k, 1} g)\|_2^2 / 2 \big) \phi_{\xi_k, \tau_k}(g) \dd g}{\int \exp \big( \eta_k \langle \widehat m(t_k, Y_{t_k} + \zeta_{k, 1} g), v_k - \kappa_k g \rangle - \gamma_k \|\widehat m( t_k, Y_{t_k} + \zeta_{k, 1} g)\|_2^2 / 2 \big) \phi_{\xi_k, \tau_k}(g) \dd g} = \log \frac{N_k + N_k^c}{\widehat D_k}. 
\end{align*}
Hence, in order to upper bound term $(ii)$ (cf.~\cref{eq:term1-and-2}), it suffices to upper bound $\widehat D_k^{-1}(N_k + N_k^c)$, 
towards which 
we intend to upper bound $\widehat D_k^{-1} |N_k - \widehat N_k|$ and $\widehat D_k^{-1} |N_k^c - \widehat N_k^c|$ separately.

Let us start with the first term $\widehat D_k^{-1} |N_k - \widehat N_k|$. 
Note that for $a_1, a_2 \in \RR$, we have $|e^a - e^b| \leq e^{\max\{a, b\}}|a - b|$. 
As a consequence,
when $g$ falls inside $\cS_{Y_{t_k, t_{k + 1}}}$, we have 
\begin{align}
\label{eq:g-S-exp}
    & \widehat D_k^{-1}|N_k - \widehat N_k|   \\
    & \leq \frac{\eps_{\mathsf{score},k}^{1 / 2}\sigma_{T - t_k}}{\widehat D_k \lambda_{T - t_k}^{1 / 2}} {\int e^ { 2\eta_k \sqrt{d} (\|v_k\|_2 + \kappa_k \|g - \xi_k\|_2 + \kappa_k\|\xi_k\|_2)  }    (\eta_k\|v_k\|_2 +  \eta_k\|\kappa_k \xi_k\|_2 + \eta_k\kappa_k \|g - \xi_k\|_2 + 3\gamma_k\sqrt{d} / 2)\phi_{\xi_k, \tau_k}(g) \dd g}. \nonumber 
\end{align}
Note that 
\begin{align}
    & \int e^ { 2\eta_k \sqrt{d} (\|v_k\|_2 + \kappa_k \|g - \xi_k\|_2 + \kappa_k\|\xi_k\|_2)  }    (\eta_k\|v_k\|_2 +  \eta_k\|\kappa_k \xi_k\|_2 + \eta_k\kappa_k \|g - \xi_k\|_2 + 3\gamma_k\sqrt{d} / 2)\phi_{\xi_k, \tau_k}(g) \dd g \nonumber \\
    & \leq  \int e^ { 2\eta_k \sqrt{d} (\|v_k\|_2 + \kappa_k\|\xi_k\|_2) + 4\eta_k \kappa_k d + \eta_k \kappa_k \|g - \xi_k\|_2^2 /4 }    (\eta_k\|v_k\|_2 +  \eta_k\|\kappa_k \xi_k\|_2 + \eta_k\kappa_k \|g - \xi_k\|_2 + 3\gamma_k\sqrt{d} / 2)\phi_{\xi_k, \tau_k}(g) \dd g \nonumber \\
    & \leq \frac{e^ { 2\eta_k \sqrt{d} (\|v_k\|_2 + \kappa_k\|\xi_k\|_2) + 4\eta_k \kappa_k d}}{(1 -  \tau_k^2 \eta_k \kappa_k / 2)^{d / 2}} \Big(\eta_k\|v_k\|_2 +  \eta_k\|\kappa_k \xi_k\|_2 + 3\gamma_k\sqrt{d} / 2 + \eta_k \kappa_k \big( \frac{d \tau_k^2}{1 - \tau_k^2 \eta_k \kappa_k / 2 } \big)^{1/2}\Big). \label{eq:upper-bound-DN1}
\end{align}
Note that the validity of \cref{eq:upper-bound-DN1} is conditional on $\eta_k \kappa_k \tau_k^2 < 2$: Only under this condition can we apply Gaussian integral to derive the last upper bound.  
We verify this condition in the next lemma. 
\begin{lemma}
\label{lemma:a-condition}
    Under the assumptions of \cref{lemma:discretized-KL}, it holds that $1 - \eta_k \kappa_k \tau_k^2 / 2 \geq 0.4$ for all $0\leq k \leq K - 1$. 
\end{lemma}
\begin{proof}[Proof of \cref{lemma:a-condition}]
    Note that 
    \begin{align*}
        \eta_k \kappa_k \tau_k^2 = \frac{\kappa_k \lambda_{T - t_k} (e^{\Delta_k} - e^{-\Delta_k})}{\sigma_{T - t_k}^2 (\zeta_{k, 3}^2 + \kappa_k^2)}. 
    \end{align*}
    %
    Inspecting the proof of \cref{lemma:bound-constants}, we see that: under the current assumptions, we have $\sigma_{T - t_k}^{-2} (e^{\Delta_k} - e^{-\Delta_k}) \leq 3.25 \sqrt{\Delta_k \kappa}$,  $\zeta_{k, 1} \in [0.8 \Delta_k^{1/2}, 0.9 \Delta_k^{1/2}]$ and $\zeta_{k, 3} \in [0.7 \Delta_k^{1/2}, 0.75 \Delta_k^{1/2}]$. 
    As a consequence, one has
    \begin{align*}
        \Big| \frac{\kappa_k \lambda_{T - t_k} (e^{\Delta_k} - e^{-\Delta_k})}{\sigma_{T - t_k}^2 (\zeta_{k, 3}^2 + \kappa_k^2)} \Big| \leq \Big|\frac{ (e^{\Delta_k} - e^{-\Delta_k})}{2 \sigma_{T - t_k}^2 \zeta_{k, 3}}\Big| \leq 2.4 \sqrt{\kappa},   
    \end{align*}
    which is no larger than $1.2$ when $\kappa < 1 / 4$. 
\end{proof}
\noindent 
Note that under Assumption \ref{assumption:score}, the quantity $\widehat D_k$ admits the following lower bound: 
\begin{align}
\label{eq:exp-eta-m-hat}
    \widehat D_k & \geq  \int \exp \big( - 2\eta_k\sqrt{d}\|v_k\|_2 - 2 \eta_k \kappa_k \sqrt{d} \|g - \xi_k\|_2 - 2\eta_k \kappa_k \sqrt{d} \|\xi_k\|_2 - 2 \gamma_k d \big) \phi_{\xi_k, \tau_k}(g) \dd g \nonumber \\
    & \geq  \exp \big( - 2\eta_k\sqrt{d}\|v_k\|_2 - \eta_k \kappa_k d - 2\eta_k \kappa_k \sqrt{d} \|\xi_k\|_2 - 2 \gamma_k d  \big) / (1 + \tau_k^2 \eta_k \kappa_k)^{d} \\
    & \gtrsim  \exp \big( - 2\eta_k\sqrt{d}\|v_k\|_2 - \eta_k \kappa_k d - 2\eta_k \kappa_k \sqrt{d} \|\xi_k\|_2 - 2 \gamma_k d  \big), \nonumber
\end{align}
where the last inequality is due to the upper bound $(1 + \tau_k^2 \eta_k \kappa_k)^{d} \leq \exp(d \tau_k^2 \eta_k \kappa_k)$, which by the proof of \cref{lemma:a-condition} is  no larger than $\exp(2.4 \sqrt{\kappa} d) \lesssim 1$. Here, we utilize the assumption that $\kappa d^2 \lesssim 1$. 
Again by inspecting the proof of \cref{lemma:a-condition}, we see that $(1 - \tau_k^2 \eta_k \kappa_k / 2)^{d / 2} \geq (1 - 1.2\sqrt{\kappa})^{d / 2} \geq \exp(-\sqrt{\kappa} d)$, which by Assumption \ref{assumption:step-size} is lower bounded by a positive numerical constant. 
Using this lower bound, we arrive at the following conclusion:  
\begin{align}
\label{eq:upper-bound-eq57}
\begin{split}
    & \mbox{The last line of \cref{eq:upper-bound-DN1}} \\
    & \lesssim {e^ { 2\eta_k \sqrt{d} (\|v_k\|_2 + \kappa_k\|\xi_k\|_2) + 4\eta_k \kappa_k d}} \Big(\eta_k\|v_k\|_2 +  \eta_k\|\kappa_k \xi_k\|_2 + 3\gamma_k\sqrt{d} / 2 + \eta_k \kappa_k  \tau_k d^{1/2}\Big). 
\end{split}
\end{align}
Putting together \cref{eq:upper-bound-DN1,eq:exp-eta-m-hat,eq:upper-bound-eq57} and using \cref{eq:many-bounds}, we have
\begin{align*}
    & \widehat D_k^{-1}|N_k - \widehat N_k| \\
    & \lesssim \frac{\eps_{\mathsf{score},k}^{1 / 2}\sigma_{T - t_k}}{\lambda_{T - t_k}^{1 / 2}} {e^ { 4\eta_k \sqrt{d} (\|v_k\|_2 + \kappa_k\|\xi_k\|_2) + 5\eta_k \kappa_k d + 2\gamma_k d}} \Big(\eta_k\|v_k\|_2 +  \eta_k\|\kappa_k \xi_k\|_2 + 3\gamma_k\sqrt{d} / 2 + \eta_k \kappa_k  \tau_k d^{1/2}\Big). 
\end{align*}
Taking the expectation over $Q_{k + 1}$ leads to
\begin{align}
\label{eq:hatD-N-hatN}
    \EE_{Q_{k + 1}}\big[ \widehat D_k^{-1}|N_k - \widehat N_k| \big] \lesssim \frac{\eps_{\mathsf{score},k}^{1 / 2}\sigma_{T - t_k}}{\lambda_{T - t_k}^{1 / 2}} \exp(d \sqrt{\kappa} )\kappa^{1/2} d^{1/2} \lesssim \frac{\eps_{\mathsf{score},k}^{1 / 2} \kappa^{1/2}\sigma_{T - t_k}  d^{1/2}}{\lambda_{T - t_k}^{1 / 2}}.
\end{align}
%

We then move on to control $\EE_{Q_{k + 1}}\big[\widehat D_k^{-1} |N_k^c - \widehat N_k^c|\big]$. 
Once again using the fact that for any $a_1, a_2 \in \RR$, we have $|e^a - e^b| \leq e^{\max\{a, b\}}|a - b|$.
As a result, 
\begin{align*}
    & \widehat D_k^{-1}|N_k^c - \widehat N_k^c| \\
    & \leq \widehat D_k^{-1} \int \mathbbm{1}\{g \in {\cS^c_{Y_{t_k}, Y_{t_{k + 1}}}}\} e^{ 2\eta_k \sqrt{d} (\|v_k\|_2 + \kappa_k \|g - \xi_k\|_2 + \kappa_k\|\xi_k\|_2)  } \big( 3\eta_k \sqrt{d} (\|v_k\|_2 + \kappa_k \|g\|_2 ) + 2\gamma_k d \big) \phi_{\xi_k, \tau_k}(g) \dd g \\
    & \leq \widehat D_k^{-1} \PP\big(\cS^c_{Y_{t_k}, Y_{t_{k + 1}}} \big)^{1/2} \Big( \int e^{4\eta_k \sqrt{d} (\|v_k\|_2 + \kappa_k \|g - \xi_k\|_2 + \kappa_k\|\xi_k\|_2)  } \big( 3\eta_k \sqrt{d} (\|v_k\|_2 + \kappa_k \|g\|_2) + 2\gamma_k d \big)^2 \phi_{\xi_k, \tau_k}(g) \dd g\Big)^{1/2}, 
\end{align*}
where the last inequality above arises from the Cauchy-Schwarz inequality. 
Recall that $\PP(\cS_{Y_{t_k}, Y_{t_{k + 1}}}^c) \leq p_k(Y_{t_k}, Y_{t_{k + 1}})$, which satisfies $\EE[p_k(Y_{t_k}, Y_{t_{k + 1}})] \leq \sigma_{T - t_k}^{2} \lambda_{T - t_k}^{-1}\eps_{\mathsf{score},k}$. 
Taking the expectation over $Q_{k + 1}$ gives 
\begin{align}
\label{eq:hatD-Nc-hatNc}
    \EE_{Q_{k + 1}} \big[ \widehat D_k^{-1}|N_k^c - \widehat N_k^c| \big] \lesssim \frac{\eps_{\mathsf{score},k}^{1 / 2} \kappa^{1/2}\sigma_{T - t_k} d}{\lambda_{T - t_k}^{1 / 2}}. 
\end{align}

Finally,  put together \cref{eq:hatD-N-hatN,eq:hatD-Nc-hatNc} to demonstrate that 
\begin{align*}
    (ii) & = \EE_{Q_{k + 1}} \big[ \log \big(\widehat D_k^{-1} (N_k + N_k^c)\big) \big] \\ & \leq   \EE_{Q_{k + 1}} \big[ \log \big(1 + \widehat D_k^{-1} |N_k - \widehat N_k| + \widehat D_k^{-1} |N_k^c - \widehat N_k^c|\big) \big] \\
    & \leq \EE_{Q_{k + 1}} \big[\widehat D_k^{-1} |N_k - \widehat N_k|\big] + \EE_{Q_{k + 1}} \big[\widehat D_k^{-1} |N_k^c - \widehat N_k^c|\big] \\
    & \lesssim \frac{\eps_{\mathsf{score},k}^{1 / 2} \kappa^{1/2}\sigma_{T - t_k} d}{\lambda_{T - t_k}^{1 / 2}}, 
\end{align*}
thus concluding the proof.

\subsection{Proof of \cref{lemma:T1234}}
\label{sec:proof-lemma:T1234}


\subsubsection*{Proof of the first point}

To prove the first point, we note that by \cref{eq:parametric-form},
\begin{align*}
    \TT_1 & =  \sum_{k = 0}^{K - 1} \int_{t_k}^{t_{k + 1}} (e^{t - t_k} - e^{-t + t_k})^2 \EE\big[ \|  s(t_k, Y_{t_k})  \|_2^2 \big] \dd t \\
    & \lesssim \sum_{k = 0}^{K - 1} \Delta_k^3 \sigma_{T - t_{k}}^{-2} \EE \big[ \| \EE[g \mid \lambda_{T - t_k} \theta + \sigma_{T - t_k}g = Y_{t_k}] \|_2^2 \big] \\
    & \lesssim  \sum_{k = 0}^{K - 1} \sigma_{T - t_{k}}^{-2} \Delta_k^3 d,  
\end{align*}
where the last upper bound is by Jensen's inequality. 

\subsubsection*{Proof of the second point}

By the triangle inequality, we can show that  
\begin{align*}
     \TT_2 & \leq \underbrace{\sum_{k = 0}^{K - 1} \int_{t_k}^{t_{k + 1}} \EE\big[ \big\| s(t, Y_t) - s(t_k, Y_{t_k}) - \sqrt{2} \nabla_x s(t_k, Y_{t_k})(B_t - B_{t_k}) \big\|_2^2 \big] \dd t}_{(i)} \\
     & + \underbrace{ \sum_{k = 0}^{K - 1} \int_{t_k}^{t_{k + 1}} \EE\big[ \big\| \nabla_x s(t_k, Y_{t_k} + b_{t_k, t})(B_t - B_{t_k}) - \nabla_x s(t_k, Y_{t_k})(B_t - B_{t_k}) \big\|_2^2 \big] \dd t}_{(ii)}. 
\end{align*}
In what follows, let us upper bound terms $(i)$ and $(ii)$ separately.

To upper bound term $(ii)$, we note that by the fundamental theorem of calculus, we have 
\begin{align*}
     & \nabla_x s(t_k, Y_{t_k} + b_{t_k, t}) (B_t - B_{t_k}) - \nabla_x s(t_k, Y_{t_k}) (B_t - B_{t_k}) \\
     & = \int_0^1 \nabla_x^2 s(t_k, Y_{t_k} + \eta b_{t_k, t}) [b_{t_k, t} \otimes (B_t - B_{t_k})] \dd \eta.  
\end{align*}
%
%
%
In view of the above decomposition, it suffices to separately upper bound $\EE[\|\nabla_x^2 s(t_k, Y_{t_k}) [b_{t_k, t} \otimes (B_t - B_{t_k})]\|_2^2]$ and $\EE[\|(\nabla_x^2 s(t_k, Y_{t_k} + \eta b_{t_k, t}) - \nabla_x^2 s(t_k, Y_{t_k})) [b_{t_k, t} \otimes (B_t - B_{t_k})]\|_2^2]$. 
Let us start with the first term. 
Invoking the second point of \cref{lemma:score-properties}, we see that 
\begin{align}
\label{eq:nabla-x-2-Y-tk}
    \nabla_x^2 s(t_k, Y_{t_k}) = -\frac{1}{\sigma_{T - t_k}^3} \EE \big[ (g - \EE[g \mid \lambda_{T - t_k} \theta + \sigma_{T - t_k}g = Y_{t_k}])^{\otimes 3} \mid  \lambda_{T - t_k} \theta + \sigma_{T - t_k}g = Y_{t_k}\big], 
\end{align}
where the expectation is over $(\theta, g) \sim q_0 \otimes \cN(0, I_d)$. 
For simplicity, we write $z_1 = b_{t_k, t}$ and $z_2 = B_t - B_{t_k}$. 
Observe that $z_1$ and $z_2$ are jointly normal with mean zero, with $(z_1, z_2) \perp Y_{t_k}$.  
As for the covariance structure, by \cref{lemma:BM-integral-cov}, it holds that 
\begin{align*}
\Cov[z_1, z_2] &= 2\sqrt{2} (e^{t - t_k} - e^{-t + t_k})^{-1} (e^{t - t_k} - t + t_k - 1) I_d,\\ 
\Cov[z_1] &= 8 (e^{t - t_k} - e^{-t + t_k})^{-2} f_1(t - t_k) I_d,\\
\Cov[z_2] &= (t - t_k) I_d.
\end{align*}
For $i \in [d]$, we define $L_i = g_i - \EE[g_i \mid \lambda_{T - t_k} \theta + \sigma_{T - t_k} g = Y_{t_k}]$. 
For all indices $j_1, j_2, \ell_1, \ell_2 \in [d]$ that are not paired up\footnote{If $j_1, j_2, \ell_1, \ell_2$ are paired up, then we must have $\{j_1, j_2, \ell_1, \ell_2\} = \{x, x, y, y\}$ or $\{x, x, x, x\}$ for some $x, y \in [d]$. }, it holds that
\begin{align*}
        \EE\big[ \EE\left[ L_iL_{j_1}L_{\ell_1} \mid \lambda_{T - t_k} \theta + \sigma_{T - t_k} g = Y_{t_k} \big] \EE\big[ L_iL_{j_2}L_{\ell_2} \mid \lambda_{T - t_k} \theta + \sigma_{T - t_k} g = Y_{t_k} \right] z_{1, j_1} z_{2, \ell_1} z_{1, j_2} z_{2, \ell_2} \big] = 0. 
    \end{align*}
Substituting the above equation into \cref{eq:nabla-x-2-Y-tk}, 
and applying the Cauchy-Schwartz inequality and the Jensen inequality, we arrive at
\begin{align}
\label{eq:ii-term1}
    \EE\big[ \|\nabla_x^2 s(t_k, Y_{t_k}) [b_{t_k, t} \otimes (B_t - B_{t_k})]\|_{\mathrm{F}}^2 \big] \lesssim \sigma_{T - t_k}^{-6} \Delta_k^2 d^3. 
\end{align}
We then upper bound $\EE[\|(\nabla_x^2 s(t_k, Y_{t_k} + \eta b_{t_k, t}) - \nabla_x^2 s(t_k, Y_{t_k})) [b_{t_k, t} \otimes (B_t - B_{t_k})]\|_2^2]$. 
Once again by the fundamental theorem of calculus, we have 
%
\begin{align}
\label{eq:ii-term2}
\begin{split}
    & \EE[\|(\nabla_x^2 s(t_k, Y_{t_k} + \eta b_{t_k, t}) - \nabla_x^2 s(t_k, Y_{t_k})) [b_{t_k, t} \otimes (B_t - B_{t_k})]\|_2^2] \\
   & \qquad = \EE\big[ \| \int_0^1 \nabla_x^3 s(t_k, Y_{t_k} + \kappa \eta b_{t_k, t}) [\eta b_{t_k, t} \otimes b_{t_k, t} \otimes (B_t - B_{t_k})] \dd \kappa \|_2^2 \big] \\
   & \qquad \lesssim \Delta_k^3 d^3 \int_0^1 \EE[ \|\nabla_x^3 s(t_k, Y_{t_k} + \kappa \eta b_{t_k, t})\|_{\mathrm{F}}^2 ] \dd \kappa \\
   & \qquad \lesssim \lambda_{T - t_k}^{8} \sigma_{T - t_k}^{-16} \Delta_k^3 d^7,
\end{split}
\end{align}
where the last inequality arises from  \cref{lemma:grad3-s}. 
Taking \cref{eq:ii-term1,eq:ii-term2} together, we derive an upper bound on term $(ii)$ as follows:
\begin{align}
\label{eq:(ii)-bound}
    (ii) \lesssim \sum_{k = 0}^{K - 1} \Big(\sigma_{T - t_k}^{-6} \Delta_k^3 d^3 + \lambda_{T - t_k}^{8} \sigma_{T - t_k}^{-16} \Delta_k^4 d^7 \Big). 
\end{align}

We now turn attention to term $(i)$.  Define 
\begin{align*}
    E_{t_k, t} = \EE\big[ \| s(t, Y_t) - s(t_k, Y_{t_k}) - \sqrt{2} \nabla_x s(t_k, Y_{t_k})(B_t - B_{t_k}) \|_2^2 \big]. 
\end{align*}
By the \ito formula, one has
\begin{subequations}
\begin{align*}
    & s(t, Y_t) - s(t_k, Y_{t_k}) - \sqrt{2} \nabla_x s(t_k, Y_{t_k}) (B_{t} - B_{t_k}) \\
   & =  \int_{t_k}^{t} \left[ \partial_{\tau} s(\tau, Y_{\tau}) + \nabla_x s(\tau, Y_{\tau}) (Y_{\tau} + 2 s(\tau, Y_{\tau}))  + \nabla_x^2 s(\tau, Y_{\tau}) [I_d] \right] \dd \tau \\
    & \qquad + \sqrt{2} \int_{t_k}^{t}  \big[ \nabla_x s(\tau, Y_{\tau}) - \nabla_x s(t_k, Y_{t_k}) \big] \dd B_{\tau}. 
\end{align*}
\end{subequations}
Hence, in order to upper bound $E_{t_k, t}$, it suffices to control the following two quantities:
\begin{subequations}
\begin{align}
    & c_{t_k, t}^{(1)} = \EE \left[ \Big\| \int_{t_k}^{t}  [ \nabla_x s(\tau, Y_{\tau}) - \nabla_x s(t_k, Y_{t_k})] \dd B_{\tau} \Big\|_2^2 \right], \label{eq:c-row1} \\
    & c_{t_k, t}^{(2)} = \EE \left[ \Big\| \int_{t_k}^{t} \left[ \partial_{\tau} s(\tau, Y_{\tau}) + \nabla_x s(\tau, Y_{\tau}) (Y_{\tau} + 2 s(\tau, Y_{\tau})) + \nabla_x^2 s(\tau, Y_{\tau}) [I_d] \right] \dd \tau \Big\|_2^2  \right],\label{eq:c-row2} 
\end{align}
\end{subequations}
and then invoke $E_{t_k, t} \lesssim c_{t_k, t}^{(1)} + c_{t_k, t}^{(2)}$.
In what follows, we upper bound $c_{t_k, t}^{(1)}$ and $c_{t_k, t}^{(2)}$ separately. 
\begin{itemize}
\item 
Note that $c_{t_k, t_k}^{(1)} = 0$. 
Hence, in order to upper bound $c_{t_k, t}^{(1)}$, we can take the differential of $c_{t_k, t}^{(1)}$ with respect to $t$. 
Specifically, according to the \ito formula, 
\begin{align}
\label{eq:dc1}
    \dd c_{t_k, t}^{(1)} & = 2 \EE \left[ \left< \int_{t_k}^{t}  [ \nabla_x s(\tau, Y_{\tau}) - \nabla_x s(t_k, Y_{t_k})] \dd B_{\tau}, \,  [ \nabla_x s(t, Y_{t}) - \nabla_x s(t_k, Y_{t_k})] \dd B_{t}\right> \right] \nonumber \\
    & \qquad + \EE \left[  \|\nabla_x s(t, Y_{t}) - \nabla_x s(t_k, Y_{t_k})\|_{\mathrm{F}}^2  \right]\dd t \\
    & =  3 \EE \left[  \|\nabla_x s(t, Y_{t}) - \nabla_x s(t_k, Y_{t_k})\|_{\mathrm{F}}^2  \right]\dd t. \nonumber
\end{align}
The term $\EE [  \|\nabla_x s(t, Y_{t}) - \nabla_x s(t_k, Y_{t_k})\|_{\mathrm{F}}^2  ]$ shall be bounded in the next lemma, whose proof is postponed to Appendix \ref{sec:proof-lemma:upper-bound-2}.  
\begin{lemma}
\label{lemma:upper-bound-2}
For $t_k \leq t \leq t_{k + 1}$, we define 
\begin{align*}
    M_{t_k, t} = \EE \left[ \|\nabla_x s(t, Y_t) - \nabla_x s(t_k, Y_{t_k})\|_{\mathrm{F}}^2 \right]. 
\end{align*}
Then, under the conditions of \cref{lemma:Q-barQ}, for all $t \in [t_k, t_{k + 1}]$ we have $M_{t_k, t} \lesssim d^3 \sigma_{T - t_{k + 1}}^{-6} \Delta_k$. 
\end{lemma}
%

Armed with \cref{lemma:upper-bound-2} and \cref{eq:dc1}, we conclude that for all $t \in [t_k, t_{k + 1}]$, 
\begin{align}
\label{eq:c1-upper-bound}
    c_{t_k, t}^{(1)} \lesssim d^3 \sigma_{T - t_{k + 1}}^{-6} \Delta_k^2. 
\end{align}

\item We now turn to establishing an upper bound on $c_{t_k, t}^{(2)}$, as defined in \cref{eq:c-row2}. 
Leveraging the Cauchy–Schwarz inequality, we obtain that for all $t \in [t_k, t_{k + 1}]$, 
\begin{align}
\begin{split}
    c_{t_k, t}^{(2)} & \lesssim  \EE\left[ (t - t_k) \int_{t_k}^{t} \big\| \partial_{\tau} s(\tau, Y_{\tau}) + \nabla_x s(\tau, Y_{\tau}) \big(Y_{\tau} + 2 s(\tau, Y_{\tau})\big) + \nabla_x^2 s(\tau, Y_{\tau}) [I_d] \big\|_2^2 \dd \tau  \right] \label{eq:c-tk-r-2-b} \\
    & \lesssim  (t - t_k) \int_{t_k}^{t} \left( \EE\big[\big\|\partial_{\tau} s(\tau, Y_{\tau}) \big\|_2^2\big] +  \EE\big[\big\| \nabla_x s(\tau, Y_{\tau}) \big(Y_{\tau} + s(\tau, Y_{\tau})\big)\big\|_2^2\big] + \EE\big[\big\|\nabla_x^2 s(\tau, Y_{\tau}) [I_d] \big\|_2^2\big]\right) \dd \tau. 
\end{split}
\end{align}
We develop separate upper bounds for the above summands in the lemma below; the proof is deferred to Appendix~\ref{sec:proof-lemma:four-summands}.  
\begin{lemma}
\label{lemma:four-summands}
    Under the conditions of \cref{lemma:Q-barQ}, the following upper bounds hold for all $\tau \in [0, T)$: 
    \begin{enumerate}
        \item $\EE\big[\|\partial_{\tau} s(\tau, Y_{\tau})\|_2^2\big] \lesssim d^3{\lambda_{T - \tau}^4 }{\sigma_{T - \tau}^{-6}} + d {\lambda_{T - \tau}^2 }{\sigma_{T - \tau}^{-4}}$; 
        \item $\EE\big[\| \nabla_x s(\tau, Y_{\tau}) s(\tau, Y_{\tau})\|_2^2\big] \lesssim {d^3}{\sigma_{T - \tau}^{-6}}$;
        \item $\EE\big[\| \nabla_x s(\tau, Y_{\tau}) Y_{\tau}\|_2^2\big] \lesssim {d^3}{\sigma_{T - \tau}^{-4}}$;  
        \item $\EE\big[\|\nabla_x^2 s(\tau, Y_{\tau}) [I_d] \|_2^2\big] \lesssim \sigma_{T - \tau}^{-6} d^3$. 
    \end{enumerate}
\end{lemma}
%
%
Substituting the upper bounds derived in \cref{lemma:four-summands} into \cref{eq:c-tk-r-2-b}, we conclude that for all $t \in [t_k, t_{k + 1}]$, 
\begin{align}
\label{eq:c-tk-r-2}
    c_{t_k, t}^{(2)} \lesssim d^3 \sigma_{T - t_{k + 1}}^{-6} \Delta_k^2. 
\end{align}
\end{itemize}
Combining the preceding bounds in \cref{eq:c1-upper-bound,eq:c-tk-r-2}, we obtain 
\begin{align}
\label{eq:E-tk-r-1}
    E_{t_k, t} \lesssim c_{t_k, t}^{(1)} + c_{t_k, t}^{(2)} \lesssim d^3 \sigma_{T - t_{k + 1}}^{-6} \Delta_k^2,  
\end{align}
which further implies that 
\begin{align}
\label{eq:(i)-bound}
    (i) \lesssim \sum_{k = 0}^{K - 1} d^3 \sigma_{T - t_{k + 1}}^{-6} \Delta_k^3. 
\end{align}
Putting \cref{eq:(ii)-bound,eq:(i)-bound} together results in
\begin{align*}
    \TT_2 \lesssim  \sum_{k = 0}^{K - 1} \Big( d^3 \sigma_{T - t_{k + 1}}^{-6} \Delta_k^3  + d^7 \lambda_{T - t_k}^{8} \sigma_{T - t_k}^{-16} \Delta_k^4  \Big), 
\end{align*}
thus completing the proof. 

\subsubsection*{Proof of the third point}

From the triangle inequality, we obtain 
\begin{align*}
         \TT_3 & \lesssim {\sum_{k = 0}^{K - 1} \int_{t_k}^{t_{k + 1}} \int_0^{t - t_k} \EE\big[ \big\| s(t_k + r, Y_{t_k + r}) - s(t_k, Y_{t_k}) - \sqrt{2} \nabla_x s(t_k, Y_{t_k})(B_{t_k + r} - B_{t_k}) \big\|_2^2 \big] \dd t} \\
     & \qquad + { \sum_{k = 0}^{K - 1} \int_{t_k}^{t_{k + 1}} \int_0^{t - t_k} \EE\Big[ \big\| \nabla_x s(t_k, Y_{t_k} + b_{t_k, t})(B_{t_k + r} - B_{t_k}) - \nabla_x s(t_k, Y_{t_k})(B_{t_k + r} - B_{t_k}) \big\|_2^2 \Big] \dd t}.
\end{align*}
Similar to the derivation of point 2, for all $t \in [0, t_{k + 1} - t_k]$ and all $r \in [0, t - t_k]$ it holds that 
\begin{align*}
     \EE\big[ \big\| s(t_k + r, Y_{t_k + r}) - s(t_k, Y_{t_k}) - \sqrt{2} \nabla_x s(t_k, Y_{t_k})(B_{t_k + r} - B_{t_k}) \big\|_2^2 \big] &\lesssim d^3 \sigma_{T - t_{k + 1}}^{-6} \Delta_k^2, \\
     \EE\big[ \big\| \nabla_x s(t_k, Y_{t_k} + b_{t_k, t})(B_{t_k + r} - B_{t_k}) - \nabla_x s(t_k, Y_{t_k})(B_{t_k + r} - B_{t_k}) \big\|_2^2 \big] &\lesssim \sigma_{T - t_k}^{-6} \Delta_k^2 d^3 + \lambda_{T - t_k}^{8} \sigma_{T - t_k}^{-16} \Delta_k^3 d^7. 
\end{align*}
%
The desired claim then follows.

\subsubsection*{Proof of the fourth point}

This is similar to the proof of the second point. We skip the proof for the sake of brevity. 


\subsection{Proof of \cref{lemma:upper-bound-2}}
\label{sec:proof-lemma:upper-bound-2}

\begin{proof}[Proof of \cref{lemma:upper-bound-2}]

By \ito's lemma we can write 
\begin{align}
\label{eq:17}
    \dd \nabla_x s(t, Y_t) = \partial_t \nabla_x s(t, Y_t) \dd t + \nabla_x^2 s(t, Y_t) (Y_t + 2s(t, Y_t)) \dd t + \sqrt{2} \nabla_x^2 s(t, Y_t) \dd B_t + \nabla_x^3 s(t, Y_t) [I_d] \dd t.
\end{align}
From the proof of \citet[Lemma 3]{benton2024nearly}, we deduce that for all $x \in \RR^d$ and $t \in [0, T)$,  
\begin{align*}
    \partial_t s(t, x) = - \big[ s(t, x) + \nabla_x s(t, x) x + \Delta s(t, x) + 2 \nabla_x s(t, x) s(t, x) \big].
\end{align*}
where $\Delta$ denotes the Laplace operator. 
Further taking the Jacobian of the above mapping, we obtain that
\begin{align}
\label{eq:18}
    \partial_t \nabla_x s(t, x) = - \big[ 2\nabla_x s(t, x) + \nabla_x^2 s(t, x) x + \nabla_x^3 s(t, X_t)[I_d] + 2 \nabla_x s(t, x)^2 + 2 \nabla_x^2 s(t, x) s(t, x)  \big]. 
\end{align}
Putting together \cref{eq:17,eq:18}, we derive 
    \begin{align}
    \label{eq:dnablas}
        \dd \nabla_x s(t, Y_t) = - \left[ 2\nabla_x s(t, Y_t) + 2 \nabla_x s(t, Y_t)^2 \right] \dd t + \sqrt{2}\nabla_x^2 s(t, Y_t) \dd B_t. 
    \end{align}
    With \cref{eq:dnablas}, we can analyze the differential of $M_{t, t_k}$ with respect to $t$. In particular, 
    \begin{align}
    \label{eq:dMttk}
       \dd M_{t_k, t} = -4 \EE\left[ \left< \nabla_x s(t, Y_t) - \nabla_x s(t_k, Y_{t_k}), \nabla_x s(t, Y_t) + \nabla_x s(t, Y_t)^2  \right> \right] \dd t + 2\EE\left[ \| \nabla_x^2 s(t, Y_t) \|_{\mathrm{F}}^2 \right] \dd t. 
    \end{align}
    By \cref{lemma:score-properties}, we obtain that (recall $m_g(t, x) = \EE[g \mid \lambda_{T - t} \theta + \sigma_{T - t} g = x]$)
    \begin{align}
    \label{eq:three-upper-bounds}
    \begin{split}
         \EE\left[ \|\nabla_x^2 s(t, Y_t)\|_{\mathrm{F}}^2 \right] &= \frac{1}{\sigma_{T - t}^{6}} \EE\left[ \| \EE[(g - m_g(t, Y_t))^{\otimes 3} \mid \lambda_{T - t} \theta + \sigma_{T - t} g = Y_t] \|_{\mathrm{F}}^2 \right] \lesssim \frac{d^3}{\sigma_{T - t}^6}, \\
         \EE\left[ \|\nabla_x s(t, Y_t)\|_{\mathrm{F}}^2 \right] &\lesssim \EE\left[ \| \sigma_{T - t}^{-2} I_d \|_{\mathrm{F}}^2 \right]  + \EE\left[ \| \sigma_{T - t}^{-2} \EE[(g - m_g(t, Y_t))^{\otimes 2} \mid \lambda_{T - t} \theta + \sigma_{T - t} g = Y_t] \|_{\mathrm{F}}^2 \right] \lesssim  \frac{d^2}{\sigma_{T - t}^4}, \\
         \EE\left[ \|\nabla_x s(t, Y_t)^2\|_{\mathrm{F}}^2 \right] &\lesssim \EE\left[ \| \sigma_{T - t}^{-4} I_d \|_{\mathrm{F}}^2 \right] + \EE\left[ \|  \sigma_{T - t}^{-4} \EE[(g - m_g(t, Y_t))^{\otimes 2} \mid \lambda_{T - t} \theta + \sigma_{T - t} g = Y_t]^2 \|_{\mathrm{F}}^2 \right] \lesssim \frac{d^4}{\sigma_{T - t}^8 }. 
    \end{split}
    \end{align}
    Applying the Cauchy-Schwartz inequality to \cref{eq:dMttk} and substituting in the upper bounds from \cref{eq:three-upper-bounds} give
    \begin{align*}
        \dd M_{t_k, t} & \lesssim \EE\left[ \|\nabla_x s(t, Y_t) - \nabla_x s(t_k, Y_{t_k}) \|_{\mathrm{F}}^2 \right]^{1/2} \cdot \EE\left[ \| \nabla_x s(t, Y_t) + \nabla_x s(t, Y_t)^2 \|_{\mathrm{F}}^2 \right]^{1/2} \dd t + \EE\left[ \| \nabla_x^2 s(t, Y_t) \|_{\mathrm{F}}^2 \right] \dd t \\
        & \lesssim \frac{d^3}{\sigma_{T - t}^6} \dd t. 
    \end{align*}
    Observe that $M_{t_k, t_k} = 0$. 
    As a consequence, for all $t \in [t_k, t_{k + 1}]$, it holds that $M_{t_k, t} \lesssim d^3 \sigma_{T - t_{k + 1}}^{-6} \Delta_k$. 
    The proof is complete. 
\end{proof}

\subsection{Proof of \cref{lemma:four-summands}}
\label{sec:proof-lemma:four-summands}


\begin{proof}[Proof of \cref{lemma:four-summands}, point 1]

By the third point of \cref{lemma:score-properties}, 
\begin{align*}
    -\partial_{\tau} s(\tau, Y_{\tau}) &= - \frac{\lambda_{T - \tau}}{\sigma_{T - \tau}^2} \EE\big[\theta \mid \lambda_{T - \tau} \theta + \sigma_{T - \tau} g = Y_{\tau}\big] + \frac{2\lambda_{T - \tau}^2}{\sigma_{T - \tau}^3} \EE\big[g \mid \lambda_{T - \tau} \theta + \sigma_{T - \tau} g = Y_{\tau}\big] \\
            & + \frac{\lambda_{T - \tau}}{\sigma_{T - \tau}^2}\EE\big[\big(\theta - m({\tau}, Y_{\tau})\big)\big(\cF(\theta, Y_{\tau}, {\tau}) - m_{\cF}\big)\mid \lambda_{T - \tau} \theta + \sigma_{T - \tau} g = Y_{\tau}\big],
\end{align*}
where $m_{\cF} = \EE[\cF(\theta, Y_{\tau}, {\tau}) \mid \lambda_{T - \tau} \theta + \sigma_{T - \tau} g = Y_{\tau}]$.
Under Assumption \ref{assumption:moments}, by Jensen's inequality we have  
\begin{align*}
    & \EE\left[ \| \EE[\theta \mid \lambda_{T - \tau} \theta + \sigma_{T - \tau} g = Y_{\tau}]\|_2^2\right] \lesssim d, \\
    & \EE\left[ \|\EE[g \mid \lambda_{T - \tau} \theta + \sigma_{T - \tau} g = Y_{\tau}]\|_2^2 \right] \lesssim d. 
\end{align*}
Recall that $\cF$ is defined in \cref{eq:B2-cF}. Conditional on $\lambda_{T - \tau}\theta + \sigma_{T - \tau} g =  Y_{\tau}$, we have 
    \begin{align}
    \label{eq:complicated-cF}
         &  \cF(\theta, Y_{\tau}, \tau) - \EE\big [ \cF(\theta, Y_{\tau}, \tau) \mid \lambda_{T - \tau}\theta + \sigma_{T - \tau} g = Y_{\tau}\big] \nonumber \\
         &=  \frac{\lambda_{T - \tau}^2}{\sigma_{T - \tau}^4} \|\theta\|_2^2 - \frac{\lambda_{T - \tau} + \lambda_{T - \tau}^3}{\sigma_{T - \tau}^4} \langle Y_{\tau}, \theta \rangle \nonumber \\
         & \qquad - \EE\left[ \frac{\lambda_{T - \tau}^2}{\sigma_{T - \tau}^4} \|\theta\|_2^2 - \frac{\lambda_{T - \tau} + \lambda_{T - \tau}^3}{\sigma_{T - \tau}^4} \langle Y_{\tau}, \theta \rangle \,\Big|\, \lambda_{T - \tau}\theta + \sigma_{T - \tau} g =  Y_{\tau} \right] \\
         & =  - \frac{\lambda_{T - \tau}}{\sigma_{T - \tau}} \langle \theta, g \rangle + \frac{\lambda_{T - \tau}^2}{\sigma_{T - \tau}^2} \|g\|_2^2  + \EE\left[ \frac{\lambda_{T - \tau}}{\sigma_{T - \tau}} \langle \theta, g \rangle - \frac{\lambda_{T - \tau}^2}{\sigma_{T - \tau}^2} \|g\|_2^2 \,\Big|\, \lambda_{T - \tau}\theta + \sigma_{T - \tau} g =  Y_{\tau} \right]. \nonumber
    \end{align}
Note that conditioning on $\theta$, $\langle \theta, g \rangle$ has conditional distribution $\cN(0, \|\theta\|_2^2)$. 
Therefore, 
\begin{align*}
    & \EE\left[ \big\| \EE\big[\big(\theta - m({\tau}, Y_{\tau})\big)\big(\cF(\theta, Y_{\tau}, {\tau}) - m_{\cF}\big)\mid \lambda_{T - \tau} \theta + \sigma_{T - \tau} g = Y_{\tau}\big] \big\|_2^2 \right] \\
    & \lesssim  \EE\left[ \EE\big[\|\theta - m({ \tau}, Y_{\tau})\|_2^2 \mid \lambda_{T - \tau} \theta + \sigma_{T - \tau} g = Y_{\tau}\big] \EE\big[ \big(\cF(\theta, Y_{\tau}, {\tau}) - m_{\cF}\big)^2 \mid \lambda_{T - \tau} \theta + \sigma_{T - \tau} g = Y_{\tau}\big]\right] \\
    & \lesssim  \EE\left[ \big\|\Theta - m( \tau, Y_{\tau})\big\|_2^4 \right]^{1/2} \EE\left[ \big(\cF(\Theta, Y_{\tau}, {\tau}) - m_{\cF}\big)^4  \right]^{1/2} \\
    & \lesssim   \frac{\lambda_{T - \tau}^2 d^3}{\sigma_{T - \tau}^2},
\end{align*}
where we write $Y_{\tau} = \lambda_{T - \tau} \Theta + \sigma_{T - \tau} G$ for $(\Theta, G) \sim q_0 \otimes \cN(0, 1)$. 
In the above display, the last inequality is by Jensen's inequality and Assumption \ref{assumption:moments}. 
The proof is thus complete. 
\end{proof}

\begin{proof}[Proof of \cref{lemma:four-summands}, point 2]
By \cref{lemma:score-properties}, we have 
\begin{align*}
    & \EE[\| \nabla_x s(\tau, Y_{\tau}) s(\tau, Y_{\tau})\|_2^2] \\
    & =  \EE\left[ \| ( \sigma_{T - \tau}^{-3} I_d - \lambda_{T - \tau}^2 \sigma_{T - \tau}^{-5} \Cov[\theta \mid \lambda_{T - \tau} \theta + \sigma_{T - \tau} g = Y_{\tau}]) \EE[g \mid \lambda_{T - \tau} \theta + \sigma_{T - \tau} g = Y_{\tau}] \|_2^2 \right] \\
    & \lesssim   \frac{1}{\sigma_{T - \tau}^6} \EE\left[ \|\EE[g \mid \lambda_{T - \tau} \theta + \sigma_{T - \tau} g = Y_{\tau}]\|_2^2 \right] \\
    & \qquad + \frac{1}{\sigma_{ T - \tau}^{6}} \EE\left[ \|\EE[(g - m_g( \tau, Y_{\tau}))^{\otimes 2} \mid \lambda_{T - \tau} \theta + \sigma_{T - \tau} g = Y_{\tau}] \EE[g \mid \lambda_{T - \tau} \theta + \sigma_{T - \tau} g = Y_{\tau}] \|_2^2 \right] \\
    & \lesssim  \, {d^3}{\sigma_{T - \tau}^{-6}},
\end{align*}
where we recall that $m_g(\tau, Y_{\tau}) = \EE[g \mid \lambda_{T - \tau} \theta + \sigma_{T - \tau} g = Y_{\tau}]$. 
  %
\end{proof}

\begin{proof}[Proof of \cref{lemma:four-summands}, point 3]
Next, we look at the third term $\EE[\| \nabla_x s(\tau, Y_{\tau}) Y_{\tau}\|_2^2]$. 
By the first point of \cref{lemma:score-properties}, we have 
\begin{align}
\label{eq:grad-s-X-four}
     \EE\big[\big\| \nabla_x s(\tau, Y_{\tau}) Y_{\tau}\big\|_2^2\big] 
    & =  \EE\big[\big\| \big( \sigma_{T - \tau}^{-2} I_d - \lambda_{T - \tau}^2 \sigma_{T - \tau}^{-4} \Cov[\theta \mid \lambda_{T - \tau} \theta + \sigma_{T - \tau} g = Y_{\tau}]\big) Y_{\tau}\big\|_2^2\big] \nonumber  \\
    & \lesssim  \frac{1}{\sigma_{T - \tau}^{4}} \EE[\|Y_{\tau}\|_2^2] + \frac{1}{\sigma_{T - \tau}^4} \EE\left[ 
    \|\EE[(g - m_g(\tau, Y_{\tau}))^{\otimes 2} \mid \lambda_{T - \tau} \theta + \sigma_{T - \tau} g = Y_{\tau}] Y_{\tau}\|_2^2  \right] \nonumber  \\
    & \leq   \frac{1}{\sigma_{T - \tau}^{4}} \EE[\|Y_{\tau}\|_2^2] + \frac{1}{\sigma_{T - \tau}^4} \EE\left[ \| \EE[gg^{\top} \mid \lambda_{T - \tau} \theta + \sigma_{T - \tau} g = Y_{\tau}] Y_{\tau}  \|_2^2 \right].  
\end{align}
Given that $(\mathbb{E}[M\mid \cF])^2 \preceq \mathbb{E}[M^2\mid \cF]$ for any random matrix $M$ and filtration $\cF$, we can derive
\begin{align}
\label{eq:gggg}
\begin{split}
    \EE\left[ \big\| \EE[gg^{\top} \mid \lambda_{T - \tau} \theta + \sigma_{T - \tau} g = Y_{\tau}] Y_{\tau}  \big\|_2^2 \right] &\leq \EE\big[Y_{\tau}^{\top} \EE\big[\|g\|_2^2gg^{\top} \mid \lambda_{T - \tau} \theta + \sigma_{T - \tau} g = Y_{\tau}\big] Y_{\tau}\big] \\
    &=  \EE\big[(Y_{\tau}^{\top} G)^2 \|G\|_2^2\big], 
\end{split}
\end{align}
where $Y_{\tau} = \lambda_{T - \tau} \Theta + \sigma_{T - \tau} G$ for $(\Theta, G) \sim q_0 \otimes \cN(0, I_d)$. 
Also, observe that
\begin{align}
\label{eq:X-tau-G-G}
    \EE\big[(Y_{\tau}^{\top} G)^2 \|G\|_2^2\big] \lesssim d^3. 
\end{align}
Substituting \cref{eq:X-tau-G-G,eq:gggg} into \cref{eq:grad-s-X-four} gives  
\begin{align*}
    \EE\big[\big\| \nabla_x s(\tau, Y_{\tau}) Y_{\tau}\big\|_2^2\big] \lesssim  \frac{d^3}{\sigma_{T - \tau}^4},  
\end{align*}
which completes the proof of the lemma. 
\end{proof}

\begin{proof}[Proof of \cref{lemma:four-summands}, point 4]

Finally, we upper bound $\EE[\|\nabla_x^2 s(\tau, Y_{\tau}) [I_d] \|_2^2]$. By the second point of \cref{lemma:score-properties}, the $i$-th entry of $\nabla_x^2 s(\tau, X_{\tau}) [I_d]$ admits the form  
\begin{align}
\label{eq:g-third-moment}
      \sum_{j \in [d]}\sigma_{T - \tau}^{-3} \EE[(g_i - m_g(\tau, Y_{\tau})_i)(g_j - m_g(\tau, Y_{\tau})_j)^2 \mid \lambda_{T - \tau} \theta + \sigma_{T - \tau}g = Y_{\tau}], 
\end{align}
where we recall that $m_g(t, x) = \EE[g \mid \lambda_{T - t} \theta + \sigma_{T - t} g = x]$, $m_g(t, x)_i$ is the $i$-th entry of $m_g(t, x)$ and $g_i$ is the $i$-th entry of $g$. 
Applying Jensen's inequality to \cref{eq:g-third-moment}, we conclude that $\EE[\|\nabla_x^2 s(\tau, Y_{\tau}) [I_d] \|_2^2] \lesssim \sigma_{T - \tau}^{-6} d^3$. 
This concludes the proof.
\end{proof}

\subsection{Proof of Corollary \ref{cor:example}}
\label{sec:proof-cor:example}

    Given $\delta$ and $\kappa$, we first construct a sequence of step sizes as follows: 
    \begin{itemize}
    \item set $\Delta_{K - 1} = \kappa \delta^2$. 
    
    \item For $k = K - 1, K - 2, \cdots, 1$, 
    \begin{align}
    \label{eq:stepsizes-example}
    \Delta_{k - 1} = 
    \begin{cases}
    \Delta_k (1 + \sqrt{\kappa \Delta_k})^2,\qquad  & \text{if }\Delta_k (1 + \sqrt{\kappa \Delta_k})^2 \leq \kappa; \\
    \kappa, & \text{else}.
    \end{cases}
    \end{align}
    %
    
\end{itemize}

    Next, we prove that the step sizes defined as above satisfies $$\Delta_k \leq \kappa \min \big\{1, (T - t_{k + 1})^2\big\}, 
    \qquad k = 0, 1, \cdots, K - 1.$$ 
    Let us prove this claim by induction. When $k = K - 1$, this is true by definition. Now suppose  $\Delta_k \leq \kappa \min \{1, (T - t_{k + 1})^2\}$ for some $k$, and we shall use this upper bound to prove $\Delta_{k - 1} \leq \kappa \min \{1, (T - t_{k})^2\}$. If $\Delta_k (1 + \sqrt{\kappa \Delta_k})^2 > \kappa$, then this is automatically true as $\Delta_{k - 1} = \kappa$. 
    Otherwise if $\Delta_k (1 + \sqrt{\kappa \Delta_k})^2 \leq \kappa$, we have $\Delta_{k - 1} = \Delta_k (1 + \sqrt{\kappa \Delta_k})^2$. 
    By induction hypothesis, $T - t_{k + 1} \geq \sqrt{\Delta_k / \kappa}$. Therefore, 
\begin{align*}
    T - t_k = T - t_{k + 1} + \Delta_k  \geq \sqrt{\Delta_k / \kappa} + \Delta_k = \sqrt{\Delta_{k - 1} / \kappa}. 
\end{align*}
%
In this case, $\Delta_{k - 1} \geq \Delta_k$ holds for all $k = 1, 2, \cdots, K - 1$, which further implies that $\Delta_k \geq \kappa \delta^2$ for all $k = 0, 1, \cdots, K - 1$. 
As a consequence, $\Delta_{k} \geq \Delta_{K - 1} (1 + \kappa \delta)^{K - k - 1} = \kappa \delta^2 (1 + \kappa \delta)^{K - k - 1}$. 

We then upper bound the number of steps $K$ needed as a function of $\delta, \kappa$ and $T$. 
Let $K_1 = \sup\{k: \Delta_k = \kappa\}$, then $K \leq T / \kappa + K - K_1$. 
Recall that $\Delta_{k} \geq \kappa \delta^2 (1 + \kappa \delta)^{K - k - 1}$, hence 
\begin{align*}
    K - K_1 \lesssim \frac{1}{\kappa \delta} \log (1 + 1 / \delta). 
\end{align*}
We define
\begin{align*}
    T = \frac{1}{2} \log (d / \eps^2), \qquad \kappa = \min \bigg\{ \frac{\eps}{d^{3 / 2} T^{1 / 2}},\,\, \frac{1}{d^2} \bigg\}. 
\end{align*}
By definition, we have $\kappa d^2 \leq 1$. 

When $\eps \leq \sqrt{d} / 2$, it holds that $T \gtrsim 1$. With $T$ and $\kappa$ selected as above, we have $d^3 \kappa^2 T \leq \eps^2$ and $ de^{-2T} = \eps^2$. 
In addition, it is seen that
\begin{align*}
    d^7  \kappa^3 \delta^{-1} \lesssim d^{5 / 2} \eps^3 \delta^{-1}, \qquad  d^7 \kappa^3 T \lesssim d^{5 / 2} \eps^3. 
\end{align*}
Also, note that 
\begin{align*}
    \sum_{k = 0}^{K - 1}  \frac{\eps_{\mathsf{score},k}^{1 / 2} \kappa^{1/2}\sigma_{T - t_k} d}{\lambda_{T - t_k}^{1 / 2}} \leq \sum_{k = 0}^{K - 1} \frac{\eps_{\mathsf{score},k}^{1/2} \eps^{1/2} d e^{T / 2} }{d^{3 / 4} T^{1 / 4}} \lesssim \sum_{k = 0}^{K - 1} d^{1/2}\eps_{\mathsf{score},k}^{1/2}. 
\end{align*}
Combining the preceding upper bounds yields 
\begin{align*}
    \KL (q_{\delta} \parallel p_{\mathsf{output}}) \lesssim \eps^2 + \eps^3 d^{5/2} \delta^{-1} + \sum_{k = 0}^{K - 1} d^{1/2}\eps_{\mathsf{score},k}^{1/2}.
\end{align*}
In this case, $K \lesssim \frac{1}{\kappa \delta} \log (1 + 1 / \delta) + \frac{1}{2\kappa} \log (d / \eps^2)$. 
\begin{itemize}
\item
If $\eps \leq 1 / \sqrt{d}$, then 
\begin{align*}
      K \lesssim \frac{d^{3/2}}{\delta \eps} \log (1 + 1 / \delta) \sqrt{\log (d / \eps^2)} + \frac{d^{3/2} }{\eps}\big[\log (d / \eps^2)\big]^{3/2} = \widetilde O\big(d^{3/2} (\eps \delta)^{-1}\big).
\end{align*}
\item
 In addition, if $\eps \geq 1 / \sqrt{d}$, then 
\begin{align*}
    K \lesssim \frac{d^2}{\delta} \log (1 + 1 / \delta) \sqrt{\log (d / \eps^2)} + d^2 \big[\log (d / \eps^2)\big]^{3/2} = \widetilde{O}\big(d^2 \delta^{-1}\big). 
\end{align*}
\end{itemize}
The proof is thus complete. 
 

\bibliographystyle{plainnat}
\bibliography{pratik}

\end{document}